\documentclass[11pt]{article}

\PassOptionsToPackage{numbers}{natbib}

\usepackage{etoolbox}
\newtoggle{arxiv}

\toggletrue{arxiv}

\usepackage[final]{nips_2016}
\usepackage[T1]{fontenc}    % use 8-bit T1 fonts
\usepackage{url}            % simple URL typesetting
\usepackage{booktabs}       % professional-quality tables
\usepackage{microtype}      % microtypography
\usepackage{graphics}
\usepackage{lscape}

\usepackage{mydefs}

\title{Graphons, mergeons, and so on!}

\author{
  Justin Eldridge \quad Mikhail Belkin \quad Yusu Wang\\
  The Ohio State University\\
  \texttt{\{eldridge, mbelkin, yusu\}@cse.ohio-state.edu}
}
\date{}

\begin{document}

\maketitle

\begin{abstract}

In this work we develop a theory of hierarchical clustering for graphs. Our
modeling assumption is that graphs are sampled from a graphon, which is a
powerful and general model for generating graphs and analyzing large networks.
Graphons are a  far richer class of graph models than stochastic blockmodels,
the primary setting for recent progress in the statistical theory of graph
clustering. We define what it means for an algorithm to produce the ``correct"
clustering, give sufficient conditions in which a method is statistically
consistent, and provide an explicit algorithm satisfying these properties.

\end{abstract}

\section{Introduction}

A fundamental problem in the theory of clustering is that of defining a cluster.
There is no single answer to this seemingly simple question. The right approach
depends on the nature of the data and the proper  modeling assumptions.  In a
statistical setting where the objects to be clustered come from some underlying
probability distribution, it is natural to define clusters in terms of the
distribution itself. The task of a clustering, then, is twofold -- to identify
the appropriate cluster structure of the distribution and to recover that
structure from a finite sample.  Thus we would like to say that a clustering is
good if it is in some sense close to the ideal structure of the underlying
distribution, and that a clustering method is \emph{consistent} if it produces
clusterings which converge to the true clustering,  given larger and larger
samples. Proving the consistency of a clustering method deepens our
understanding of it, and provides justification for using the method in the
appropriate setting.

In this work, we consider the setting in which the objects to be clustered are
the vertices of a graph sampled from a \emph{graphon} -- a very general random
graph model of significant recent interest.  We develop a statistical theory of
graph clustering in the graphon model; To the best of our knowledge, this is the
first general consistency framework developed for such a rich family of random
graphs.  The specific contributions of this paper are threefold.  First, we
define the clusters of a graphon.  Our definition results in a graphon having a
tree of clusters, which we call its \term{graphon cluster tree}. We introduce an
object called the \term{mergeon} which is a particular representation of the
graphon cluster tree that encodes the heights at which clusters merge. Second,
we develop a notion of consistency for graph clustering algorithms in which a
method is said to be consistent if its output converges to the graphon cluster
tree. Here the graphon setting poses subtle yet fundamental challenges which
differentiate it from classical clustering models, and which must be carefully
addressed. Third, we prove the existence of consistent clustering algorithms.
In particular, we provide sufficient conditions under which a graphon estimator
leads to a consistent clustering method. We then identify a specific practical
algorithm which satisfies these conditions, and in doing so present a simple
graph clustering algorithm which provably recovers the graphon cluster tree. 

\textbf{Related work.}
Graphons are objects of significant recent interest in graph theory, statistics,
and machine learning. The theory of graphons is rich and diverse; A graphon can
be interpreted as a generalization of a weighted graph with uncountably many
nodes, as the limit of a sequence of finite graphs, or, more importantly for the
present work, as a very general model for generating unweighted, undirected
graphs. Conveniently, any graphon can be represented as a symmetric,
measurable function $W : [0,1]^2 \to [0,1]$, and it is this representation that
we use throughout this paper.

The graphon as a graph limit was introduced in recent years by
\cite{Lovasz2006-dq}, \cite{Borgs2008-lv}, and others. The interested reader is
directed to the book by \citet{Lovasz2012-ww} on the subject. There has also been
a considerable recent effort to produce consistent estimators of the graphon,
including the work of \cite{Wolfe2013-xi}, \cite{Chan2014-js},
\cite{Airoldi2013-eu}, \cite{Rohe2011-jb}, and others. We will analyze a simple
modification of the graphon estimator proposed by \cite{Zhang2015-ik} and show
that it leads to a graph clustering algorithm which is a consistent estimator of
the graphon cluster tree.

Much of the previous statistical theory of graph clustering methods assumes
that graphs are generated by the so-called \emph{stochastic blockmodel}. The
simplest form of the model generates a graph with $n$ nodes by assigning each
node, randomly or deterministically, to one of two communities. An edge between
two nodes is added with probability $\alpha$ if they are from the same
community and with probability $\beta$ otherwise. A graph clustering method is
said to achieve \emph{exact recovery} if it identifies the true community
assignment of every node in the graph with high probability as $n \to \infty$.
The blockmodel is a special case of a graphon model, and our notion of
consistency will imply exact recovery of communities.

Stochastic blockmodels are widely studied, and it is known that, for example,
spectral methods like that of \cite{McSherry2001-be} are able to recover the
communities exactly as $n \to \infty$, provided that $\alpha$ and $\beta$
remain constant, or that the gap between them does not shrink too quickly. For
a summary of consistency results in the blockmodel, see \cite{Abbe_undated-lp},
which also provides information-theoretic thresholds for the conditions under
which exact recovery is possible.  In a related direction,
\cite{Balakrishnan2011-wb} examines the ability of spectral clustering to
withstand noise in a hierarchical block model.

\textbf{The density setting.} The problem of defining the underlying cluster
structure of a probability distribution goes back to \citet{hartigan_1981} who
considered the setting in which the objects to be clustered are points sampled
from a density $f : \mathcal X \to \positivereals$. In this case, the
\emph{high density clusters} of $f$ are defined to be the connected components
of the upper level sets $\{ x : f(x) \geq \lambda \}$ for any $\lambda > 0$.
The set of all such clusters forms the so-called \emph{density cluster tree}.
\citet{hartigan_1981} defined a notion of consistency for the density cluster
tree, and proved that single-linkage clustering is \emph{not} consistent. In
recent years, \cite{chaudhuri_2010} and \cite{kpotufe_2011} have demonstrated
methods which \emph{are} Hartigan consistent. \cite{eldridge2015-kr} introduced
a distance between a clustering of the data and the density cluster tree,
called the \term{merge distortion metric}. A clustering method is said to be
\term{consistent} if the trees it produces converge in merge distortion to
density cluster tree. It is shown that convergence in merge distortion is
stronger than Hartigan consistency, and that the method of
\cite{chaudhuri_2010} is consistent in this stronger sense.

In the present work, we will be motivated by the approach taken in
\cite{hartigan_1981} and \cite{eldridge2015-kr}. We note, however, that there
are significant and fundamental differences between the density case and the
graphon setting.  Specifically, it is possible for two graphons to be
equivalent in the same way that two graphs are: up to a relabeling of the
vertices. As such, a graphon $W$ is a representative of an equivalence class of
graphons modulo appropriately defined relabeling. It is therefore necessary to
define the clusters of $W$ in a way that does not depend upon the particular
representative used. A similar problem occurs in the density setting when we
wish to define the clusters not of a single density function, but rather of a
\emph{class} of densities which are equal almost everywhere;
\citet{Steinwart2011-uk} provides an elegant solution. But while the domain of
a density is equipped with a meaningful metric -- the mass of a ball around a
point $x$ is the same under two equivalent densities -- the ambient metric on
the vertices of a graphon is not useful.  As a result, approaches such as that
of \cite{Steinwart2011-uk} do not directly apply to the graphon case, and we
must carefully produce our own.  Additionally, we will see that the procedure
for sampling a graph from a graphon involves latent variables which are in
principle unrecoverable from data.  These issues have no analogue in the
classical density setting, and present very distinct challenges.

\textbf{Miscellany.} 
\iftoggle{arxiv}{
    For simplicity,%
}{
    Due to space constraints,%
}
most of the (rather involved) technical details are in the appendix. We will use
$[n]$ to denote the set $\{ 1, \ldots, n\}$, $\symdiff$ for the symmetric
difference, $\mu$ for the Lebesgue measure on $[0,1]$, and bold letters to
denote random variables. 

\section{The graphon model}

In order to discuss the statistical properties of a graph clustering algorithm,
we must first model the process by which graphs are generated. Formally, a
\term{random graph model} is a sequence of random variables $\rv{G}_1, \rv{G}_2,
\ldots$ such that the range of $\rv{G}_n$ consists of undirected, unweighted
graphs with node set $[n]$, and the distribution of $\rv{G}_n$ is invariant
under relabeling of the nodes -- that is, isomorphic graphs occur with equal
probability. A random graph model of considerable recent interest is the
\term{graphon model}, in which the distribution over graphs is determined by a
symmetric, measurable function $W : [0,1]^2 \to [0,1]$ called a \term{graphon}.
Informally, a graphon $W$ may be thought of as the weight matrix of an infinite
graph whose node set is the continuous unit interval, so that $W(x,y)$
represents the weight of the edge between nodes $x$ and $y$.

Interpreting $W(x,y)$ as a probability suggests the following graph sampling
procedure: To draw a graph with $n$ nodes, we first select $n$ points $\rv{x}_1,
\ldots, \rv{x}_n$ at random from the uniform distribution on $[0,1]$ -- we can
think of these $\rv{x_i}$ as being random ``nodes'' in the graphon. We
then sample a random graph $\rv{G}$ on node set $[n]$ by admitting the edge
$(i,j)$ with probability $W(\rv{x}_i, \rv{x}_j)$; by convention, self-edges are
not sampled. It is important to note that while we begin by drawing a set of
nodes $\{\rv{x_i}\}$ from the graphon, the graph as given to us is labeled by
integers. Therefore, the correspondence between node $i$ in the graph and node
$\rv{x_i}$ in the graphon is latent.

It can be shown that this sampling procedure defines a distribution on finite
graphs, such that the probability of graph $G = ([n], E)$ is given by

\begin{minipage}{\textwidth}
\begin{equation}
    \label{eqn:graphon_integral}
    \prob_W(\rv{G} = G) = \int_{[0,1]^n} 
    \prod_{(i,j) \in E} W(x_i, x_j) 
    \prod_{(i,j) \not \in E} \left[ 1 - W(x_i, x_j) \right]
    \prod_{i \in [n]} dx_i.
\end{equation}
\end{minipage}

For a fixed choice of $x_1, \ldots, x_n \in [0,1]$, the integrand represents the
likelihood that the graph $G$ is sampled when the probability of the edge
$(i,j)$ is assumed to be $W(x_i, x_j)$. By integrating over all possible choices
of $x_1, \ldots, x_n$, we obtain the probability of the graph.

    \begin{wrapfigure}{r}{.16\textwidth}
    \centering
    \vspace{-1em}
    \begin{subfigure}[t]{.15\textwidth}
        \centering
        \includegraphics[width=\textwidth]{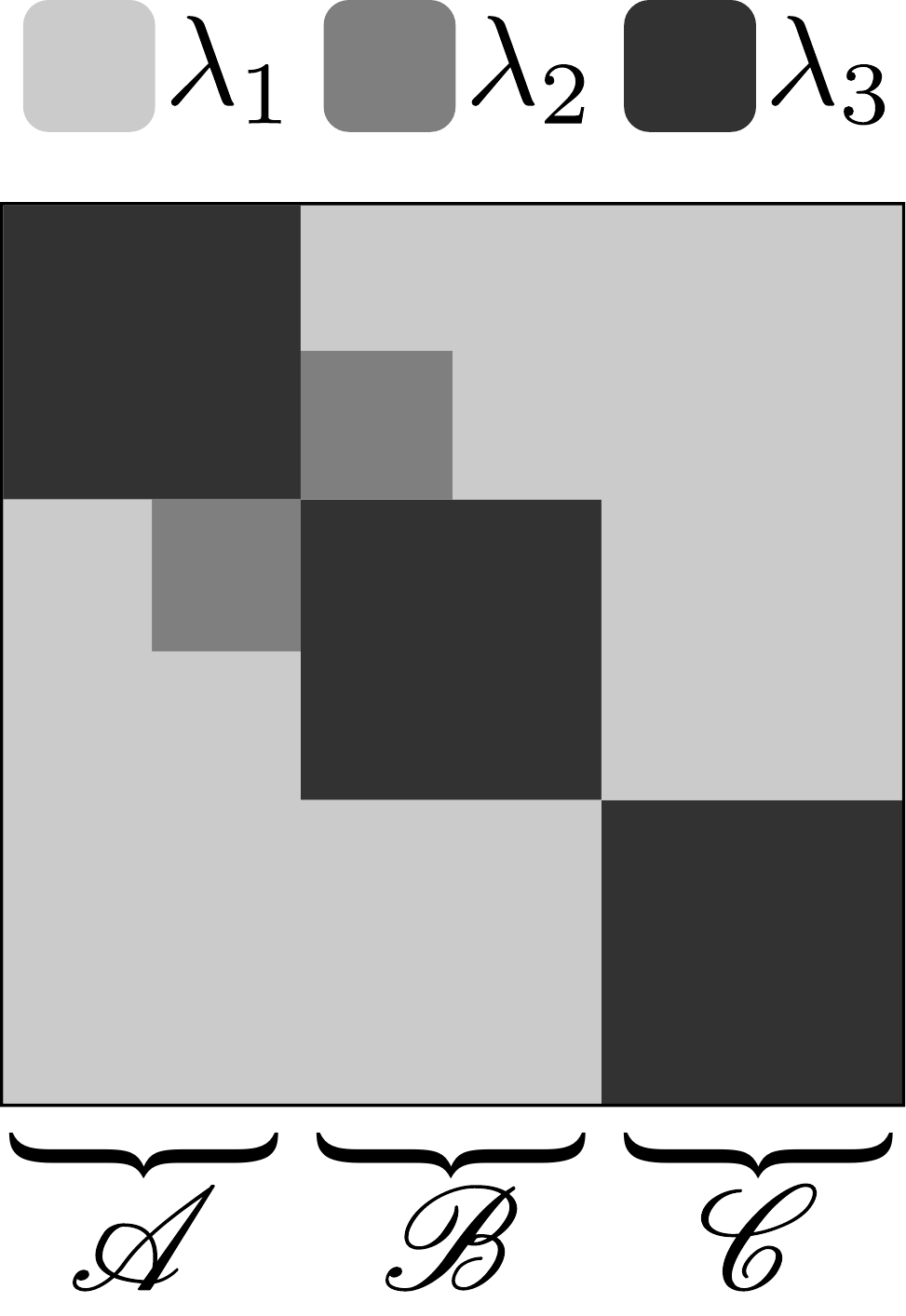}
        \caption{
        \label{fig:example_graphon:w}
        Graphon $W$.}
    \end{subfigure}\\[.5em]
    \begin{subfigure}[t]{.15\textwidth}
        \centering
        \includegraphics[width=\textwidth]{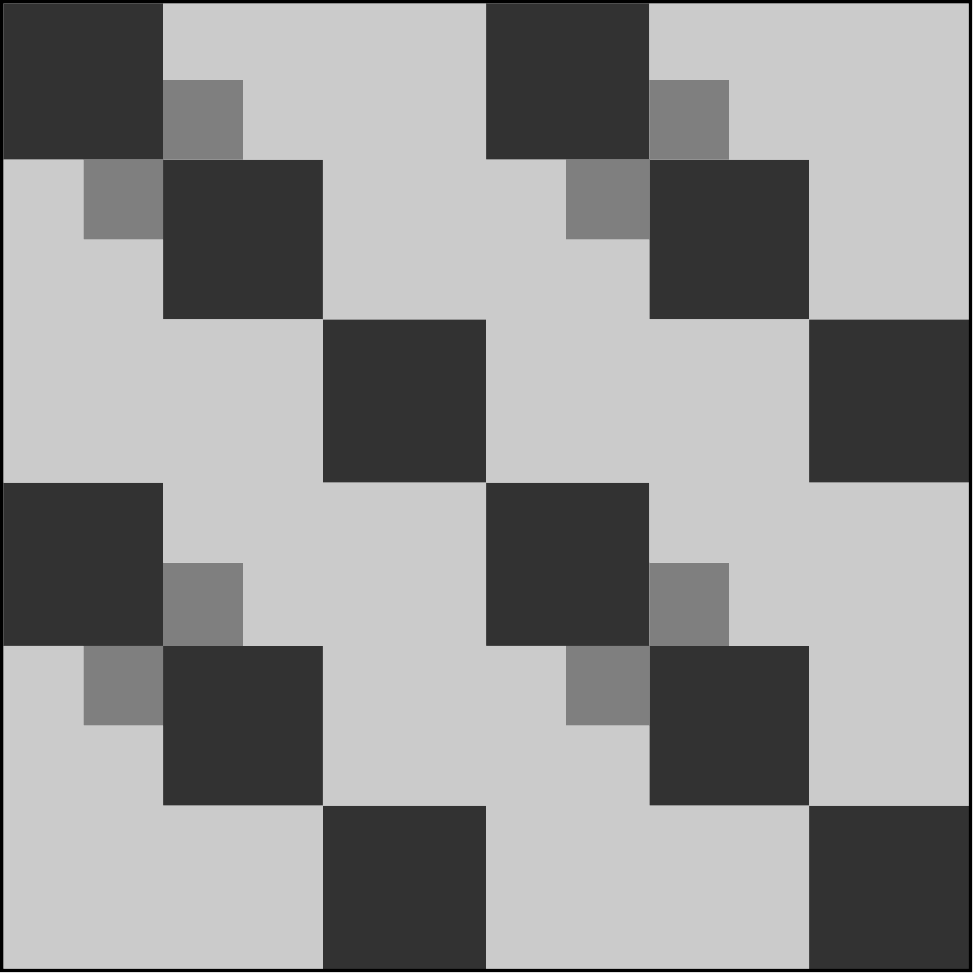}
        \caption{
        \label{fig:example_graphon:w_phi}
        \hbadness=10000
        $W^\varphi$ weakly isomorphic to $W$.}
     \end{subfigure}\\[.5em]
    \begin{subfigure}[t]{.15\textwidth}
        \centering
        \includegraphics[width=\textwidth]{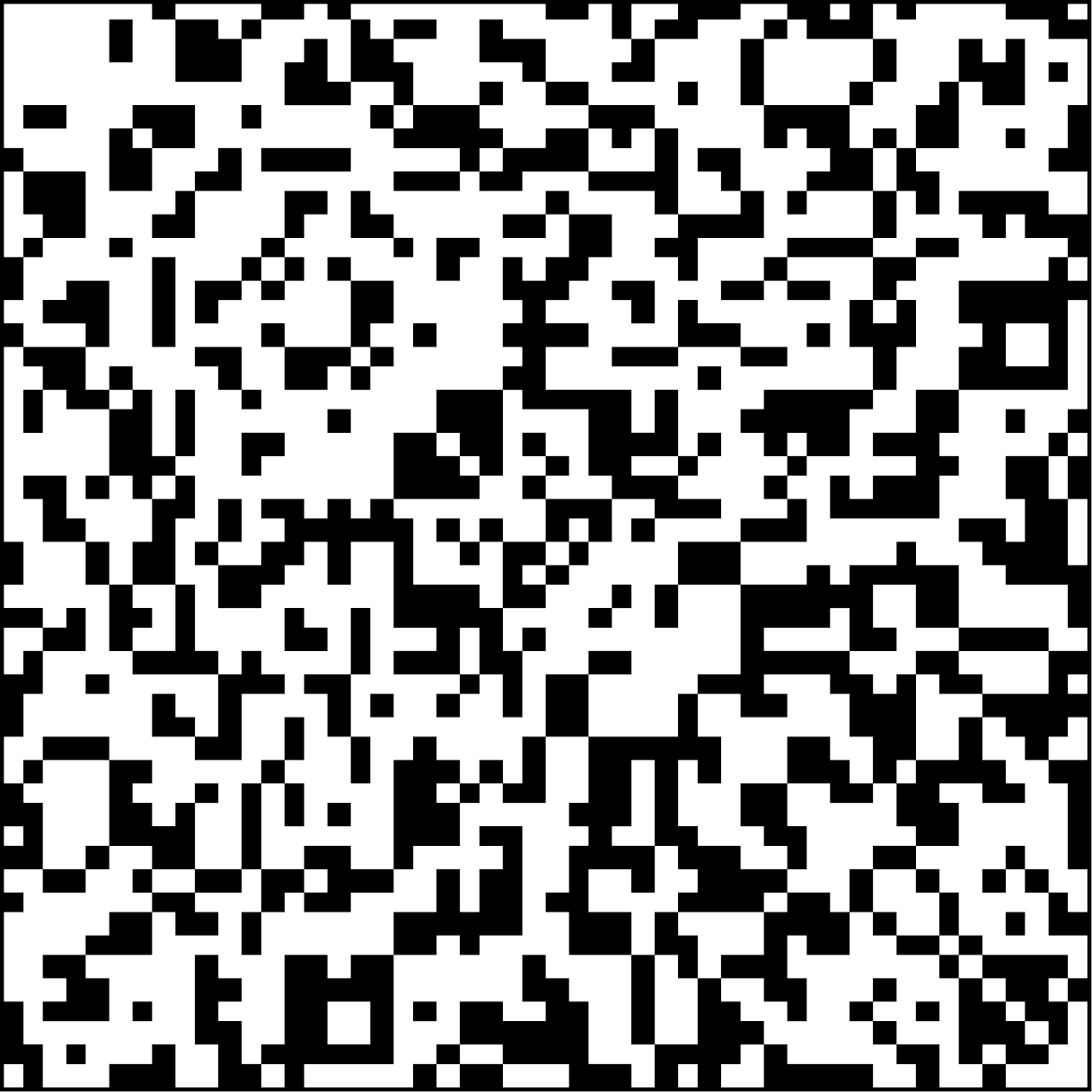}
        \caption{
        \label{fig:example_graphon:w_sampled}
        \hbadness=10000
        An instance of a graph  adjacency sampled from $W$.}
    \end{subfigure}
    \caption{\label{fig:example_graphons}}
        \vspace{-3em}
\end{wrapfigure}

A very general class of random graph models may be represented as graphons. In
particular, a random graph model $\rv{G}_1, \rv{G}_2, \ldots$ is said to be
\term{consistent} if the random graph $\rv{F}_{k-1}$ obtained by deleting node
$k$ from $\rv{G}_k$ has the same distribution as $\rv{G}_k$. A random graph
model is said to be \term{local} if whenever $S,T \subset [k]$ are disjoint, the
random subgraphs of $\rv{G}_k$ induced by $S$ and $T$ are independent random
variables. A result of \citet{Lovasz2006-dq} is that any consistent, local
random graph model is equivalent to the distribution on graphs defined by
$\prob_W$ for some graphon $W$; the converse is true as well. That is, any such
random graph model is equivalent to a graphon.

A particular random graph model is not uniquely defined by a graphon -- it is
clear from Equation~\ref{eqn:graphon_integral} that two graphons $W_1$ and $W_2$
which are equal almost everywhere (i.e., differ on a set of measure zero) define
the same distribution on graphs. In fact, the distribution defined by $W$ is
unchanged by ``relabelings'' of $W$'s nodes. More formally, if $\Sigma$ is the
sigma-algebra of Lebesgue measurable subsets of $[0,1]$ and $\mu$ is the
Lebesgue measure, we say that a relabeling function $\varphi : ([0,1], \Sigma)
\to ([0,1], \Sigma)$ is \term{measure preserving} if for any measurable set $A
\in \Sigma$, $\lebesguemeasure(\varphi^{-1}(A)) = \lebesguemeasure(A)$. We
define the relabeled graphon $W^\varphi$ by $W^\varphi(x,y) = W(\varphi(x),
\varphi(y))$.  By analogy with finite graphs, we say that graphons $W_1$ and
$W_2$ are \term{weakly isomorphic} if they are equivalent up to relabeling,
i.e., if there exist measure preserving maps $\varphi_1$ and $\varphi_2$ such
that $W_1^{\varphi_1} = W_2^{\varphi_2}$ almost everywhere. Weak isomorphism is
an equivalence relation, and most of the important properties of a graphon in
fact belong to its equivalence class. For instance, a powerful result of
\cite{Lovasz2012-ww} is that two graphons define the same random graph model if
and only if they are weakly isomorphic.

An example of a graphon $W$ is shown in Figure~\ref{fig:example_graphon:w}. It
is conventional to plot the graphon as one typically plots an adjacency matrix:
with the origin in the upper-left corner. Darker shades correspond to higher
values of $W$. Figure~\ref{fig:example_graphon:w_phi} depicts a graphon
$W^\varphi$ which is weakly isomorphic to $W$. In particular, $W^\varphi$ is the
relabeling of $W$ by the measure preserving transformation $\varphi(x) = 2x \mod
1$. As such, the graphons shown in
Figures~\ref{fig:example_graphon:w}~and~\ref{fig:example_graphon:w_phi} define
the same distribution on graphs. Figure~\ref{fig:example_graphon:w_sampled}
shows the adjacency matrix $A$ of a graph of size $n=50$ sampled from the
distribution defined by the equivalence class containing $W$ and $W^\varphi$.
Note that it is in principle not possible to determine from $A$ alone which
graphon $W$ or $W^\varphi$ it was sampled from, or to what node in $W$ a
particular column of $A$ corresponds to.

\section{The graphon cluster tree}

We now identify the cluster structure of a graphon.  We will define a graphon's
clusters such that they are analogous to the maximally-connected components of a
finite graph. It turns out that the collection of all clusters has hierarchical
structure; we call this object the \term{graphon cluster tree}. We propose that
the goal of clustering in the graphon setting is the recovery of the graphon
cluster tree.

\textbf{Connectedness and clusters.} Consider a finite weighted graph. It is
natural to cluster the graph into connected components. In fact, because of the
weighted edges, we can speak of the clusters of the graph at various levels.
More precisely, we say that a set of nodes $A$ is \term{internally connected} --
or, from now on, just \term{connected} -- at level $\lambda$ if for every pair
of nodes in $A$ there is a path between them such that every node along the path
is also in $A$, and the weight of every edge in the path is at least $\lambda$.
Equivalently, $A$ is \term{connected} at level $\lambda$ if and only if for
every partitioning of $A$ into disjoint, non-empty sets $A_1$ and $A_2$ there is
an edge of weight $\lambda$ or greater between $A_1$ and $A_2$.  The clusters at
level $\lambda$ are then the largest connected components at level $\lambda$.

A graphon is, in a sense, an infinite weighted graph, and we will define the
clusters of a graphon using the example above as motivation. In doing so, we
must be careful to make our notion robust to changes of the graphon on a set of
zero measure, as such changes do not affect the graph distribution defined by
the graphon. We base our definition on that of \citet{Janson2008-qn}, who
defined what it means for a graphon to be connected as a whole. We extend the
definition in \cite{Janson2008-qn} to speak of the connectivity of subsets of
the graphon's nodes at a particular height. Our definition is directly analogous
to the notion of internal connectedness in finite graphs.

\begin{defn}[Connectedness]
    Let $W$ be a graphon, and let $A \subset [0,1]$ be a set of positive
    measure.  We say that $A$ is \term{disconnected at level $\lambda$} if there
    exists a measurable $S \subset A$ such that $0 < \lebesguemeasure(S) <
    \lebesguemeasure(A)$, and $W < \lambda$ almost everywhere on $S \times (A
    \setminus S)$.  Otherwise, we say that $A$ is \term{connected at level
    $\lambda$.}
\end{defn}

We now identify the clusters of a graphon; as in the finite case, we will frame
our definition in terms of maximally-connected components. We begin by gathering
all subsets of $[0,1]$ which should belong to some cluster at level $\lambda$.
Naturally, if a set is connected at level $\lambda$, it should be in a cluster
at level $\lambda$; for technical reasons, we will also say that a set which is
connected at all levels $\lambda' < \lambda$ (though perhaps not at $\lambda$)
should be contained in a cluster at level $\lambda$, as well. That is, for any
$\lambda$, the collection $\upperlevel$ of sets which should be contained in
some cluster at level $\lambda$ is
$
    \upperlevel = \{ \,A \in \Sigma : \mu(A) > 0 \text{ and $A$ is connected at
    every level $\lambda' < \lambda$} \}.
$
Now suppose $A_1,A_2 \in \upperlevel$, and that there is a set
$A \in \upperlevel$ such that $A \supset A_1 \cup A_2$. Naturally, the cluster
to which $A$ belongs should also contain $A_1$ and $A_2$, since both are subsets
of $A$.  We will therefore consider $A_1$ and $A_2$ to be equivalent, in the
sense that they should be contained in the same cluster at level $\lambda$. More
formally, we define a relation $\clusteredwith_\lambda$ on $\upperlevel$ by
$
    A_1 \clusteredwith_\lambda A_2 \Longleftrightarrow 
    \exists A \in \upperlevel 
    \;\text{s.t.}\;
    A \supset A_1 \cup A_2.
$
It can be verified that $\clusteredwith_\lambda$ is an equivalence relation on
$\upperlevel$; see Claim~\ref{claim:connected_equivalence} in
Appendix~\ref{apx:proofs}.

\newcommand{\eqclass}{\mathscr A}

Each equivalence class $\eqclass$ in the quotient space $\upperlevel /
{\clusteredwith_\lambda}$.  consists of connected sets which should intuitively
be clustered together at level $\lambda$. Naturally, we will define the clusters
to be the largest elements of each class; in some sense, these are the
maximally-connected components at level $\lambda$. More precisely, suppose
$\eqclass$ is such an equivalence class.  It is clear that in general no single
member $A \in \eqclass$ can contain all other members of $\eqclass$, since
adding a null set (i.e., a set of measure zero) to $A$ results in a larger set
$A'$ which is nevertheless still a member of $\eqclass$.  However, we can find a
member $A^* \in \eqclass$ which contains all but a null set of every other set
in $\eqclass$.  More formally, we say that $A^*$ is an \term{essential maximum}
of the class $\eqclass$ if $A^* \in \eqclass$ and for every $A \in \eqclass$,
$\mu(A \setminus A^*) = 0$.  $A^*$ is of course not unique, but it is unique up
to a null set; i.e., for any two essential maxima $A_1, A_2$ of $\eqclass$, we
have $\mu(A_1 \symdiff A_2) = 0$.  We will write the set of essential maxima of
$\eqclass$ as $\essmax \eqclass$; the fact that the essential maxima are
well-defined is proven in Claim~\ref{claim:essmax_cluster} in
Appendix~\ref{apx:proofs}. We then define clusters as the maximal members of
each equivalence class in ${\upperlevel}/{\clusteredwith_\lambda}$:

\begin{defn}[Clusters]
    \label{defn:clusters}
    The set of clusters at level $\lambda$ in $W$, written
    $\clustertree_W(\lambda)$, is defined to be the countable collection
    $
        \clustertree_W(\lambda) = \left\{
            \; \essmax \eqclass : \eqclass \in {\upperlevel
        /}{\clusteredwith_\lambda}
        \right\}.
    $
\end{defn}

Note that a cluster $\cluster$ of a graphon is not a subset of the unit interval
per se, but rather an \term{equivalence class} of subsets which differ only by
null sets. It is often possible to treat clusters as sets rather than
equivalence classes, and we may write $\mu(\cluster)$, $\cluster \cup
\cluster'$, etc., without ambiguity. In addition, if $\varphi : [0,1] \to [0,1]$
is a measure preserving transformation, then $\varphi^{-1}(\cluster)$ is
well-defined.

For a concrete example of our notion of a cluster, consider the graphon $W$
depicted in Figure~\ref{fig:example_graphon:w}. $A$, $B$, and $C$ represent sets
of the graphon's nodes. By our definitions there are three clusters at level
$\lambda_3$: $\clusterstyle{A}$, $\clusterstyle{B}$, and $\clusterstyle{C}$.
Clusters $\clusterstyle{A}$ and $\clusterstyle{B}$ merge into a cluster
$\clusterstyle{A} \cup \clusterstyle{B}$ at level $\lambda_2$, while
$\clusterstyle{C}$ remains a separate cluster.  Everything is joined into a
cluster $\clusterstyle{A} \cup \clusterstyle{B} \cup \clusterstyle{C}$ at level
$\lambda_1$.

We have taken care to define the clusters of a graphon in such a way as to be
robust to changes of measure zero to the graphon itself. In fact, clusters are
also robust to measure preserving transformations. The proof of this result is
non-trivial, and comprises \Cref{apx:mptclusters}.

\begin{restatable}{claim}{claimmptclusters}
    \label{claim:mptclusters}
    Let $W$ be a graphon and $\varphi$ a measure preserving transformation. Then
    $\cluster$ is a cluster of $W^\varphi$ at level $\lambda$ if and only if
    there exists a cluster $\cluster'$ of $W$ at level $\lambda$ such that
    $\cluster = \varphi^{-1}(\cluster')$.
\end{restatable}

\textbf{Cluster trees and mergeons.}
The set of all clusters of a graphon at any level has hierarchical structure in
the sense that, given any pair of distinct clusters $\cluster_1$ and
$\cluster_2$, either one is ``essentially'' contained within the other, i.e.,
$\cluster_1 \subset \cluster_2$, or
$\cluster_2 \subset \cluster_1$, 
or they are ``essentially'' disjoint, i.e.,
$\mu(\cluster_1 \cap \cluster_2) = 0$, as is proven by
Claim~\ref{claim:connected_union_connected} in \Cref{apx:proofs}.  Because of
this hierarchical structure, we call the set $\clustertree_W$ of all clusters
from any level of the graphon $W$ the \term{graphon cluster tree} of $W$. It is
this tree that we hope to recover by applying a graph clustering algorithm to a
graph sampled from $W$.

\begin{wrapfigure}{r}{.16\textwidth}
    \centering
    \vspace{-1em}
    \begin{subfigure}{.15\textwidth}
        \centering
        \includegraphics[width=\textwidth]{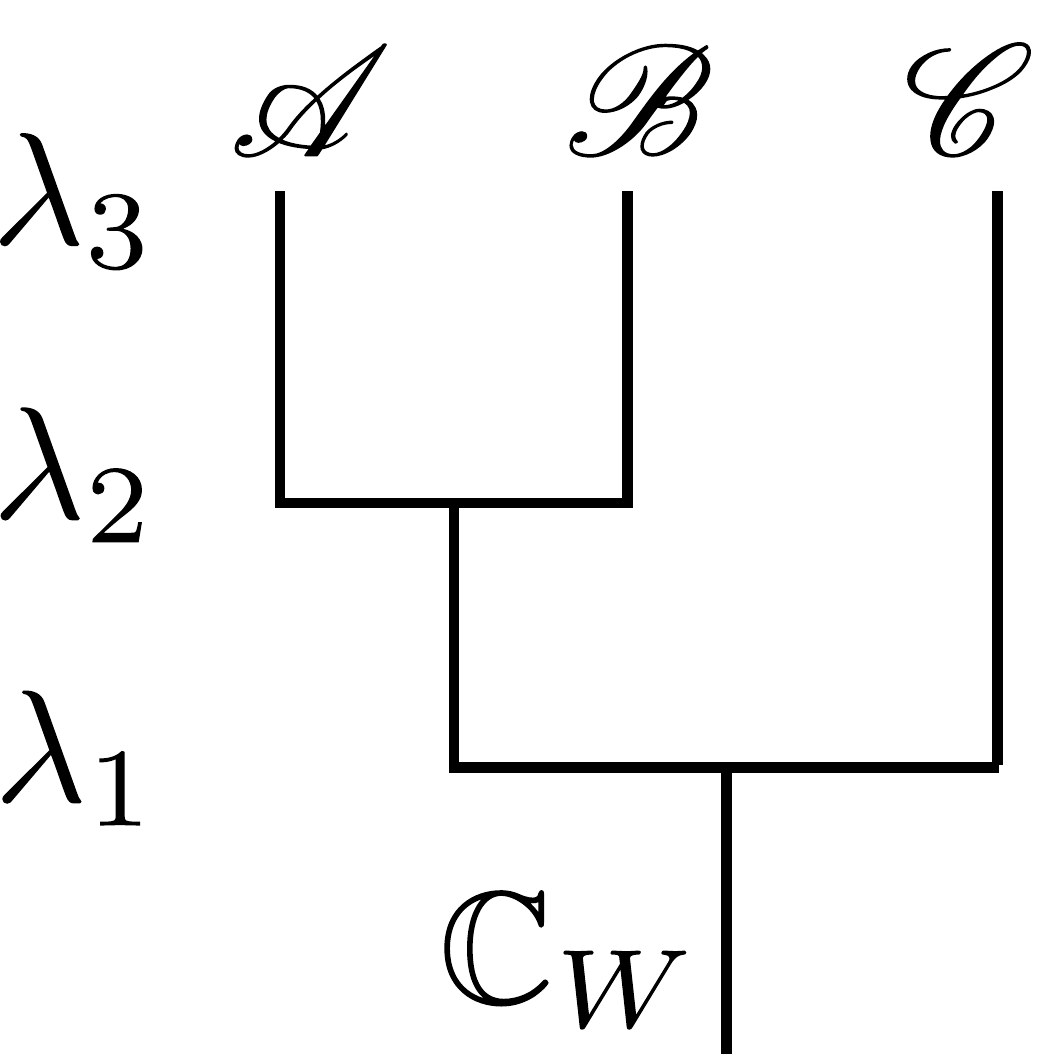}
        \caption{
        \label{fig:mergeon:cluster_tree}
        Cluster tree $\clustertree_W$ of $W$.}
    \end{subfigure}\\[.5em]
    \begin{subfigure}{.15\textwidth}
        \centering
        \includegraphics[width=\textwidth]{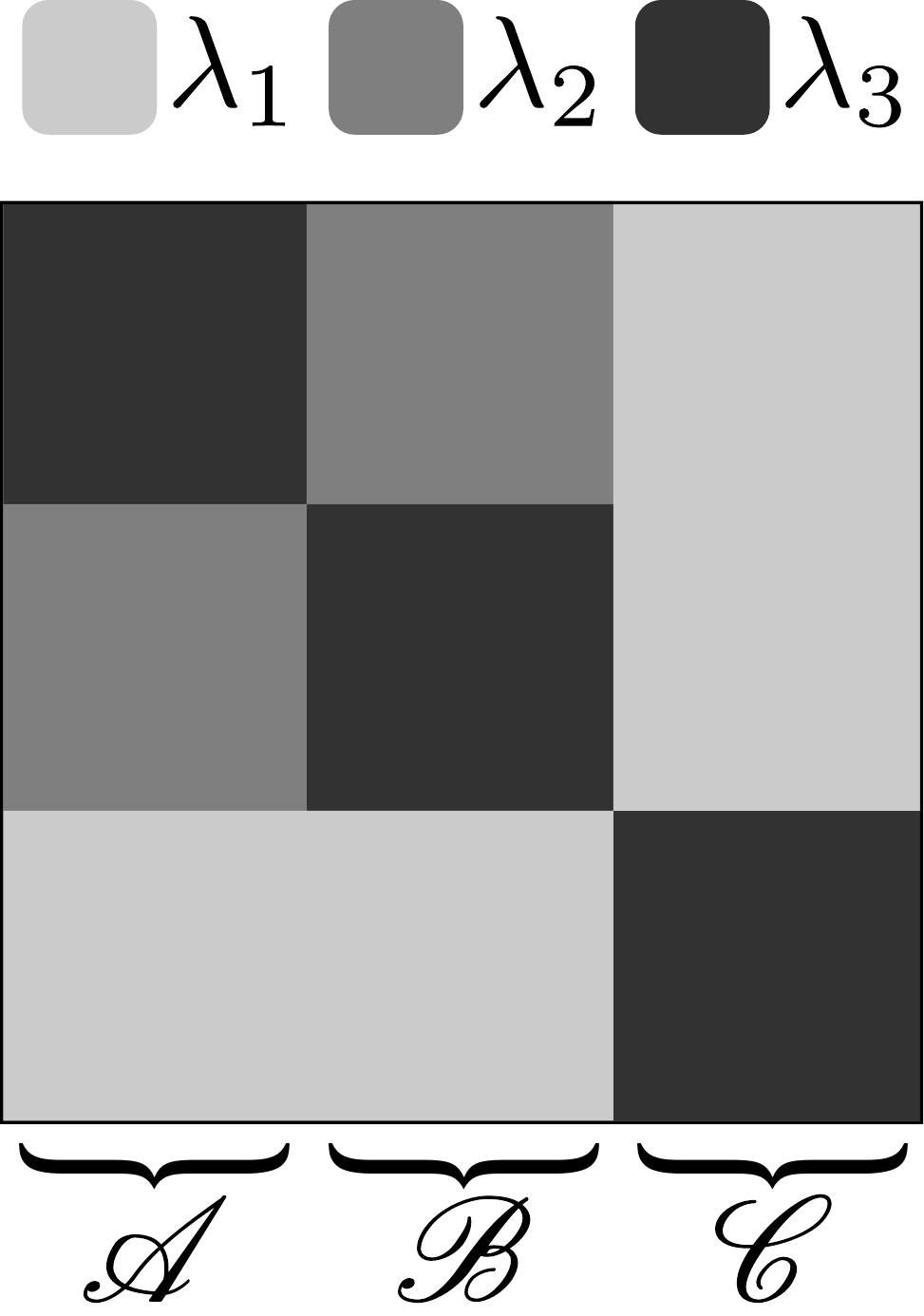}
        \caption{
        \label{fig:mergeon:mergeon}
        Mergeon $M$ of $\clustertree_W$.}
    \end{subfigure}
    \caption{\label{fig:mergeon}}
    \vspace{-3em}
\end{wrapfigure}

We may naturally speak of the height at which pairs of distinct clusters merge
in the cluster tree.  For instance, let $\cluster_1$ and $\cluster_2$ be
distinct clusters of $\clustertree$. We say that the \term{merge height} of
$\cluster_1$ and $\cluster_2$ is the level $\lambda$ at which they are joined
into a single cluster, i.e., 
$
    \max \{ \lambda : \cluster_1 \cup \cluster_2 \in \clustertree(\lambda)
    \}.
$
However, while the merge height of clusters is well-defined, the merge height of
individual points is not. This is because the cluster tree is not a collection
of sets, but rather a collection of equivalence classes of sets, and so a point
does not belong to any one cluster more than any other. Note that this is
distinct from the classical density case considered in \cite{hartigan_1981},
\cite{chaudhuri_2010}, and \cite{Abbe_undated-lp}, where the merge height of any
pair of points is well-defined.

Nevertheless, consider a measurable function $M : [0,1]^2 \to [0,1]$ which
assigns a merge height to every pair of points. While the value of $M$ on any
given pair is arbitrary, the value of $M$ on sets of positive measure is
constrained. Intuitively, if $\cluster$ is a cluster at level $\lambda$, then we
must have $M \geq \lambda$ almost everywhere on $\cluster \times \cluster$. If
$M$ satisfies this constraint for every cluster $\cluster$ we call $M$ a
\term{mergeon} for $\clustertree$, as it is a graphon which determines a
particular choice for the merge heights of every pair of points in $[0,1]$. More
formally:

\begin{defn}[Mergeon]
    \label{defn:mergeon}
    Let $\clustertree$ be a cluster tree.
    A \term{mergeon}\footnote{The definition given here involves a slight abuse
    of notation. For a precise -- but more technical -- version, see
    \Cref{apx:precise_mergeon}.} of $\clustertree$ is a graphon $M$ such that for
all
    $\lambda \in [0,1]$,
    $
        M^{-1}[\lambda,1]
        =
        \bigcup_{\cluster \in \clustertree_W(\lambda)} 
        \cluster \times \cluster,
    $
    where $M^{-1}[\lambda,1] = \{(x,y) \in [0,1]^2 : M(x,y) \geq
    \lambda \}$. 
\end{defn}

An example of a mergeon and the cluster tree it represents is shown in
Figure~\ref{fig:mergeon}. In fact, the cluster tree depicted is that of the
graphon $W$ from Figure~\ref{fig:example_graphon:w}. The mergeon encodes the
height at which clusters $\clusterstyle{A}$, $\clusterstyle{B}$, and
$\clusterstyle{C}$ merge. In particular, the fact that $M = \lambda_2$
everywhere on $\clusterstyle{A} \times \clusterstyle{B}$ represents the merging
of $\clusterstyle{A}$ and $\clusterstyle{B}$ at level $\lambda_2$ in $W$.

It is clear that in general there is no unique mergeon representing a graphon
cluster tree, however, the above definition implies that two mergeons
representing the same cluster tree are equal almost everywhere. Additionally, we
have the following two claims, whose proofs are in \Cref{apx:proofs}.

\begin{restatable}{claim}{claimmergeonequivalence}
    Let $\clustertree$ be a cluster tree, and suppose $M$ is a mergeon
    representing $\clustertree$. Then $\cluster \in \clustertree(\lambda)$ if
    and only if $\cluster$ is a cluster in $M$ at level $\lambda$. In other
    words, the cluster tree of $M$ is also $\clustertree$.
\end{restatable}

\begin{restatable}{claim}{claimmptmergeon}
    Let $W$ be a graphon and $M$ a mergeon of the cluster tree of $W$. If
    $\varphi$ is a measure preserving transformation, then $M^\varphi$ is a
    mergeon of the cluster tree of $W^\varphi$.
\end{restatable}

\section{Notions of consistency}

We have so far defined the sense in which a graphon has hierarchical cluster
structure.  We now turn to the problem of determining whether a clustering
algorithm is able to recover this structure when applied to a graph sampled from
a graphon. Our approach is to define a distance between the infinite graphon
cluster tree and a finite clustering. We will then define consistency by
requiring that a consistent method converge to the graphon cluster tree in this
distance for all inputs minus a set of vanishing probability.

\textbf{Merge distortion.} A \term{hierarchical clustering} $\clustering$ of a
set $S$ -- or, from now on, just a \term{clustering} of $S$ -- is hierarchical
collection of subsets of $S$ such that $S \in \clustering$ and for all $C,C' \in
\clustering$, either $C \subset C'$, $C' \subset C$, or $C \cap C' = \emptyset$.
Suppose $\clustering$ is a clustering of a finite set $S$ consisting of graphon
nodes; i.e, $S \subset [0,1]$. How might we measure the distance between this
clustering and a graphon cluster tree $\clustertree$?  Intuitively, the two
trees are close if every pair of points in $S$ merges in $\clustering$ at about
the same level as they merge in $\clustertree$. But this informal description
faces two problems: First, $\clustertree$ is a collection of equivalence classes
of sets, and so the height at which any pair of points merges in $\clustertree$
is not defined. Recall, however, that the cluster tree has an alternative
representation as a \term{mergeon}. A mergeon \emph{does} define a merge height
for every pair of nodes in a graphon, and thus provides a solution to this first
issue.  Second, the clustering $\clustering$ is not equipped with a height
function, and so the height at which any pair of points merges in $\clustering$
is also undefined.  Following \cite{eldridge2015-kr}, our approach is to
\emph{induce} a merge height function on the clustering using the mergeon in the
following way: 

\begin{defn}[Induced merge height]
    \label{defn:induced_merge_height}
    Let $M$ be a mergeon, and suppose $S$ is a finite subset of $[0,1]$. Let
    $\clustering$ be a clustering of $S$.  The \term{merge height
    function on $\clustering$ induced by $M$} is defined by
    $
        \hat{M}_\clustering(s,s') = \min_{u,v \in \clustering(s,s')} M(u,v),
    $
    for every $s,s' \in S \times S$, where $\clustering(s,s')$ denotes the
    smallest cluster $C \in \clustering$ which contains both $s$ and $s'$.
\end{defn}

We measure the distance between a clustering $\clustering$ and the cluster
tree $\clustertree$ using the \term{merge distortion}:

\begin{defn}
    Let $M$ be a mergeon, $S$ a finite subset of $[0,1]$, and $\clustering$ a
        clustering of $S$. The \term{merge distortion} is defined by
        $\mergedist_S(M, \hat{M}_\clustering) =
        \max_{s,s' \in S,\,s\neq s'}
        |M(s,s') - \hat{M}_\clustering(s,s')|
        $.
\end{defn}

Defining the induced merge height and merge distortion in this way leads to an
especially meaningful interpretation of the merge distortion. In particular, if
the merge distortion between $\clustering$ and $\clustertree$ is $\epsilon$,
then any two clusters of $\clustertree$ which are separated at level $\lambda$
but merge below level $\lambda - \epsilon$ are correctly separated in the
clustering $\clustering$. A similar result guarantees that a cluster in
$\clustertree$ is connected in $\clustering$ at within $\epsilon$ of the correct
level. For a precise statement of these results, see
\Cref{claim:connected_separated} in 
\Cref{apx:merge_distortion_structure}.

\textbf{The label measure.} 
We will use the merge distortion to measure the distance between
$\clustering$, a hierarchical
clustering of a graph, and $\clustertree$, the graphon cluster tree. Recall,
however, that the nodes of a graph sampled from a graphon have integer labels.
That is, $\clustering$ is a clustering of $[n]$, and not of a subset of $[0,1]$.
Hence, in order to apply the merge distortion, we must first relabel the nodes
of the graph, placing them in direct correspondence to nodes of the graphon,
i.e., points in $[0,1]$.

\begin{wrapfigure}{r}{.16\textwidth}
    \vspace{-1em}
    \centering
    \begin{subfigure}{0.15\textwidth}
        \centering
        \includegraphics[width=\textwidth]{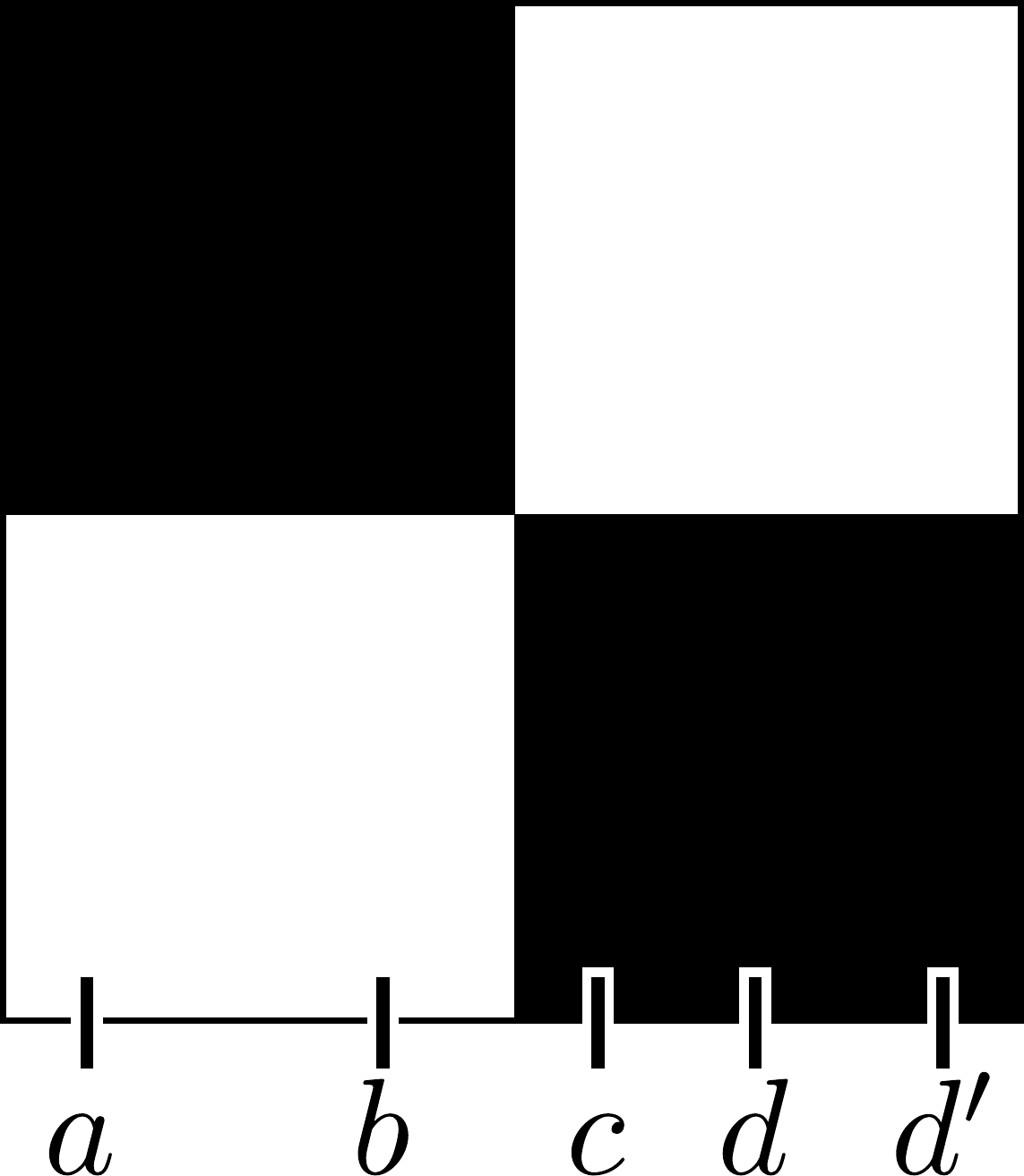}
    \end{subfigure}\\
    \vspace{1em}
    \begin{subfigure}[c]{0.15\textwidth}
        \centering
        \includegraphics[width=.9\textwidth]{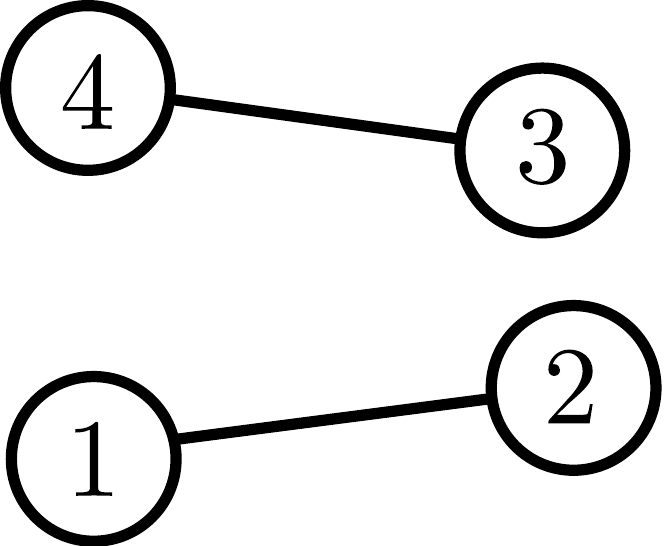}
    \end{subfigure}
    \vspace{-1.5em}
\end{wrapfigure}

Recall that we sample a graph of size $n$ from a graphon $W$ by first drawing
$n$ points $\rv{x_1}, \ldots, \rv{x_n}$ uniformly at random from the unit
interval. We then generate a graph on node set $[n]$ by connecting nodes $i$ and
$j$ with probability
$W(\rv{x_i}, \rv{x_j})$.  However, the nodes of the sampled graph are not
labeled by $\rv{x_1}, \ldots, \rv{x_n}$, but rather by the integers $1, \ldots,
n$. Thus we may think of $\rv{x_i}$ as being the ``true'' latent label of
node $i$.  In general the latent node labeling is not recoverable from data, as
is demonstrated by the figure to the right. We might suppose that the graph
shown is sampled from the graphon above it, and that node 1 corresponds to $a$,
node 2 to $b$, node 3 to $c$, and node $4$ to $d$.  However, it is just as
likely that node 4 corresponds to $d'$, and so neither labeling is more
``correct''. It is clear, though, that some labelings are less likely than
others.  For instance, the existence of the edge $(1,2)$ makes it impossible
that 1 corresponds to $a$ and 2 to $c$, since $W(a,c)$ is zero.

Therefore, given a graph $G = ([n], E)$ sampled from a graphon, there are many
possible relabelings of $G$ which place its nodes in correspondence with nodes
of the graphon, but some are more likely than others. The merge distortion
depends which labeling of $G$ we assume, but, intuitively, a good clustering of
$G$ will have small distortion with respect to highly probable labelings, and
only have large distortion on improbable labelings.  Our approach is to assign a
probability to every pair $(G,S)$ of a graph and possible labeling.  We will
thus be able to measure the probability mass of the set of pairs for which a
method performs poorly, i.e., results in a large merge distortion.

More formally, let $\allgraphs_n$ denote the set of all undirected, unweighted
graphs on node set $[n]$, and let $\Sigma^n$ be the sigma-algebra of
Lebesgue-measurable subsets of $[0,1]^n$. A graphon $W$ induces a unique product
measure $\labelmeasure$ defined on the product sigma-algebra $2^{\allgraphs_n}
\times \Sigma^n$ such that for all $\mathcal{G} \in 2^{\allgraphs_n}$ and
$\mathcal{S} \in \Sigma^n$:
\[
    \resizebox{.95\textwidth}{!}{$
    \labelmeasure(\mathcal G \times \mathcal S) = 
    \sum_{G \in \mathcal G} \left( \int_{\mathcal S} \likelihood_W(S|G) \, dS
    \right),
    \;
    \text{\small where}
    \;
    \displaystyle
    \likelihood_W(S \mid G) = 
    \prod_{(i,j) \in E(G)} W(x_i, x_j) 
    \prod_{(i,j) \not \in E(G)} \left[ 1 - W(x_i, x_j) \right]$,
    }
\]
where $E(G)$ represents the edge set of the graph $G$. We recognize
$\likelihood_W(S \mid G)$ as the integrand in
Equation~\ref{eqn:graphon_integral} for the probability of a graph as determined
by a graphon. If $G$ is fixed, integrating $\likelihood_W(S \mid G)$ over all
$S \in [0,1]^n$ gives the probability of $G$ under the model defined by $W$.

We may now formally define our notion of consistency. First, some notation: If
$\clustering$ is a clustering of $[n]$ and $S = (x_1, \ldots, x_n)$, write
$\clustering\relabeledby S$ to denote the relabeling of $\clustering$ by $S$, in
which $i$ is replaced by $x_i$ in every cluster.  Then if $f$ is a hierarchical
graph clustering method, $f(G)\relabeledby S$ is a clustering of $S$, and
$\inducedheights_{f(G)\relabeledby S}$ denotes the merge function induced on
$f(G)\relabeledby S$ by $M$.
\begin{defn}[Consistency]
    \label{defn:label_measure}
    Let $W$ be a graphon and $M$ be a mergeon of $W$. A hierarchical graph
    clustering method $f$ is said to be a \term{consistent} estimator of the
    graphon cluster tree of $W$ if for any fixed $\epsilon > 0$, as $n \to
    \infty$,
    $
        \labelmeasure\left(
            \left\{
                (G,S) : \mergedist_S(M, \inducedheights_{f(G)\relabeledby S}) > \epsilon
            \right\}
        \right) \to 0.
        $
\end{defn}

The choice of mergeon for the graphon $W$ does not affect consistency, as any
two mergeons of the same graphon differ on a set of measure zero. Furthermore,
consistency is with respect to the random graph model, and not to any particular
graphon representing the model. The following claim, the proof of which
is in \Cref{apx:proofs}, makes this precise.

\begin{restatable}{claim}{claimconsistencympt}
    Let $W$ be a graphon and $\varphi$ a measure preserving transformation. A
    clustering method $f$ is a consistent estimator of the graphon cluster tree
    of $W$ if and only if it is a consistent estimator of the graphon cluster
    tree of $W^\varphi$.
\end{restatable}

\textbf{Consistency and the blockmodel.} If a graph clustering method is
consistent in the sense defined above, it is also consistent in the stochastic
blockmodel; i.e., it ensures strict recovery of the communities with high
probability as the size of the graphs grow large. For instance, suppose $W$ is a
stochastic blockmodel graphon with $\alpha$ along the block-diagonal and $\beta$
everywhere else. $W$ has two clusters at level $\alpha$, merging into one
cluster at level $\beta$. When the merge distortion between the graphon cluster
tree and a clustering is less than $\alpha - \beta$, which will eventually be
the case with high probability if the method is consistent, the two clusters are
totally disjoint in $\clustering$; this implication is made precise by
\Cref{claim:connected_separated} in \Cref{apx:merge_distortion_structure}.

\section{Consistent algorithms}

We now demonstrate that consistent clustering methods exist. We present two
results: First, we show that any method which is capable of consistently
estimating the probability of each edge in a random graph leads to a consistent
clustering method. We then analyze a modification of an existing algorithm to
show that it consistently estimates edge probabilities. As a corollary, we
identify a graph clustering method which satisfies our notion of consistency.
Our results will be for graphons which are piecewise Lipschitz (or weakly
isomorphic to a piecewise Lipschitz graphon):

\begin{restatable}[Piecewise Lipschitz]{defn}{defnlipschitz}
    \label{defn:block_partition}
    We say that $\mathcal B = \{ B_1, \ldots, B_k \}$ is a \term{block
    partition} if each $B_i$ is an open, half-open, or closed interval in
    $[0,1]$ with positive measure, $B_i \cap B_j$ is empty whenever $i \neq
    j$, and $\bigcup \mathcal B = [0,1]$.  We say that a graphon $W$ is
    \term{piecewise $\lipschitz$-Lipschitz} if there exists a set of blocks
    $\mathcal B$ such that for any $(x,y)$ and $(x',y')$ in $B_i \times
    B_j$, $|W(x,y) - W(x',y')| \leq \lipschitz (|x - x'| + |y -y'|)$.
\end{restatable}

Our first result concerns methods which are able to consistently estimate edge
probabilities in the following sense.  Let $\rv{S} = (\rv{x_1}, \ldots,
\rv{x_n})$ be an ordered set of $n$ uniform random variables drawn from the unit
interval. Fix a graphon $W$, and let $\rv{P}$ be the random matrix whose $ij$
entry is given by $W(\rv{x_i}, \rv{x_j})$. We say that $\rv{P}$ is the random
\term{edge probability matrix}.  Assuming that $W$ has structure, it is possible
to estimate $\edgeprobrv$ from a single graph sampled from $W$. We say that an
estimator $\hat{\rv{P}}$ of $\edgeprobrv$ is \term{consistent} in max-norm
if, for any $\epsilon > 0$, $\lim_{n \to \infty} \prob(\max_{i \neq j}
|\rv{P}_{ij} - \hat{\rv{P}}_{ij}| > \epsilon) = 0$. The following non-trivial
theorem, whose proof comprises \Cref{apx:estprob}, states that any estimator
which is consistent in this sense leads to a consistent clustering algorithm:

\begin{restatable}{theorem}{thmconditions}
    \label{thm:conditions}
    Let $W$ be a piecewise $\lipschitz$-Lipschitz graphon.
    Let $\hat{\rv{P}}$ be a consistent estimator of $\rv{P}$ in max-norm.
    Let $f$ be the clustering method which performs single-linkage clustering
    using $\hat{\rv{P}}$ as a similarity matrix.  Then $f$ is a consistent
    estimator of the graphon cluster tree of $W$.
\end{restatable}

\begin{wrapfigure}{r}{0.5\textwidth}
    \iftoggle{arxiv}{
        \vspace{-1.5em}
    }{
        \vspace{-2.3em}
    }
\begin{minipage}{0.5\textwidth}
  \begin{algorithm}[H]
    \caption{
    \label{fig:algorithm}
    Clustering by nbhd. smoothing}
    \begin{algorithmic}
        \REQUIRE Adjacency matrix $A$, $C \in (0,1)$
        \STATE \emph{%
            \% Step 1: Compute the estimated edge\\
            \% probability matrix $\hat P$ using neighborhood\\
            \% smoothing algorithm based on \cite{Zhang2015-ik}
        }
        \STATE $n \gets \operatorname{\textsc{Size}}(A)$
        \STATE $h \gets C\sqrt{(\log n) / n}$
        \FOR{$i \neq j \in [n] \times [n]$}
             \STATE{$\hat A \gets A$ after setting row/column $j$ to zero}
             \FOR{$i' \in [n] \setminus \{i,j\}$}
             \STATE $d_j(i,i') \gets \max_{k \neq i,i',j} 
                 |(\hat A^2/n)_{ik} - (\hat A^2/n)_{i'k}|$
             \ENDFOR
             \STATE $q_{ij} \gets $ $h$th quantile of $\{d _j(i,i') : i' \neq i,j \}$
             \STATE $N_{ij} \gets \{ i' \neq i,j : d_j(i,i') \leq q_{ij}(h) \}$
         \ENDFOR
         \FOR{$(i,j) \in [n]\times[n]$}
             \STATE $\hat{P}_{ij} \gets \frac{1}{2}\left( 
                 \frac{1}{N_{ij}} \sum_{i' \in N_{ij}} A_{i'j} + 
                 \frac{1}{N_{ji}} \sum_{j' \in N_{ji}} A_{ij'}
             \right)$
         \ENDFOR
         \STATE \emph{\% Step 2: Cluster $\hat P$ with single linkage}
         \STATE $\clustering \gets$ the single linkage clusters of $\hat P$
         \RETURN $\clustering$
    \end{algorithmic}
  \end{algorithm}
\end{minipage}
\vspace{-1em}
\end{wrapfigure}

Estimating the matrix of edge probabilities has been a direction of recent
research, however we are only aware of results which show consistency in mean
squared error; That is, the literature contains estimators for which
$\nicefrac{1}{n^2}\|\rv{P} - \hat{\rv{P}}\|_F^2$ tends to zero in probability.
One practical method is the neighborhood smoothing algorithm of
\cite{Zhang2015-ik}.  The method constructs for each node $i$ in the graph
$\rv{G}$ a neighborhood of nodes $\rv{N_i}$ which are similar to $i$ in the
sense that for every $i' \in \rv{N_i}$, the corresponding column $\rv{A_{i'}}$
of the adjacency matrix is close to $\rv{A_i}$ in a particular distance.  $\rv{A_{ij}}$
is clearly not a good estimate for the probability of the edge $(i,j)$, as it is
either zero or one, however, if the graphon is piecewise Lipschitz, the average
$\rv{A_{i'j}}$ over $i' \in \rv{N_{ij}}$ will intuitively tend to the true
probability.  Like others, the method of \cite{Zhang2015-ik} is proven to be
consistent in mean squared error.  Since Theorem~\ref{thm:conditions} requires
consistency in max-norm, we analyze a slight modification of this algorithm
and show that it consistently estimates $\edgeprobrv$ in this stronger sense.
The technical details are in \Cref{apx:smoothing}.

\begin{theorem}
    \label{thm:linftyestimator}
    If the graphon $W$ is piecewise Lipschitz, the modified neighborhood
    smoothing algorithm in \Cref{apx:smoothing} is a consistent estimator of
    $\rv{P}$ in max-norm.
\end{theorem}

As a corollary, we identify a practical graph clustering algorithm
which is a consistent estimator of the graphon cluster tree. The algorithm
is shown in \Cref{fig:algorithm}, and details are in 
\Cref{apx:smoothing:modifications}. 
\Cref{apx:experiments} contains experiments in which the algorithm is applied to
real and synthetic data.

\begin{corollary}
    If the graphon $W$ is piecewise Lipschitz, \Cref{fig:algorithm} is a
    consistent estimator of the graphon cluster tree of $W$.
\end{corollary}

\section{Discussion}

We have presented a consistency framework for clustering in the graphon model
and demonstrated that a practical clustering algorithm is consistent. We now
identify two interesting directions of future research. First, it would be
interesting to consider the extension of our framework to \emph{sparse} random
graphs; many real-world networks are sparse, and the graphon generates dense
graphs. Recently, however, sparse models which extend the graphon have been
proposed; see \cite{Caron2014-pw,Borgs2016-uv}. It would be interesting to see
what modifications are necessary to apply our framework in these models.

Second, it would be interesting to consider alternative ways of defining the
ground truth clustering of a graphon. Our construction is motivated by
interpreting the graphon $W$ not only as a random graph model, but also as a
similarity function, which may not be desirable in certain settings.  For
example, consider a ``bipartite'' graphon $W$, which is one along the
block-diagonal and zero elsewhere. The cluster tree of $W$ consists of a single
cluster at all levels, whereas the ideal bipartite clustering has two clusters.
Therefore, consider applying a transformation $S$ to $W$ which maps it to a
``similarity'' graphon.  The goal of clustering then becomes the recovery of the
cluster tree of $S(W)$ given a random graph sampled from $W$.  For instance, let
$S : W \mapsto W^2$, where $W^2$ is the operator square of the bipartite graphon
$W$.  The cluster tree of $S(W)$ has two clusters at all positive levels, and so
represents the desired ground truth. In general, any such transformation $S$
leads to a different clustering goal. We speculate that, with minor
modification, the framework herein can be used to prove consistency results in a
wide range of graph clustering settings.

\enlargethispage{1em}

\textbf{Acknowledgements.} This work was supported by NSF grants IIS-1550757
\& DMS-1547357.

\newpage
\small
\bibliography{citations}

\normalsize
\appendix
\newpage

\section{Technical details}

\subsection{Measurable sets modulo null sets}

\newcommand{\eqzero}{\sim_\emptyset}

Let $(\Omega, \Sigma, \mu)$ be a measure space. Let $A, A'$ be any measurable
sets and define $\eqzero$ to be the relation $A \eqzero A' \Leftrightarrow \mu(A
\symdiff A') = 0$; that is, two measurable sets are equivalent under $\eqzero$
if they differ by a null set. Write $\Sigma / {\eqzero}$ for the quotient space
of measurable sets by $\eqzero$, and denote by $\zeroclass{A}$ the equivalence
class containing the set $A$. Throughout, we use script letters such as
$\clusterstyle{A}$ to denote these equivalence classes of measurable sets modulo
null sets.

We can often use the normal set notation to manipulate such classes without
ambiguity. For instance, if $\clusterstyle{A}$ and $\clusterstyle{A}'$ are
two classes in $\Sigma / {\eqzero}$, we define $\clusterstyle{A} \cup
\clusterstyle{A}'$ to be $\zeroclass{A \cup A'}$, where $A$ and $A'$ are
arbitrary members of $\clusterstyle{A}$ and $\clusterstyle{A'}$.
$\clusterstyle{A} \cap \clusterstyle{A}'$ and $\clusterstyle{A} \setminus
\clusterstyle{A}'$ are defined similarly. We can define $\clusterstyle{A} \times
\clusterstyle{A}'$ in this manner too; note that the result is an equivalence
class in $\Sigma \times \Sigma / {\eqzero}$, where the relation $\eqzero$ is
implicitly assumed to be with respect to the product measure $\mu \times \mu$.
Similarly, we can unambiguously order such equivalence classes. For example, we
write $\clusterstyle{A} \subset \clusterstyle{A}'$ to denote
$\mu(\clusterstyle{A} \setminus \clusterstyle{A}') = 0$.

In some instances it will be more convenient to work with sets as opposed to
equivalence classes of sets. In such cases we will use a \term{section} map
$\rho$ which returns an (often arbitrary) member of the class,
$\rho(\clusterstyle{A})$.

\subsection{A precise definition of a mergeon}
\label{apx:precise_mergeon}

In \Cref{defn:mergeon}, we introduce the \emph{mergeon} of a cluster tree
$\clustertree$ as a graphon $M$ satisfying 
\[
        M^{-1}[\lambda,1]
        =
        \bigcup_{\cluster \in \clustertree_W(\lambda)} 
        \cluster \times \cluster
\]
for all $\lambda \in [0,1]$. This definition involves a slight abuse of
notation. In particular, $\cluster$ is an equivalence
class of sets modulo null sets. Therefore, as descibed in the previous
subsection, $\cluster \times \cluster$ is defined to be an equivalence class of
measurable subsets of $[0,1] \times [0,1]$ modulo null sets. On the other hand,
$M^{-1}[\lambda,1]$ is simply a measurable subset of $[0,1]\times[0,1]$.
Therefore, the left and the right of the above equation are two different types
of objects, and we are being imprecise in equating them.

% In the journal version of the paper, I intend to define the mergeon as an
% *equivalence class* of functions. Then the natural definition of the upper-level
% set of M is as an equivalence class of sets modulo null sets, and we can in fact
% use equality here.

A precise statement of this definition is as follows:

\begin{defn}[Mergeon, rigorous]
    A \term{mergeon} of $\clustertree$ is a graphon $M$ such that for all
    $\lambda \in [0,1]$,
    \[
            M^{-1}[\lambda,1] \symdiff
            \bigcup_{\cluster \in \clustertree_W(\lambda)}
            \rho(\cluster) \times \rho(\cluster)
    \]
    is a null set, where $M^{-1}[\lambda,1] = \{(x,y) : M(x,y) \geq \lambda\}$,
    $\symdiff$ is the symmetric difference operator, and $\rho$ is an arbitrary
    section map. Equivalently, a \term{mergeon} of $\clustertree$ is a graphon
    $M$ such that for all $\lambda \in [0,1]$,
    \[
            \zeroclass{M^{-1}[\lambda,1]}
            =
            \bigcup_{\cluster \in \clustertree_W(\lambda)}
            \cluster \times \cluster,
    \]
    where $\zeroclass{M^{-1}[\lambda,1]}$ is the equivalence class 
    of measurable subsets of $[0,1] \times [0,1]$ modulo null sets which
    contains $M^{-1}[\lambda,1]$.
\end{defn}

In these more rigorous definitions we compare sets to sets, or equivalence
classes to equivalence classes, and thus precisely define the mergeon.

\subsection{Strict cluster trees and their mergeons}
\label{apx:strict_mergeon}

A graphon cluster tree is a hierarchical collection of equivalence classes of
sets. It is sometimes useful to instead to work with a hierarchical collection
of subsets of $[0,1]$. We may always do so by choosing a section map $\rho$ and
applying it to every cluster in the cluster tree.  Though the choice of
representative of a given cluster is often arbitrary, it will sometimes be
useful to choose it in such a way that the cluster tree has strictly nested
structure, as made precise in the following definition.

\begin{defn}[Strict section]
    Let $\clustertree$ be a cluster tree. A \term{strict section}
    $\strictsection : \clustertree \to \Sigma$ is a function which selects a
    unique representative from each cluster $\cluster$ such that if: 
    \begin{enumerate}
        \item
        $\mu(\cluster \cap \cluster') = 0 
        \Rightarrow 
        \strictsection(\cluster) \cap \strictsection(\cluster) = \emptyset$, 
        \item
            $\cluster \subset \cluster' \Rightarrow
            \strictsection(\cluster) \subset \strictsection(\cluster')$, and 
        \item (Technical condition)\;
        $
        \strictsection(\cluster) = \bigcap \{
            \strictsection(\cluster') : \cluster' \supset \cluster
        \}$.
    \end{enumerate}
    The \term{strict cluster tree} $\tilde\clustertree$ induced by applying
    $\strictsection$ to $\clustertree$ is defined by
    $
        \tilde\clustertree(\lambda) = \{ \,
            \strictsection(\cluster) : \cluster \in \clustertree(\lambda)
        \}.
    $
\end{defn}

\Cref{claim:strict_section} in \Cref{apx:proofs} proves that it is always
possible to construct a strict section.  Furthermore, given a cluster tree and a
strict section, there is a \emph{unique} mergeon representing the strict cluster
tree, defined as follows:

\begin{defn}[Strict mergeon]
    Let $\clustertree$ be a cluster tree, and suppose $\strictsection$ is a
    strict section for the clusters of $\clustertree$. Then $M$ is a
    \term{strict mergeon} of the strict cluster tree induced by $\strictsection$
    if, for every $\lambda \in [0,1]$,
    \[
        M^{-1}[\lambda,1] = 
            \bigcup_{\cluster \in \clustertree_W(\lambda)}
            \strictsection({\cluster})
            \times 
            \strictsection({\cluster}).
    \]
\end{defn}

Because any two mergeons of the same cluster tree differ only on a null set, we
are typically free to assume that a mergeon is strict without much loss. Making
this assumption will simplify some statements and proofs.

\subsection{Merge distortion and cluster structure}
\label{apx:merge_distortion_structure}

\Cref{defn:induced_merge_height} introduces the merge height induced on a
clustering by a mergeon. There are, of course, other approachs to inducing a
merge height on a clustering, but our definition allows for a particular
interpretation of the merge distortion in terms of the cluster structure that is
recovered by the finite clustering, as the following claim makes precise.
It is convenient to state the claim using the notion of a \term{strict mergeon}
defined in \Cref{apx:strict_mergeon}; analogous (but equally as strong)
claims can be made for general mergeons and cluster trees.

\begin{restatable}{claim}{claimconnectedseparated}
    \label{claim:connected_separated}
    Let $\clustertree$ be cluster tree, and let $\strictclustertree$ be a strict
    cluster tree obtained by applying a strict section $\strictsection$ to each
    cluster of $\clustertree$.  Let $M$ be the strict mergeon representing
    $\strictclustertree$.  Let $S = (x_1, \ldots, x_n)$ with each $x_i \in
    [0,1]$ and suppose $\clustering$ is a clustering of $S$. Let
    $\inducedheights$ be the induced merge height on $\clustering$. If
    $\mergedist_S(M, \inducedheights) < \epsilon$, we then have:
    \begin{enumerate}
        \item \term{Connectedness}: If $C$ is a cluster of $\strictclustertree$
            at level $\lambda$ and $|S \cap C| \geq 2$, then the smallest
            cluster in $\clustering$ which contains all of $S \cap C$ is
            contained within $C' \cap S$, where $C'$ is the cluster of
            $\strictclustertree$ at level $\lambda' = \lambda - \epsilon$
            which contains $C$.
        \item \term{Separation}: If $C_1$ and $C_2$ are two clusters of
            $\strictclustertree$ at level $\lambda$ such that $C_1$ and $C_2$
            merge at level $\lambda' < \lambda - \epsilon$, then if $|C_1 \cap
            S|, |C_2 \cap S| \geq 2$, the smallest
            cluster in $\clustering$ containing $C_1 \cap S$ and the smallest
            cluster containing $C_2 \cap S$ are disjoint.
    \end{enumerate}
\end{restatable}

The proof of this claim is found in \Cref{apx:proofs}.

\section{Proofs}
\label{apx:proofs}

In the following proofs we will often work with the clusters of a graphon.
Recall that, formally, a cluster is not a subset of $[0,1]$, but
rather an equivalence class of measurable subsets of $[0,1]$ which differ by
null sets. Nevertheless, we will often speak as though the clusters are in fact
subsets of $[0,1]$; we can typically do so without issue. For instance, we might
say ``$C \subset [0,1]$ is a cluster of a graphon $W$ at level $\lambda$.''
Technically-speaking, this is incorrect. However, we might interpret the above
statement as saying that there exists a cluster $\cluster$ at level $\lambda$
such that $C$ is a representative of $\cluster$.

\newcommand{\A}{\mathfrak A}

\claimmergeonequivalence*

\begin{proof}

    Let $C$ be an arbitrary representative of the cluster $\cluster$.
    By definition of the
    mergeon, all but a null set of $C \times C$ is contained within
    $M^{-1}[\lambda, 1]$, and therefore $M \geq \lambda$ almost everywhere on $C
    \times C$. This implies that $C$ is connected at level $\lambda$ in $M$,
    which in turn implies that $C$ is contained in some cluster $C'$ of $M$ at
    level $\lambda$.  By definition, $C'$ is connected in $M$ at every level
    $\lambda' < \lambda$,  and so
    Claim~\ref{claim:mergeon_connected_contained_in_cluster} in
    \Cref{apx:proofs} implies that $C'$ is contained in some cluster of $W$ at
    every level $\lambda' < \lambda$.  Claim~\ref{claim:cluster_completion} in
    \Cref{apx:proofs} then implies that $C'$ is contained in some cluster of $W$
    at level $\lambda$. In other words, $C$ is a cluster of $W$ at level
    $\lambda$, and $C \subset C'$, so the fact that $C'$ is contained in a
    cluster of $W$ at level $\lambda$ implies that $C'$ differs from $C$ by at
    most a null set. Hence $C$ is a cluster of $M$.

    Now suppose $C$ is a cluster of $M$ at level $\lambda$. Then $C$ is
    connected in $M$ at every level $\lambda' < \lambda$,  and so
    Claim~\ref{claim:mergeon_connected_contained_in_cluster} implies that $C$ is
    contained in some cluster of $W$ at every level $\lambda' < \lambda$.
    Claim~\ref{claim:cluster_completion} then implies that $C$ is contained in
    some cluster of $W$ at level $\lambda$. Let $C'$ be this cluster. Then the
    above implies that $C'$ is a cluster at level $\lambda$ in $M$. But $C
    \subset C'$, and $C$ is a cluster of $M$, so it must be that $C$ and $C'$
    differ by a null set, and hence $C$ is a cluster of $W$.

\end{proof}

\claimmptmergeon*

\begin{proof}
    One one hand, the function defined by
    \[
        M_0^{-1}[\lambda,1] = \bigcup_{C \in \clustertree_W(\lambda)} 
        \varphi^{-1}(C) \times \varphi^{-1}(C)
    \]
    is a mergeon of $W^\varphi$, since $C$ is a cluster of $W$ if and only if
    $\varphi^{-1}(C)$ is a cluster of $W^\varphi$ by \Cref{claim:mptclusters} in
    the body of the paper.  Now consider the pullback $M^\varphi$ and its upper
    level set 
    \begin{align*}
        (M^\varphi)^{-1}[\lambda,1] 
        &= \{ (x,y): M^\varphi(x,y) \geq \lambda\}\\
        &= \{(x,y) : M(\varphi(x), \varphi(y)) \geq \lambda\},
        \intertext{which, by definition of the mergeon, is}
        &= \left\{ (x,y) : (\varphi(x),
            \varphi(y)) \in \bigcup_{C \in \clustertree_W(\lambda)} C \times C 
            \right\}.
        \intertext{It is well-known that if $\varphi$ is a measure preserving
            map, then the transformation defined by $\Phi : (x,y) \mapsto
            (\varphi(x), \varphi(y))$ is also measure preserving and measurable.
            Therefore we have}
        &= \Phi^{-1}\left( \,
            \bigcup_{C \in \clustertree_W(\lambda)} C \times C 
            \right).
        \intertext{Since preimages commute with arbitrary unions:}
        &=  \bigcup_{C \in \clustertree_W(\lambda)} \Phi^{-1}(C \times C)
        \intertext{Some thought will show that $\Psi^{-1}(C \times C) =
            \varphi^{-1}(C) \times \varphi^{-1}(C)$, such that:}
        &= \bigcup_{C \in \clustertree_W(\lambda)}
            \varphi^{-1}(C) \times \varphi^{-1}(C)
    \end{align*}
    Comparing this to the definition of $M_0$ above, which was a mergeon of
    $W^\varphi$, we see that $M^\varphi$ is a mergeon of $W^\varphi$.

\end{proof}

\claimconsistencympt*

\begin{proof}
    Let $M$ be a mergeon of the cluster tree of $W$ and fix any $\epsilon > 0$.
    Consider the set
    \[
        F = \left\{
            (G,S) : \mergedist_S\left(
                M, \inducedheights_{f(G)\relabeledby S}
            \right) >
            \epsilon
        \right\},
    \]
    which is the set of graph/sample pairs for which the merge distortion
    between the clustering and the mergeon $M$ is greater than $\epsilon$. 
    Consistency with respect to the cluster tree of $W$ requires that
    $\labelmeasurechar_{W,n}(F) \to 0$ as $n \to \infty$.
    Now recall that $M^\varphi$ is a mergeon of the cluster tree of $W^\varphi$,
    and consider
    \[
        F_\varphi = \left\{
            (G,S) : \mergedist_S\left(
                M^\varphi, \inducedheights^\varphi_{f(G)\relabeledby S}
        \right) > \epsilon
        \right\}
    \]
    where $\hat{M^\varphi}_{f(G)\relabeledby S}$ is the merge height induced on
    the clustering $f(G)\relabeledby S$ by the mergeon $M^\varphi$.
    $F_\varphi$ is the set of graph/sample pairs for which the merge distortion
    between the clustering and the mergeon $M^\varphi$ is greater than
    $\epsilon$.  Consistency with respect to the cluster tree of $W^\varphi$
    requires that $\labelmeasurechar_{W^\varphi, n}(F_\varphi) \to 0$ as $n \to
    \infty$.  It will therefore be sufficient to show that
    $\labelmeasurechar_{W,n}(F) = \labelmeasurechar_{W^\varphi,n}(F_\varphi)$ to
    prove the claim.

    Now we compute the measure under $\labelmeasurechar_{W^\varphi,n}$ of
    $F_\varphi$:
    \[
        \labelmeasurechar_{W^\varphi,n}(F_\varphi) = 
        \sum_{G \in \allgraphs_n} \int_{F_\varphi(G)}
        \likelihood_{W^\varphi}(S\mid G) \,dS,
    \]
    where $F_\varphi(G)$ denotes the \emph{section} of $F_\varphi$ by graph $G$,
    that is, the set $F_\varphi(G) = \{ S : (G, S) \in F_\varphi \}$. It is easy
    to see that $\likelihood_{W^\varphi}(S \mid G) = \likelihood_W(\varphi(S),
    G)$, such that:
    \[
        \labelmeasurechar_{W^\varphi,n}(F_\varphi) = 
        \sum_{G \in \allgraphs_n} \int_{F_\varphi(G)}
        \likelihood_{W}(\varphi(S)\mid G) \,dS,
    \]

    Since $M^\varphi(x,y) = M(\varphi(x), \varphi(y))$, we have
    \[
        \mergedist_S\left(M^\varphi, \hat{M}^\varphi_{f(G)\relabeledby S}\right) = 
        \mergedist_{\varphi(S)} \left(M, \hat{M}_{f(G)\relabeledby
        \varphi(S)}\right)
    \]
    such that
    \[
        F_\varphi = \left\{
            (G,S) : 
            \mergedist_{\varphi(S)} \left(M, \hat{M}_{f(G)\relabeledby
            \varphi(S)}\right) > \epsilon
        \right\}.
    \]
    Now consider the section of $F$ by $G$, defined by $F(G) = \{ S : (G,S) \in
    F \}$. It is clear that $F_\varphi(G) = \varphi^{-1}(F(G))$ for every graph $G$.
    Therefore,
    \[
        \labelmeasurechar_{W^\varphi,n}(F_\varphi) = 
        \sum_{G \in \allgraphs_n} \int_{\varphi^{-1}(F(G))}
        \likelihood_{W}(\varphi(S)\mid G) \,dS.
    \]

    Now, it is a property of measure preserving maps
    that $\int_{\varphi^{-1}(A)} f(\varphi(x)) \, d\mu(x) = \int_{A} f(x) \,
    d\mu(x)$; See, for example, \cite{Ash2000-gh}.  Therefore, we
    have
    \begin{align*}
        \labelmeasurechar_{W^\varphi,n}(F_\varphi) 
        &= 
        \sum_{G \in \allgraphs_n} \int_{\varphi^{-1}(F(G))}
        \likelihood_{W}(\varphi(S)\mid G) \,dS\\
        &= 
        \sum_{G \in \allgraphs_n} \int_{F(G)}
        \likelihood_{W}(S \mid G) \,dS\\
        &= \labelmeasurechar_{W,n}(F)
    \end{align*}
    which proves the claim.

\end{proof}

\claimconnectedseparated*

\begin{proof}
    First we prove connectedness.
    Let $\hat C$ be the smallest cluster in $\strictclustertree$ containing $C
    \cap S$. Suppose $\hat C$ contains a point $y$ from outside of $C'$. Let
    $x,x'$ be any two distinct points in $C \cap S$. Then necessarily $M(x,y) <
    \lambda' = \lambda - \epsilon$, as, since $M$ is strict, $M(x,y) \geq
    \lambda$ if and only if $x,y$ are in the same cluster of
    $\strictclustertree$ at level $\lambda$. Hence the merge distortion is at
    least $M(x,x') - M(x,y) > \epsilon$, which is a contradiction.

    Separation follows from connectedness. Let $\hat C_1$ be the smallest
    cluster in the clustering containing $C_1 \cap S$, and similarly for $\hat
    C_2$. Let $\tilde C_1$ and $\tilde C_2$ be the clusters at level $\lambda -
    \epsilon$ which contain $C_1$ and $C_2$. Then $\tilde C_1 \cap \tilde C_2 =
    \emptyset$, since $C_1$ and $C_2$ merge below $\lambda - \epsilon$.
    Furthermore, by connectedness, $C_1 \cap S \subset \tilde C_1$ and $C_2 \cap
    S \subset \tilde C_2$. Hence they are disjoint.
\end{proof}

\begin{claim}
    \label{claim:essmax}
    Let $(\Omega, \Sigma, \mu)$ be a measure space, with $\mu$ a finite measure
    (i.e., $\mu(\Omega) < \infty$). Let $\A \subset \Sigma$ be closed under
    countable unions. Define the set of \term{essential maxima} of $\A$ by
    \[
        \essmax \A = \{ M \in \A : \mu(A \setminus M) = 0 \quad \forall A \in \A \}.
    \]
    Then if $\A$ is nonempty, $\essmax \A$ is nonempty.  Furthermore, for any
    $M, M' \in \essmax \A$, $\mu(M \symdiff M') = 0$.  
\end{claim}

\begin{proof}
    The claim holds trivially if $\A$ is empty, so suppose it is not. Let $\tau
    = \sup_{A \in \A} \mu(A)$, and note that $\tau$ is finite since
    $\mu(\Omega)$ is finite. Then for every $n \in \positivenaturals$, there
    exists a set $A_n \in \A$ such that $\tau - \mu(A_n) < 1/n$. Construct the
    sequence $\langle B_n \rangle_{n \in \positivenaturals}$ by defining
    $
        B_n = \bigcup_{i=1}^n A_i.
    $
    Then $B_n \in \A$ for every $n$, since it is the countable union of sets in
    $\A$. Furthermore, $B_n \subseteq B_{n+1}$, and $\lim_{n \to \infty}
    \mu(B_n) = \tau$. Define $M = \bigcap_{n=1}^\infty B_n =
    \bigcap_{n=1}^\infty A_n$. Then $M \in \A$ since it is a countable union of
    elements of $\A$, and by continuity of measure $\mu(M) = \tau$.

    First we show that for any set $A \in \A$, $\mu(A \setminus M) = 0$. Suppose
    for a contradiction that $\mu(A \setminus M) \neq 0$. We have $A \cup M = (A
    \setminus M) \cup M$, such that $\mu(A \cup M) = \mu(A \setminus M) + \mu(M)
    > \tau$. But $A \cup M$ is in $\A$, since $\A$ is closed under unions. This
    violates the fact that $\tau$ is the supremal measure of any set in $\A$,
    and hence it must be that $\mu(A \setminus M) = 0$. Therefore $M \in \essmax
    \A$.

    Now suppose $M'$ is an arbitrary element of $\essmax \A$. We have just seen
    that $\mu(M' \setminus M)$ must be zero, since $M' \in \A$. Likewise, $\mu(M
    \setminus M') = 0$. Therefore
    \[
        \mu(M \symdiff M') 
        = \mu((M \setminus M') \cup (M' \setminus M))
        = \mu(M \setminus M') + \mu(M' \setminus M)
        = 0
    \]
    where we used the fact that $\mu$ is an additive set function and $M
    \setminus M'$ and $M' \setminus M$ disjoint. It is also clear that if $M
    \in \essmax \A$ and $N$ is any null set, then $M \cup N$ and $M \setminus N$
    are also essential maxima.
\end{proof}

\begin{claim}
    \label{claim:essmin}
    Let $(\Omega, \Sigma, \mu)$ be a measure space, with $\mu$ a finite measure
    (i.e., $\mu(\Omega) < \infty$). Let $\A \subset \Sigma$ be closed under
    countable intersections. Define the set of \term{essential minima} of $\A$
    by
    \[
        \essmin \A = \{ M \in \A : \mu(M \setminus A) = 0 \quad \forall A \in \A \}.
    \]
    Then if $\A$ is nonempty, $\essmin \A$ is nonempty.  Furthermore, for any
    $M, M' \in \essmin \A$, $\mu(M \symdiff M') = 0$.  
\end{claim}

\begin{proof}

    The proof is analogous the that of \Cref{claim:essmax} for $\essmax$ -- we
    simply switch $\tau$ from a supremum to an infimum and construct a
    descending sequence. It is therefore omitted.

\end{proof}

\begin{claim}
    \label{claim:connected_union_connected}
    Let $W$ be a graphon, and suppose $A$ and $A'$ are measurable sets with
    positive measure, and that each is connected at level $\lambda$ in $W$. 
    If $\mu(A \cap A') > 0$, then $A \cup A'$ is connected at level $\lambda$.
\end{claim}

\begin{proof}
    Suppose $\mu(A \cap A') > 0$ and that $A \cup A'$ is disconnected at level
    $\lambda$. Then, by definition, there exists a measurable set $S \subset A
    \cup A'$ such that $0 < \mu(S) < \mu(A \cup A')$ and $W < \lambda$ almost
    everywhere on $S \times ((A \cup A') \setminus S)$.

    It is either the case that $0 < \mu(A \cap S) < \mu(A)$ or $0 < \mu(A \cap
    S') < \mu(S')$, as otherwise we would have $\mu(S) = \mu(A \cup A')$. If $0
    < \mu(A \cap S) < \mu(A)$, then $A \setminus S$ is of positive measure.
    Since $W < \lambda$ almost everywhere on $S \times ((A \cup A') \setminus
    S)$, it follows that $W < \lambda$ almost everywhere on $S \times (A
    \setminus S)$. This implies that $A$ is not connected at level $\lambda$.
    Likewise, if $0 < \mu(A' \cap S) < \mu(A')$, it follows that $W < \lambda$
    almost everywhere on $S \times (A' \setminus S)$, and hence $A'$ is
    disconnected at level $\lambda$. Both cases lead to contradictions, and so
    it must be that $A \cup A'$ is connected at level $\lambda$.
\end{proof}

\begin{claim}
    \label{claim:connected_equivalence}
    The relation $\clusteredwith_\lambda$ is an equivalence relation on
    $\upperlevel$.
\end{claim}

\begin{proof}

    The symmetry and reflexive properties of $\clusteredwith_\lambda$ are clear.
    We need only prove that $\clusteredwith_\lambda$ is transitive. Suppose $A_1
    \clusteredwith_\lambda A_2$ and $A_2 \clusteredwith_\lambda A_3$. By the
    definition of $\clusteredwith_\lambda$, there exist measurable sets $B_{12},
    B_{23} \in \upperlevel$ such that $B_{12} \supset A_1 \cup A_2$ and $B_{23}
    \supset A_2 \cup A_3$. Since both $B_{12}$ and $B_{23}$ contain $A_2$, a set
    of positive measure, their intersection is not null. Furthermore, $B_{12}$
    and $B_{23}$ are each connected at every level $\lambda' < \lambda$, by
    virtue of being in $\upperlevel$. Hence we can use
    Claim~\ref{claim:connected_union_connected} in \Cref{apx:proofs} to conclude
    that their union $B_{12} \cup B_{23}$ is connected at every level $\lambda'
    < \lambda$, and is hence an element of $\upperlevel$. Since $A_1 \cup A_3
    \subset B_{12} \cup B_{23}$, we have $A_1 \clusteredwith_\lambda A_3$.

\end{proof}

\renewcommand{\eqclass}{\mathscr C}
\begin{claim}
    \label{claim:essmax_cluster}
    Let $\eqclass$ be an equivalence class in
    $\upperlevel/\clusteredwith_\lambda$. Then $\essmax \eqclass$ is
    well-defined and non-empty.
\end{claim}

\begin{proof}
    \newcommand{\class}{\mathscr{F}}

    We will invoke Claim~\ref{claim:essmax} in \Cref{apx:proofs} to show that
    $\essmax \eqclass$ has the desired properties. To do so, we need only show
    that $\eqclass$ is closed under countable unions. Let $\class \subset
    \eqclass$ be a countable subset of $\eqclass$, and let $F = \bigcup \class$.
    We will show that $F$ is connected at every level $\lambda' < \lambda$ and
    is thus contained in $\class$ by its definition.

    Suppose $F$ is disconnected at some level $\lambda' < \lambda$. Then there
    exists a set $S \subset F$ such that $0 < \mu(S) < \mu(F)$ and $W <
    \lambda'$ almost everywhere on $S \times (F \setminus S)$. Now, there must
    exist sets $F_1, F_2 \in \class$ such that $\mu(S \cap F_1) > 0$ and $\mu((F
    \setminus S) \cap F_2) > 0$. Let $F_{12} = F_1 \cup F_2$. Note that $\class$
    is closed under finite union by the definition of $\clusteredwith_\lambda$,
    and so $F_{12} \in \class$, meaning that $F_{12}$ is connected at
    $\lambda'$. Furthermore, $F_{12} \subset F$, so that $(F_{12} \cap S) \cup
    (F_{12} \cap (F \setminus S)) = F_{12}$. But by assumption $W < \lambda'$
    almost everywhere on $S \times (F \setminus S)$, and so in particular $W <
    \lambda'$ almost everywhere on $(F_{12} \cap S) \times (F_{12} \cap (F
    \setminus S))$. But this implies that $F_{12}$ is disconnected at level
    $\lambda'$, which is a contradiction. It must therefore be the case that $F$
    is connected at every $\lambda' < \lambda$, and so $F \in \upperlevel$.

    It is clearly true that for all $F' \in \class$, $F' \clusteredwith_\lambda
    F$, since $F = F \cup F'$. We therefore have that $F \in \class$.

\end{proof}

\begin{claim}
    Let $W$ be a graphon, and suppose $C_1, C_2 \in \clustertree_W(\lambda)$ are
    two distinct clusters at level $\lambda$, such that $\mu(C_1 \cap C_2) = 0$.
    Let $M$ be a mergeon of $W$. Then $M < \lambda$ almost everywhere on $C_1
    \times C_2$.
\end{claim}

\begin{proof}
    Consider the set
    \begin{align*}
        (C_1 \times C_2) \cap M^{-1}[\lambda,1]
        &= (C_1 \times C_2) \cap \left(\,
        \bigcup_{C \in \clustertree_W(\lambda)} C \times C
            \right)\\
        &= \bigcup_{C \in \clustertree_W(\lambda)}
            (C_1 \times C_2) \cap (C \times C)\\
        &= \bigcup_{C \in \clustertree_W(\lambda)}
            (C_1 \cap C) \times (C_2 \cap C)
    \end{align*}
    Because distinct clusters intersect only on a null set, we can have either
    $C_1 \cap C$ or $C_2 \cap C$ being non-negligible, but not both
    simultaneously for the same choice of $C$. Hence every set in the union is a
    null set, and since $\clustertree_W(\lambda)$ is countable, the union as a
    whole is a null set. Therefore $\mu\times\mu(M^{-1}[\lambda,1] \cap (C_1
    \times C_2)) = 0$, and so it must be that $M < \lambda$ almost everywhere on
    $C_1 \times C_2$.
\end{proof}

\begin{claim}
    \label{claim:mergeon_connected_contained_in_cluster}
    Let $W$ be a graphon and suppose $M$ is a mergeon of $W$. Suppose $A$ is
    connected at level $\lambda$ in $M$. Then $A$ is contained in some cluster
    $C$ in $W$ at level $\lambda$.
\end{claim}

\begin{proof}
    First, it must be the case that $\mu\left(A \setminus \bigcup
    \clustertree_W(\lambda)\right) = 0$. Suppose not. Let $R = A \setminus
    \bigcup \clustertree_W(\lambda)$. Since $A$ is connected, it follows that
    there is a subset $Q \subset R \times (A \setminus R)$ of positive measure
    such that $M \geq \lambda$ on $Q$. We then have that
    \begin{align*}
        Q \cap M^{-1}[\lambda, 1] = Q \cap \bigcup_{C \in
        \clustertree_W(\lambda)} C \times C
    \end{align*}
    is not null. Since $\clustertree_W(\lambda)$ is a countable set, it follows
    that there must be a cluster $C \in \clustertree_W(\lambda)$ such that $Q
    \cap (C \times C)$ is not null. But $Q \subset R \times (A \setminus R)$, so
    this implies that $\left(R \times (A \setminus R)\right) \cap (C \times C)$
    is not null. We have the identity:
    \[
        \left(R \times (A \setminus R)\right) \cap (C \times C)
        =
        (R \cap C) \times ((A \setminus R) \cap C),
    \]
    which then implies that $(R \cap C) \times ((A \setminus R) \cap C)$ is
    not null. However, this is a contradiction, since $R \cap C$ is necessarily
    a set of measure zero by the definition of $R$. Hence it must be that all of
    $A$ excluding a null set is contained within $\bigcup
    \clustertree_W(\lambda)$. 

    Let $\hat A = A \cap \bigcup \clustertree_W(\lambda)$. Then $\hat A$ is
    equivalent to $A$ in that it differs only by a set of measure zero, however,
    it is contained entirely within $\bigcup \clustertree_W(\lambda)$. Since
    $\hat A$ is a set of positive measure, there must exist a $\hat C \in
    \clustertree_W(\lambda)$ such that $\mu(\hat A \cap \hat C) > 0$. We will
    show that $\mu(\hat A \setminus \hat C) = 0$.
    
    Let $S = \hat A \cap \hat C$, and let $T = \hat A \setminus S$. Note that $T
    \cap \hat C$ is null. Suppose for a contradiction that $T$ is not null.
    Since $T$ is contained within $\bigcup \clustertree_W(\lambda)$, we
    may decompose it as the union
    \[
        T = \bigcup_{C \in \clustertree_W(\lambda)} T \cap C
    \]
    hence we have
    \begin{align*}
        S \times T 
        &= \bigcup_{C \in \clustertree_W(\lambda)} S \times (T \cap C)\\
        &= \bigcup_{C \in \clustertree_W(\lambda)} (\hat A \cap \hat C) \times 
            (T \cap C)\\
        &= \bigcup_{C \in \clustertree_W(\lambda)} (\hat A \times T) \cap (\hat C
            \times C).
    \end{align*}
    But $M < \lambda$ almost everywhere on $\hat C \times C$ whenever $C$ and
    $\hat C$, are disjoint clusters. Hence $M^{-1}[\lambda,1] \cap (S \times T)$
    is equal, up to a null set, to the set $M^{-1}[\lambda,1] \cap (\hat A
    \times T) \cap (\hat C \times \hat C)$. Using the identity once again, this
    is the set $M^{-1}[\lambda,1]\cap \left[(\hat A \cap \hat C) \times (\hat C
    \cap T)\right]$. But $\hat C \cap T$ is null, so that this set is null.
    This is a contradiction, as it implies that $M < \lambda$ almost everywhere
    on $S \times T$, but $S \cup T = \hat A$ is connected at level $\lambda$.
    Therefore it must be that $T$ is null, and hence $\mu(\hat A \setminus \hat
    C) = 0$. This implies that $\mu(A \setminus \hat C) = 0$, and so $A$ is
    contained in some cluster of $W$ at level $\lambda$, namely $C$.
\end{proof}

\begin{claim}
    \label{claim:cluster_completion}
    Suppose a set $A$ of positive measure is contained in a cluster at every
    level $\lambda' < \lambda$. Then $A$ is contained in a cluster at level
    $\lambda$.
\end{claim}

\begin{proof}
    We may construct a sequence $C_1, C_2, \ldots$ of clusters such that $C_i$
    is a cluster at level $\lambda - 1/n$, and $C_i$ contains $A$. Then the
    intersection $C = \bigcap_{i=1}^\infty C_i$ is connected at all levels
    $\lambda' < \lambda$, as otherwise there would exist a $\lambda^* < \lambda$
    at which $C$ is disconnected, but this would imply that $C_i$ is
    disconnected for any $i$ such that $\lambda - 1/i > \lambda^*$. Furthermore,
    $C$ has positive measure, since the measure of every $C_i$ is at least
    $\mu(A)$. Therefore, $C$ is contained in some cluster at level $\lambda$,
    and $C$ contains $A$. Hence $A$ is in some cluster at level $\lambda$.
\end{proof}

\newcommand{\strict}{\tilde{\rho}}
\begin{claim}
    \label{claim:strict_section}
    Let $\clustertree$ be a cluster tree. There exists a section function
    $\strict$ on $\clustertree$ such that if
    \begin{enumerate}
        \item
        $\mu(\cluster \cap \cluster') = 0 
        \Rightarrow 
        \strictsection(\cluster) \cap \strictsection(\cluster) = \emptyset$, 
        \item
            $\cluster \subset \cluster' \Rightarrow
            \strictsection(\cluster) \subset \strictsection(\cluster')$, and 
        \item (Technical condition)\;
        $
        \strictsection(\cluster) = \bigcap \{
            \strictsection(\cluster') : \cluster' \supset \cluster
        \}$.
    \end{enumerate}
\end{claim}

\begin{proof}
    \newcommand{\rats}{\mathbb{Q}_{[0,1]}}
    \newcommand{\goodsets}{\hat{\clustertree}}
    \newcommand{\ancestors}{\mathfrak{A}}
    We construct such a section function on the clusters at rational levels and
    extend it to $[0,1]$. Let $\rats = \mathbb{Q} \cap [0,1]$. Define 
    \[
        \goodsets = \{
            \eqclass \in \clustertree(\lambda) : \lambda \in \rats
        \}
    \]
    that is, $\goodsets$ is the set of all clusters from every rational level.
    Note that this is a countable collection. For any cluster $\eqclass \in 
    \goodsets$, define $P_{\eqclass}$ to be the set of clusters in $\goodsets$
    which have null intersection with $\eqclass$. That is:
    \[
        P_{\eqclass} = \{ \eqclass' \in \goodsets : \mu(\eqclass \cap \eqclass')
        = 0\}.
    \]
    Let $\rho_0$ be an arbitrary section function, and define the section
    function $\rho_1 : \goodsets \to \Sigma$ as follows:
    \[
        \rho_1(\eqclass) = \rho_0(\eqclass) \setminus 
            \bigcup P_{\eqclass}.
    \]
    Furthermore, let $\cluster_0$ be the equivalence class of sets differing
    from $[0,1]$ by a null set, and define $\rho_1(\cluster_0) = [0,1]$; This
    will ensure that all pairs of points have a well-defined merge height.
    The intersection of $\rho_0(\eqclass)$ and $\bigcup P_\eqclass$ is null,
    by definition of $P_\eqclass$ and the fact that it is a countable set.
    Therefore, $\rho_1(\eqclass) \symdiff \rho_0(\eqclass)$ is null, and
    $\rho_1(\eqclass)$ is hence a valid representative of $\eqclass$.
    Furthermore, for any $\cluster, \cluster' \in \goodsets$ such that
    $\mu(\cluster \cap \cluster') = 0$, we have $\rho_1(\cluster) \cap
    \rho_2(\cluster') = \emptyset$.

    We now define the section function on all levels in $[0,1]$. For a cluster
    $\eqclass$ at any level, define its set of ancestors in $\goodsets$ to be
    \[
        \ancestors_\eqclass = \{
            \eqclass' \in \goodsets : \mu(\eqclass' \setminus \eqclass) = 0
        \}.
    \]
    Then define \[
        \strict(\eqclass) = \bigcap_{\eqclass' \in \ancestors_\eqclass}
        \rho_1(\eqclass')
    \]
    Hence $\strict$ trivially satisfies the third condition of the claim.

    We must argue
    that $\strict(\eqclass)$ is a valid representative of $\eqclass$. First,
    suppose that $\eqclass$ is a cluster at level $\lambda$. Then 
    $\ancestors_\eqclass$ contains a cluster from every rational level below
    $\lambda$, and $\strict(\eqclass)$ is contained in a representative of each
    of them. It follows that $\strict(\eqclass)$ is contained in a cluster
    representative at every level $\lambda' < \lambda$. Hence, by
    Claim~\ref{claim:cluster_completion} in \Cref{apx:proofs},
    $\strict(\eqclass)$ is contained in a cluster representative at level
    $\lambda$. But $\eqclass$ is essentially contained in all of its ancestors.
    Therefore, it must be that $\strict(\eqclass) \symdiff \cluster$ is null
    and so $\strict(\eqclass)$ is a valid representative of $\cluster$.

    Now we show that $\strict$ has the desired properties. Suppose $\eqclass$
    and $\eqclass'$ have null intersection, and without loss of generality,
    assume that they are clusters at the same level $\lambda$. Let $\lambda' <
    \lambda$ be the maximal level at which their intersection is not null. Then
    there is some rational level $\tilde\lambda$ between $\lambda'$ and
    $\lambda$. Hence $\strict(\eqclass)$ is strictly contained in the
    representative $\rho_1(\tilde\eqclass)$ of some cluster $\tilde\eqclass$ at
    level $\tilde\lambda$, and similarly, $\strict(\eqclass)$ is strictly
    contained in $\rho_1(\tilde\eqclass')$ at the same level. Necessarily,
    $\tilde\eqclass$ and $\tilde\eqclass'$ have null intersection, and so
    $\rho_1(\tilde\eqclass)$ and $\rho_1(\tilde\eqclass')$ are strictly
    disjoint. Therefore, so also are $\strict(\eqclass)$ and
    $\strict(\eqclass')$.

    Furthermore, suppose that $\eqclass$ and $\eqclass'$ are such that
    $\mu(\eqclass' \setminus \eqclass) = 0$. Suppose without loss of generality
    that $\lambda > \lambda'$ (if $\lambda = \lambda'$ then $\strict(\eqclass) =
    \strict(\eqclass')$). Then the ancestors of $\eqclass$ include the ancestors
    of $\eqclass'$, and so the intersection of the ancestors of $\eqclass$ is a
    subset of the intersection of the ancestors of $\eqclass'$. This proves that
    $\strict(\eqclass) \subset \strict(\eqclass')$.

\end{proof}

\section{Clusters under measure preserving transformations}
\label{apx:mptclusters}

In this section we show that there is a bijection between the clusters of two
weakly isomorphic graphons. In particular, we show that if $\varphi$ is a
measure preserving transformation, then $\cluster$ is a cluster of $W^\varphi$ at
level $\lambda$ if and only if there exists a cluster $\cluster'$ at level
$\lambda$ of $W$ such that $\cluster = \varphi^{-1}(\cluster')$. This is made
non-trivial by the fact that a measure preserving transformation is in general
not injective. For instance, $\varphi(x) = 2x \mod 1$ defines a measure
preserving transformation, but is not an injection. Even worse, it is possible
for a measure preserving transformation to map a set of zero measure to a set of
positive measure -- it is only the measure of the \emph{preimage} which must be
preserved.

\subsection{Claims}

We will mitigate the fact that $\varphi$ may not be injective by working
whenever possible with sets whose image is necessarily stable under
even non-injective measure preserving transformations, in the sense that
$\varphi^{-1}(\varphi(A)) = A$. We will show that such stability is a property
of sets which contain all of their so-called \emph{twin} points, defined as
follows. Two points $x$ and $x'$ are \term{twins} in $W$ if $W(x,y) = W(x',y)$
for almost every $y \in [0,1]$. We say that a set $A$ \term{separates} twins if
there exist twins $x$ and $x'$ such that $x \in A$ and $x' \not \in A$. The
relation of being twins is an equivalence relation on $[0,1]$.

\newcommand{\twinclass}{\psi_{W}}

We will define a probability space on the equivalence classes of the twin
relation as follows (see \cite{Lovasz2012-ww} for the full construction):

\begin{defn}
    Let $W$ be a graphon. The \term{twin measure space} $(\Omega_W, \Sigma_W,
    \mu_W)$ is defined as follows. Let $\Omega_W$ be the set of equivalence
    classes under the twin relation in $W$, and let $\twinclass(x)$ denote the
    equivalence class in $\Omega_W$ containing $x$. If $\Sigma$ is the
    sigma-algebra of Lebesgue measurable subsets of $[0,1]$, create a new
    sigma-algebra by defining 
    \[ 
        \Sigma_W = \{ \twinclass(X) : X \in \Sigma, X \text{ does not separate twins
        in $W$} \}.  
    \]
    Furthermore, we define the measure $\mu_W(A) = \mu(\twinclass^{-1}(A))$ for
    $A \in \Sigma_W$. It can be shown that, with this measure, $\twinclass$ is
    measure preserving. 
\end{defn}

We note in passing that the random graph model defined by any graphon $W$ can
also be represented as a $\Sigma_W$-measurable function $W_T :
\Omega_W \times \Omega_W \to [0,1]$ defined on the probability space $(\Omega_W,
\Sigma_W, \mu_W)$, as is shown by \cite{Lovasz2012-ww}. $W_T$ is called a
``twin-free'' graphon, since no two points in $\Omega_W$ are twins in $W_T$. In
this representation, two twin-free graphons are weakly isomorphic if there
exists a measure preserving \emph{bijection} relating them. Our definitions of
connectedness, clusters, mergeons, etc. can be formulated for twin-free graphons
with minor modifications, and the existence of the measure preserving bijection
between twin-free graphons means that clusters transfer trivially between weakly
isomorphic graphons. In a sense, the twin-free setting is a more natural one for 
the considerations of the current section; We leave a more in-depth
investigation of this direction to future work.

We now prove some useful properties of the map $\twinclass$.

\begin{claim}
    Suppose $A \subset [0,1]$ does not separate twins in $W$. Then
    $\twinclass^{-1}(\twinclass(A))
    = A$ and $A$ is $\Sigma$-measurable.
\end{claim}

\begin{proof}
    It is clear that $A \subset \twinclass^{-1}(\twinclass(A))$. Now let $x \in
    \twinclass^{-1}(\twinclass(A))$. Then there exists a $y \in A$ such that
    $\twinclass(x) = \twinclass(y)$. But then $x$ and $y$ are twins in
    $W^\varphi$. Since $A$ does not separate twins, $x \in A$, proving that $A =
    \twinclass^{-1}(\twinclass(A))$.

    Now we prove that $A$ is $\Sigma$-measurable. $\twinclass$ is a measurable
    function, and so the inverse image of any $\Sigma_W$-measurable set is
    $\Sigma$-measurable.  We have that $\twinclass(A)$ is $\Sigma_W$-measurable,
    since $A$ does not separate twins. Hence $\twinclass^{-1}(\twinclass(A)) =
    A$ is $\Sigma$-measurable.
\end{proof}

\begin{claim}
    Let $W$ be a graphon and let $A \subset [0,1]$. Then
    $\twinclass^{-1}(\twinclass(A))$ is $\Sigma$-measurable.
\end{claim}

\begin{proof}
    It is clear that $\twinclass^{-1}(\twinclass(A))$ does not separate twins in
    $W$. Hence it is $\Sigma$-measurable by the previous claim.
\end{proof}

\begin{claim}
    \label{claim:non_separating_mpt}
    Let $W$ be a graphon and $\varphi$ a measure preserving transformation.
    Suppose $A$ does not separate twins in $W^\varphi$. Then
    \begin{enumerate}
        \item $\varphi^{-1}(\varphi(A)) = A$, and
        \item $\mu(\varphi(A)) = \mu(A)$.
    \end{enumerate}
\end{claim}

\begin{proof}
    For the first claim, we know that $A \subseteq \varphi^{-1}(\varphi(A))$.
    Now we show the other inclusion. Let $x \in \varphi^{-1}(\varphi(A))$. Then
    there exists an $x'$ in $\varphi(A)$ such that $\varphi(x) = \varphi(x')$.
    But then $x$ and $x'$ are twins, such that $x$ and $x'$ are both in $A$.
    Hence $x \in A$, proving the claim.  The second claim follows immediately
    since $\varphi$ is measure preserving.  That is,
    $\mu(\varphi^{-1}(\varphi(A))) = \mu(\varphi(A))$, but since
    $\varphi^{-1}(\varphi(A)) = A$, $\mu(\varphi(A)) = \mu(A)$.
\end{proof}

Therefore, sets which do not separate twins are stable under measure preserving
transformations. An arbitrary set $C$ may separate twins, however we can always
find a set containing all of $C$ except for a null set, and which does not
separate twins. We call the smallest such set the \term{family} of $C$.

\begin{defn}
    Let $W$ be a graphon and let $(\Omega_W, \Sigma_W, \mu_W)$ be the
    corresponding twin measure space for $W$. For any $\Sigma$-measurable set
    $C$, construct the collection
    \[
        \mathcal F_C = \{ A \in \Sigma_W : \mu(C \setminus \twinclass^{-1}(A)) = 0 \}.
    \]
    We define the \term{family} of $C$, written $\family_W C$, as
    \[
        \family_W C = \{ \twinclass^{-1}(X) : X \in \essmin \mathcal F_C \}.
    \]
\end{defn}
Recall that the $\essmin \A$ of a collection of sets $\A$ is defined to be the
set
\[
    \essmin \A = \{ M \in \A : \mu(M \setminus A) = 0 \quad \forall A \in \A \}.
\]
See Claim~\ref{claim:essmin} in \Cref{apx:proofs} for the properties of the
essential minima of a class of sets.

It is clear that $\family_W C$ cannot be empty, as $\twinclass^{-1}(\Omega_W)$
must contain almost all of $C$. To be rigorous, we must argue that $\mathcal
F_C$ is closed under countable intersections so that it has a well-defined set
of essential minima. To see this, let $\mathcal F$ be any countable subset of
$\mathcal F_C$. Define $D = \bigcap \mathcal F$. Then $D \in \Sigma_W$ since it
is a sigma-algebra, and we have
\begin{align*}
    \mu\left( C \setminus \twinclass^{-1}(D) \right) 
    & = \mu\left( C \setminus \twinclass^{-1}\left(\,\bigcap \mathcal F\right) \right),\\
    & = \mu\left( C \setminus \bigcap_{F \in \mathcal F} \twinclass^{-1}(F) \right),\\
    & = \mu\left( \, \bigcup_{F \in \mathcal F} C \setminus F \right),\\
    & = 0,
\end{align*}
where the last step follows because each $F$ has the property that $C \setminus
F$ is a null set, and the union of countably many null sets is null. Hence $D
\in \mathcal F_C$.

Also note that for any $A,A' \in \family_W C$, it must be that $\mu(A \symdiff A')
= 0$. This is because $A = \twinclass^{-1}(B)$ and $A' = \twinclass^{-1}(B')$
for some $B,B' \in \essmin \mathcal F_C$. As seen above, $\mu_W(B \symdiff B') =
0$.  Since $\twinclass^{-1}(B \setminus B') = \twinclass^{-1}(B) \setminus
\twinclass^{-1}(B')$, it must be that $\twinclass^{-1}(B) \symdiff
\twinclass^{-1}(B') = A \symdiff A'$ is a null set. Furthermore, it is clear
that for any $A \in \family_W C$, $A$ does not separate twins in $W$.

A key result is that the family of any representative of a cluster differs from
the representative by a null set, and is therefore itself a representative. That
is, we can always find a representative of a cluster which does not separate
twins.

\begin{claim}
    \label{claim:family_of_cluster}
    Let $W$ be a graphon and suppose $\cluster$ is a cluster at level $\lambda$
    in $W$. Let $\bar C \in \family \cluster$. Then $\mu(\cluster \symdiff \bar
    C) = 0$.
\end{claim}

\begin{proof}
    Take an arbitrary representative $C$ of $\cluster$.  Let $(\Omega_W,
    \Sigma_W, \mu_W)$ be the twin measure space as defined above.

    We know that $C \setminus \bar C$ is a null set, so we need only show that
    $\bar C \setminus C$ is null.  Suppose otherwise. That is, let $R = \bar C
    \setminus C$ and suppose $\mu(R) > 0$. Let $A$ be any subset of $R$ with
    positive measure.  There are two cases: (1) For some $\lambda' < \lambda$,
    $W < \lambda'$ almost everywhere on $A \times (\bar C \setminus A)$, or (2)
    for every $\lambda' < \lambda$, $W \geq \lambda'$ on some subset of $A
    \times (\bar C \setminus A)$ of positive measure.

    Suppose case (1) holds for some $\lambda'$. Then for almost all $a \in A$ it
    is true that $W(a,y) < \lambda'$ for almost every $y \in \bar C \setminus
    A$. That is, let 
    \[
        \hat A = \{ a \in A : W(a,y) < \lambda' \text{ for almost every } y
        \in \bar C \setminus A \}.
    \]
    Then $\mu(\hat A) = \mu(A)$ and $W < \lambda$ almost everywhere on $\hat A
    \times (\bar C \setminus \hat A)$.  Define $\bar A =
    \twinclass^{-1}(\twinclass(\hat A))$. There are two subcases: Either (1a)
    $\bar A \cap C$ is null, or (1b) it is of positive measure. 
    
    Consider the first subcase. Define $D = \twinclass(\bar C) \setminus
    \twinclass(\bar A)$. We will show that $\twinclass^{-1}(D)$ contains $C$
    except for a set of zero measure, and so $\twinclass^{-1}(D) \in \family_W C$.
    But as we will see, $\mu(\twinclass^{-1}(D)) < \mu(\bar C)$, which cannot
    be. We have
    \begin{align*}
        \twinclass^{-1}(D) &= \twinclass^{-1}(\twinclass(\bar C) \setminus
                    \twinclass(\bar A))\\
                     &= \twinclass^{-1}(\twinclass(\bar C)) 
                        \setminus \twinclass^{-1}(\twinclass(\bar A))\\
                     &= \bar C \setminus \bar A
    \end{align*}
    where the last step follows since $\bar C$ and $\bar A$ do not separate
    twins. Therefore,
    \begin{align*}
        C \cap \twinclass^{-1}(D) 
        &= C \cap (\bar C \setminus \bar A)\\
        &= (C \cap \bar C) \cup (C \cap \bar A)
    \end{align*}
    But $C \cap \bar A$ is a null set, so $\mu(C \cap \twinclass^{-1}(D)) =
    \mu(C \cap \bar C) = \mu(C)$. This implies that $\mu(C \setminus
    \twinclass^{-1}(D)) = 0$, and hence $\twinclass^{-1}(D) \in \family_W C$.  But
    $\mu_W(D) = \mu_W(\twinclass(\bar C) \setminus \twinclass(\bar A))$, and
    $\twinclass(\bar A) \subset \twinclass(\bar C)$ with $\mu_W(\twinclass(\bar
    A)) = \mu(\bar A) > 0$. Therefore, $\mu_W(D) < \mu_W(\twinclass(\bar C))$,
    and so $\mu(\twinclass^{-1}(D)) < \mu(C)$.  This cannot be, since all
    elements of $\family_W C$ differ only by null sets. Hence it cannot be that
    $\bar A \cap C$ is null.

    Suppose case (1b) holds, then. That is, suppose $\bar A \cap C$ is not null.
    Then for every $x \in \bar A \cap C$ it is true that $W(x,y) < \lambda'$ for
    almost all $y \in \bar C \setminus A$. In particular, since $C \setminus (A
    \cap C) \subset \bar C \setminus A$, we have that $W < \lambda'$ almost
    everywhere on $(\bar A \cap C) \times (C \setminus \bar A)$. This means that
    $C$ is disconnected at level $\lambda'$, which violates the assumption that
    $C$ is a cluster at $\lambda > \lambda'$.

    Both subcases lead to contradictions, and so (1) cannot hold. Therefore, it
    must be that case (2) holds: For every $\lambda' < \lambda$, $W \geq
    \lambda'$ on some subset of $A \times (\bar C \setminus A)$. Furthermore,
    this must hold for arbitrary $A \subset R$ with positive measure. This
    implies that $\bar C$ is connected at every level $\lambda' < \lambda$, and
    hence part of a cluster at level $\lambda$. To see this, let $S,T \subset
    \bar C$ such that $S$ has positive measure and $S \cup T = \bar C$. Without
    loss of generality, assume $A \cap S$ is not null -- if it is, swap $S$ and
    $T$. Then $T \cap (\bar C \setminus A)$ is not null.  Therefore $W \geq
    \lambda$ on some subset of $S \times T$ with positive measure -- namely, $(S
    \cap A) \times (T \cap (\bar C \setminus A))$. Since this holds for
    arbitrary $S$ and $T$, $\bar C$ is connected.

    Therefore, both cases lead to contradictions under the assumption that
    $\mu(R) > 0$. Hence $\mu(R) = 0$, and $\mu(C \symdiff \bar C) = 0$.
\end{proof}

The previous claim shows that any cluster has a representative $C$ which does
not separate twins, and so $\varphi^{-1}(\varphi(C)) = C$. The next claim shows
that there exists a (possibly different) cluster representative $C'$ such that
$\varphi(\varphi^{-1}(C')) = C'$.

\begin{claim}
    \label{claim:good_representative}
    Let $W$ be a graphon and $\varphi$ a measure preserving transformation.
    Suppose $\cluster$ is a cluster of $W$.  There exists a representative $C$
    of $\cluster$ such that $\varphi(\varphi^{-1}(C)) = C$.
\end{claim}

\begin{proof}
    First, let $\bar C = \family_W \cluster$, such that $\bar C$ is a representative of
    $\cluster$ which does not separate twins. Then $\varphi^{-1}(\bar C)$ does
    not separate twins in $W^\varphi$, and so by
    \Cref{claim:non_separating_mpt}, $\mu(\varphi(\varphi^{-1}(\bar C))) =
    \mu(\bar C)$. But $\bar C \supset \varphi(\varphi^{-1}(\bar C))$, such that $\bar
    C \symdiff \varphi(\varphi^{-1}(\bar C)) = 0$. Hence
    $\varphi(\varphi^{-1}(\bar C))$ is a representative of the cluster
    $\cluster$. Furthermore, $\varphi^{-1}(\varphi(\varphi^{-1}(\bar C))) =
    \varphi^{-1}(\bar C)$, such that, defining $C = \varphi(\varphi^{-1}(\bar
    C))$, we have $\varphi(\varphi^{-1}(C)) = C$, as claimed.
\end{proof}

Recall that a set $C$ is disconnected at level $\lambda$ in a graphon $W$ if
there exists a subset $A \subset C$ with $0 < \mu(A) < \mu(C)$ such that $W <
\lambda$ almost everywhere on $A \times (C \setminus A)$. It is of course
possible, however, that $A$ might separate twins -- even if $C$ does not.
Therefore we cannot use the above claims to manipulate $A$. The next claim says
that if a set $C$ which does not separate twins is disconnected at some level,
it is always disconnected by a set $\bar A$ which also does not separate twins.

\newcommand{\nonsep}{\bar A}
\begin{claim}
    \label{claim:completion_disconnected}
    Let $W$ be a graphon and suppose that $C$ is a set of positive measure that
    does not separate twins. If $C$ is disconnected at level $\lambda$ in $W$,
    then either $W < \lambda$ almost everywhere on $C \times C$, or there exists
    a set $\nonsep \subset C$ such that $\nonsep$ does not separate twins,
    $0 < \mu(\nonsep) < \mu(C)$, and $W < \lambda$ almost everywhere on
    $\nonsep \times (C \setminus \nonsep)$.
\end{claim}

\begin{proof}
    Since $C$ is disconnected at level $\lambda$, there exists a subset $S
    \subset C$ such that $0 < \mu(S) < \mu(C)$ and $W < \lambda$ almost
    everywhere on $S \times (C \setminus S)$. Define
    \[
        \hat S = \{ x \in S : W(x,y) < \lambda \text{ for almost every $y \in C
        \setminus S$} \}.
    \]
    Since $W < \lambda$ almost everywhere on $S \times (C \setminus S)$, it must
    be that $\mu(\hat S) = \mu(S)$; This is an application of Fubini's theorem.
    Let $\bar S = \twinclass^{-1}(\twinclass(\hat S))$. It follows that $\bar S \subset C$,
    and for every $x \in \bar S$, $W(x,y) < \lambda$ for almost every $y \in C
    \setminus S$.  Furthermore, $\bar S$ does not separate twins, and $\bar S$
    contains $\hat S$ -- which is $S$, less a null set -- so $\mu(S \setminus
    \bar S) = 0$.

    There are two cases: $\mu(\bar S) < \mu(C)$, or $\mu(\bar S) =
    \mu(C)$. Suppose the first case holds. Then, since $\mu((C \setminus S)
    \setminus (C \setminus \bar S)) = 0$, we have that for every $x \in \bar S$, 
    $W(x,y) < \lambda$ for almost every $y \in C \setminus \bar S$. Therefore,
    $W < \lambda$ almost everywhere on $\bar S \times (C \setminus \bar S)$.
    This proves the claim for the first case, as we may take $\bar A = \bar S$.

    Now suppose $\mu(\bar S) = C$, which is to say that $\bar S$ differs from
    $C$ by a null set. Since $W < \lambda$ almost everywhere on $\bar S \times
    (C \setminus S)$, it follows that $W < \lambda$ almost everywhere on $C
    \times (C \setminus S)$. By symmetry of $W$, we have $W < \lambda$ almost
    everywhere on $(C \setminus S) \times C$. This means that $W < \lambda$
    almost everywhere on $(C \times C) \setminus (S \times S)$. 
    
    Let $T = C \setminus S$. Then $W < \lambda$ almost everywhere on $T \times
    C = (C \setminus S) \times C$. Define
    \[
        \hat T = \{
            x \in T : W(x,y) < \lambda \text{ for almost every $y \in C$ }
        \}.
    \]
    Let $\bar T = \twinclass^{-1}(\twinclass(\hat T))$. Then, by a similar argument used
    above for $\bar S$, $\mu(T \setminus \bar T) = 0$, $\bar T$ does not
    separate twins, and $W < \lambda$ almost everywhere on $\bar T \times C$.

    There are two subcases: First, it may be that $\mu(\bar T) = \mu(C)$. If so,
    then $W < \lambda$ almost everywhere on $C \times C$, which proves the claim.
    Second, it may be that $\mu(\bar T) < \mu(C)$. In this case, we have $W <
    \lambda$ almost everywhere on $\bar T \times (C \setminus \bar T)$, and so
    taking $\bar A = \bar T$ proves the claim.
    
\end{proof}

The next two claims shown that the preimage under $\varphi$ of a cluster at
level $\lambda$ in a graphon $W$ is connected at every level $\lambda' <
\lambda$, and, conversely, a cluster at level $\lambda$ in $W^{\varphi}$ has a
particular representative whose image under $\varphi$ is connected at every
level $\lambda' < \lambda$ in $W$.

\newcommand{\inv}{\varphi^{-1}(C)}

\begin{claim}
    \label{claim:cluster_in_w}
    Let $W$ be a graphon and $\varphi$ a measure preserving transformation. If
    $\cluster$ is a cluster at level $\lambda$ in $W$, then
    $\varphi^{-1}(\cluster)$ is connected at every level $\lambda' < \lambda$ in
    $W^\varphi$.
\end{claim}

\begin{proof}
    For simplicity, we will work with an representative $C$ of the cluster
    $\cluster$. As \Cref{claim:good_representative} shows, we may take $C$ to be
    a representative such that $\varphi(\varphi^{-1}(C)) = C$.
    
    Suppose for a contradiction that $\inv$ is disconnected in
    $W^\varphi$ at some level $\lambda' < \lambda$. Then by
    Claim~\ref{claim:completion_disconnected} either $W^{\varphi} < \lambda'$
    almost everywhere on $\inv \times \inv$, or there exists a set $\nonsep
    \subset \inv$ such that $0 < \mu(\nonsep) < \mu(\inv)$, $W^\varphi <
    \lambda'$ almost everywhere on $\nonsep \times (\inv \setminus \nonsep)$,
    and $\nonsep$ does not separate twins.

    In the first case, $W^\varphi < \lambda'$ almost everywhere on $\inv \times
    \inv$ implies that $W < \lambda'$ almost everywhere on $C \times C$, which
    contradicts the fact that $C$ is the representative of a cluster at level
    $\lambda'$ in $W$.

    Suppose the second case, then, where $W^\varphi < \lambda'$ almost
    everywhere on $\nonsep \times (\inv \setminus \nonsep)$. Then $W < \lambda'$
    almost everywhere on $\varphi(\nonsep) \times \varphi(\inv \setminus
    \nonsep)$. We now claim that $\varphi(\varphi^{-1}(C)\setminus \nonsep)
    = \varphi(\varphi^{-1}(C)) \setminus \varphi(\nonsep) = C \setminus
    \nonsep$. To see this, note that $\varphi(\varphi^{-1}(C)\setminus \nonsep)
    \supset \varphi^{-1}(\varphi(C)) \setminus \varphi(\nonsep)$. However, we
    have chosen $C$ to be a representative such that $\varphi(\varphi^{-1}(C)) =
    C$, and so we obtain
    \[
        \varphi(\varphi^{-1}(C) \setminus \nonsep) \supset C \setminus \nonsep.
    \]

    On the other hand, suppose $y \in \varphi(\varphi^{-1}(C) \setminus
    \nonsep)$. This means that there is some $x \in \varphi^{-1}(C) \setminus
    \nonsep$ such that $\varphi(x) = y$. But $\varphi^{-1}(C) \setminus \nonsep$
    does not separate twins, so there cannot be an $x' \in \nonsep$ such that
    $\varphi(x') = \varphi(x) = y$. Therefore, $y \in \varphi(\varphi^{-1}(C)
    \setminus \nonsep)$ if $y \in C$ and there is no $a \in \nonsep$ such that
    $\varphi(a) = y$. That is, $\varphi(\inv \setminus \nonsep) \subset C
    \setminus \nonsep$. Hence $\varphi(\inv \setminus \nonsep) = C \setminus
    \varphi(\nonsep)$.
    
    Therefore $W < \lambda$ almost everywhere on $\varphi(\nonsep) \times (C
    \setminus \varphi(\nonsep))$. Since $\mu(\varphi(\nonsep)) = \mu(\nonsep) <
    \mu(C)$ by \Cref{claim:non_separating_mpt}, this implies that $C$ is
    disconnected at level $\lambda'$ in $W$. Hence $C$ is not the representative
    of a cluster at level $\lambda$, and so we have derived a contradiction.

    Both cases lead to contradictions, and so it must be that $\inv$ is
    connected in $W^\varphi$ at every level $\lambda' < \lambda$.
\end{proof}

\begin{claim}
    \label{claim:cluster_in_varphi}
    Let $W$ be a graphon and $\varphi$ be a measure preserving transformation.
    Suppose $\cluster$ is a cluster of $W^\varphi$ at level $\lambda$. Let $C
    \in \family_W \cluster$. Then $\varphi(C)$ is connected at every level
    $\lambda' < \lambda$ in $W$.
\end{claim}

\begin{proof}
    Suppose for a contradiction that $\varphi(C)$ is not connected at some level
    $\lambda' < \lambda$ in $W$. Then there exists a set $S \subset \varphi(C)$
    such that $0 < \mu(S) < \mu(\varphi(C))$ and $W < \lambda'$ almost
    everywhere on $S \times (\varphi(C) \setminus S)$. Hence $W^\varphi <
    \lambda'$ almost everywhere on $\varphi^{-1}(S) \times
    \varphi^{-1}(\varphi(C) \setminus S) = \varphi^{-1}(S) \times
    (\varphi^{-1}(\varphi(C)) \setminus \varphi^{-1}(S))$. Since $C$ does not
    separate twins in $W^\varphi$, we have by \Cref{claim:non_separating_mpt}
    that $\varphi^{-1}(\varphi(C)) = C$, and so $W^\varphi < \lambda'$ almost
    everywhere on $\varphi^{-1}(S) \times (C \setminus \varphi^{-1}(S))$.
    
    Consider $\varphi^{-1}(S)$. We have $C = \varphi^{-1}(\varphi(C))$, and
    since $S \subset \varphi(C)$, it follows that $\varphi^{-1}(S) \subset C$.
    Moreover, $\mu(\varphi^{-1}(S)) = \mu(S)$, since $\varphi$ is measure
    preserving, and $0 < \mu(S) < \mu(\varphi(C)) = \mu(C)$, where the last
    equality comes from Claim~\ref{claim:non_separating_mpt}. Hence $C$ is
    disconnected at level $\lambda'$ in $W$. This contradicts the fact that $C$
    is a representative of a cluster at level $\lambda$ in $W$. Hence it must be
    that $\varphi(C)$ is connected at every level $\lambda' < \lambda$ in $W$.

\end{proof}

The two previous claims are sufficient to prove the main result of this section.

\claimmptclusters*

\begin{proof}
    \label{proof:mptclusters}
    Suppose $\cluster$ is a cluster of $W$ at level $\lambda$ and let $C$ be a
    representative of $\cluster$. Then according to
    Claim~\ref{claim:cluster_in_w}, $\varphi^{-1}(C)$ is connected at every
    level $\lambda' < \lambda$ in $W^\varphi$, and hence there exists a cluster
    $\cluster'$ at level $\lambda$ in $W^\varphi$ which contains
    $\varphi^{-1}(C)$. Then by Claim~\ref{claim:cluster_in_varphi}, there is a
    representative $C'$ of $\cluster'$ such that $C'$ does not separate twins
    and $\varphi(C')$ is connected at every level $\lambda' < \lambda$ in $W$,
    and so there is a cluster $\cluster''$ of $W$ at level $\lambda$ such that
    $\cluster''$ contains $\varphi(C')$. However, it must be that $\cluster'' =
    \cluster$. To see this, note that we have $\varphi^{-1}(C \cap \varphi(C'))
    = \varphi^{-1}(C) \cap \varphi^{-1}(\varphi(C')) = \varphi^{-1}(C) \cap C'$.
    Since $\varphi$ is measure preserving, it follows that $\mu(C \cap
    \varphi(C')) = \mu(\varphi^{-1}(C) \cap C')$, but $C' \symdiff
    \varphi^{-1}(C)$ is a null set such that $\mu(C \cap \varphi(C')) = \mu(C)$.
    Thus $\mu(\cluster') = \mu(\varphi^{-1}(C))$, and so $\varphi^{-1}(C)$ is a
    representative of the cluster $\cluster'$. Hence $\varphi^{-1}(\cluster)$ is
    a cluster at level $\lambda$ of $W^\varphi$.

    Now suppose $\cluster$ is a cluster of $W^\varphi$ at level $\lambda$ and
    let $C$ be a representative of $\cluster$ such that $C \in
    \family_W(\cluster)$. Then according to Claim~\ref{claim:cluster_in_varphi},
    $\varphi(C)$ is connected at every level $\lambda' < \lambda$ in $W$, and
    hence there exists a cluster $\cluster'$ in $W$ at level $\lambda$ which
    contains $\varphi(C)$. By the previous argument, $\varphi^{-1}(\cluster')$
    is a cluster of $W^\varphi$ at level $\lambda$. Since $C \in \family_W
    \cluster$, $C$ does not separate twins in $W^\varphi$, and so
    $\varphi^{-1}(\varphi(C)) = C$, and thus $C$ is contained in
    $\varphi^{-1}(\cluster')$. Since $C$ is a cluster representative, and thus
    maximal, it must be that $\varphi^{-1}(\cluster') = \cluster$. 
   %  Since
   %  $\varphi$ is measure preserving, $\varphi(\cluster') = \mu(\cluster) =
   %  \mu(C) = \mu(\varphi(C'))$, where the last step follows from the fact that
   %  $C$ does not separate twins and Claim~\ref{claim:non_separating_mpt}. Hence
   %  $\varphi(C)$ is a representative of $\cluster'$.
\end{proof}

\section{Sufficient conditions for consistent clustering methods}
\label{apx:estprob}

In this section we prove that any consistent estimator of the edge probability
matrix leads to a consistent estimator of the graphon cluster tree.  Estimating
the graphon or the edge probability matrix is an area of recent research. There
are a number of methods in the literature; See, for instance,
\cite{Wolfe2013-xi}, \cite{Chan2014-js}, \cite{Airoldi2013-eu},
\cite{Rohe2011-jb}, \cite{Zhang2015-ik}.  Each work in this direction defines a
slightly different sense in which the proposed estimator is consistent, but all
use some variant of the mean squared error. Convergence in this norm ensures
that the estimate is close to the true graphon in aggregate, but still allows
the estimate to differ from the ground truth by a large amount on a set of small
measure. Since our merge distortion is sensitive to the \emph{largest} error,
regardless of measure, consistency of graphon estimators as shown in the
literature is not sufficient to show consistency in merge distortion.

In
particular we require that the estimator $\hat{\rv{P}}$ satisfies
\[
    \prob\left(
        \max_{i \neq j} 
        \left|
        \estedgeprobrv_{ij} - \edgeprobrv_{ij}
        \right| 
        > \epsilon
    \right) \to 0
\]
for every $\epsilon > 0$ as $n \to \infty$. The probability is with respect to
the label measure introduced in \Cref{defn:label_measure}. That is, to be
precise:
\[
    \prob\left(
        \max_{i \neq j} 
        \left|
        \estedgeprobrv_{ij} - \edgeprobrv_{ij}
        \right| 
        > \epsilon
    \right) = 
    \labelmeasure \left( \left\{
            (G,S) : \max_{i \neq j} | P_{ij} - \hat{P}_{ij} | > \epsilon
    \right\}
    \right)
\]
It is implicit here that $P$ is induced by the graphon $W$ and the particular
labeling $S$, and $\hat{P}$ is a function of the graph, $G$.

\newcommand{\paths}{\mathcal{P}}

Given an estimator $\estedgeprobrv$, we construct a consistent clustering
algorithm as follows. Let $\paths_n(i,j)$ be the set of all simple paths between
nodes $i$ and $j$ in the complete graph on node set $[n]$. For $p \in
\paths_n(i,j)$, let $\ell(p)$ denote the length of the path, and let $p_k$ be
the label of the $k$th node along the path. For any $i \neq j \in [n] \times
[n]$, define the \term{merge estimate} $\estmergeonrv_{ij}$ by
\[
    \estmergeonrv_{ij} = \max_{p \in \paths_n(i,j)} 
        \min_{1 \leq k \leq \ell(p)} \estedgeprobrv_{p_k p_{k+1}}.
\]
As its name implies, the merge estimate $\estmergeonrv_{ij}$ estimates the
height at which nodes $i$ and $j$ merge in the cluster tree. Intuitively, if
$\estmergeonrv_{ij}$ is close to the true merge height for every pair $i,j$, we
can use $\estmergeonrv$ to construct a clustering which is close to the cluster
tree in merge distortion.  Specifically, let $\rv{H}$ be the weighted graph on
node set $[n]$ in which the weight between nodes $i$ and $j$ is given by
$\estmergeonrv_{ij}$.  We define the clusters of $\rv{H}$ at level $\lambda$ to
be the connected components of the subgraph induced by removing every edge with
weight less than $\lambda$. The clustering $\clustering_\estmergeonrv$ is
defined to be the set of all clusters of $\rv{H}$ at any level $\lambda$.
Equivalently, $\estmergeonrv_{ij}$ is the level at which nodes $i$ and $j$ merge
in the single linkage clustering of $\estedgeprobrv$, when $\estedgeprobrv$ is
treated as a similarity matrix. Thus $\clustering_\estmergeonrv$ is simply the
single linkage clustering of $\estedgeprobrv$.

\subsection{Claims}
\label{apx:estprob:definitions}

We state our claims here, and place all technical details and proofs in
\Cref{apx:estprob:proofs} for clarity.

\newcommand{\blocksets}{\{B_i\}}
\newcommand{\refinementsets}{\{R_i\}}
\newcommand{\refinement}{\mathcal{R}}

It is sufficient to show that if $|\estedgeprobrv_{ij} - \edgeprobrv_{ij}| <
\epsilon$, $\estmergeonrv_{ij}$ is at most $\epsilon + c$ away
from the true merge height $M(x_i, x_j)$, where $c$ is a constant. It is easy to
see that the merge distortion between the cluster tree and
$\clustering_\estmergeonrv$ cannot be greater than $2(\epsilon + c)$. 

\begin{restatable}{claim}{claimmergedist}
    \label{claim:mergedist}
    Let $W$ be a graphon, $M$ be a mergeon of $W$, and $S = (x_1, \ldots, x_n)$.
    Suppose $\max_{i \neq j} | M(x_i,x_j) - \estmergeon_{ij} | < \epsilon$, and
    let $\clustering_\estmergeon$ be the clustering defined above of the
    weighted graph $H$ with weight matrix $\estmergeon$. Let $\hat{M}$ be the
    merge height on $\clustering_\estmergeon$ induced by $M$. Then the merge
    distortion $\mergedist_S(M,\hat{M}) < 2\epsilon$.
\end{restatable} 

We will now show that $\estmergeonrv_{ij}$ is close to the true merge height for
all $ij$ with high probability. First, recall the definition of a piecewise
Lipschitz graphon:

\defnlipschitz*

The idea is that the piecewise Lipschitz graphon is essentially piecewise
constant when viewed at small enough scales.  As such, we refine the blocks on
which the graphon is Lipschitz, creating a new block partition whose blocks are
small enough that $W$ varies by only a small amount on each.  We define a
refinement as follows:

\begin{defn}
    \label{defn:delta_refinement}
    A set of blocks $\refinement = \refinementsets$ is a
    \term{$\Delta$-refinement} of a block partition $\blocks = \blocksets$ if
    for every $R \in \refinement$, $\Delta \leq \lebesguemeasure(R) \leq
    2\Delta$ and there exists some $B \in \blocks$ such that $B \supseteq R$.
\end{defn}

We can think of the blocks in a refinement as being nodes in a weighted graph,
such that the weight between blocks $R$ and $R'$ is approximately the value of
$W$ on $R \times R'$. As such, we define a path of blocks in a refinement as
follows:

\begin{defn}
    Let $\refinement$ be a block partition of $[0,1]$, and suppose $R,R' \in
    \refinement$.  A $\lambda$-path from $R$ to $R'$ in a graphon $W$ is a
    sequence $\langle R = R_1, \ldots, R_t = R' \rangle$ of blocks from
    $\refinement$ such that, for all $1 \leq i < t$, $W \geq \lambda$ almost
    everywhere on $R_i \times R_{i+1}$. The elements of the path need not be
    distinct.
\end{defn}

In piecewise Lipschitz graphons, the existence of a $\lambda$-path between
blocks $R$ and $R'$ implies that there exists a set $C$ containing both $R$ and
$R'$, and which is connected at level $\lambda$, as the following claim
demonstrates. Note that is directly analogous to the case of a finite weighted
graph, where a pair of nodes is connected if there is a path between them.

\begin{restatable}{claim}{claimpathuno}
    \label{claim:path1}
    Let $W \in \smoothgraphons$ and let $M$ be a mergeon of $W$. Let
    $\refinement$ be a $\Delta$-refinement of $\blocks$. Let $\langle R_1,
    \ldots, R_t\rangle$ be a $\lambda$-path in $\refinement$.  Let $C = R_1 \cup
    \ldots \cup R_t$.  Then $C$ is connected at level $\lambda$ in $W$, and thus
    $M \geq \lambda$ almost everywhere on $C \times C = (R_1 \cup \ldots \cup
    R_t) \times (R_1 \cup \ldots \cup R_t)$.
\end{restatable}

Conversely, if $R$ and $R'$ are blocks in a $\Delta$-refinement, each of which
have non-null intersection with the same cluster $\cluster$, then there exists a
$(\lambda - 2\Delta)$-path of blocks between $R$ and $R'$:

\begin{restatable}{claim}{claimpathdos}
    \label{claim:path2}
    Let $W \in \smoothgraphons$. Let $\refinement$ be a $\Delta$-refinement of
    $\blocks$, and suppose $R,R' \in \refinement$ (possibly with $R = R'$).  If
    there exists a cluster $\cluster$ at level $\lambda$ such that
    $\lebesguemeasure(\cluster \cap R) > 0$ and $\lebesguemeasure(\cluster \cap
    R') > 0$, then there exists a $(\lambda' - 2\Delta \lipschitz)$-path $(R =
    R_1, \ldots, R_t = R')$ between $R$ and $R'$, for any $\lambda' < \lambda$.
\end{restatable}

Lastly, the Lipschitz condition on the graphon $W$ also implies that the mergeon
does not vary much:

\begin{restatable}{claim}{claimboundmergeon}
    \label{claim:boundmergeon}
    Let $R,R' \in \refinement$. Let $\lambda$ be the greatest level at which
    there exists some cluster $\cluster$ containing a non-negligible piece of
    both $R$ and $R'$. That is,
    \[
        \lambda = \sup \{ \lambda' : \exists \cluster \in \clustertree(\lambda')
            \text{ such that } \mu(R \cap \cluster) > 0 \text{ and } \mu(R' \cap
            \cluster) > 0. \}
    \]
    Then $\lambda' - 2\Delta\lipschitz \leq M \leq \lambda$ almost everywhere on
    $R \times R'$.
\end{restatable}

Putting these ideas together, we are able to bound the difference between the
true merge height of points in a mergeon, and the merge estimate $\estmergeon$.

\begin{restatable}{claim}{claimmergefixed}
    \label{claim:mergefixed}
    Let $W \in \smoothgraphons$ and let $M$ be a mergeon of $W$. Let
    $\refinement$ be a $\Delta$-refinement of $\blocks$. Let $S = (x_1, \ldots,
    x_n)$ be an ordered set of elements of $[0,1]$ such for any $R \in
    \refinement$, $R \cap S \neq \emptyset$. Let $P$ be the edge probability
    matrix, i.e., the matrix whose $(i,j)$ entry is given by $W(x_i,x_j)$, and
    suppose $\estedgeprob$ is such that $\inftynorm{\estedgeprob - P} <
    \epsilon$.  Then $\max_{i \neq j} |M(x_i,x_j) - \estmergeon_{ij}| \leq 4
    \Delta \lipschitz + \epsilon$.
\end{restatable}

The above holds for a fixed sample $S$ and thus a fixed edge probability matrix
$P$.  The following theorem considers random $\samplerv$ and $\edgeprobrv$. As
with the previous claims in this subsection, the proof of the theorem is in
\Cref{apx:estprob:proofs}.

\thmconditions*

\subsection{Proofs}
\label{apx:estprob:proofs}

\claimmergedist*
\begin{proof}
    Take any arbitrary $i \neq j$ in the clustering $\clustering_\estmergeon$.
    Let $C$ be the smallest cluster containing both $i$ and $j$. Then $C$ is a
    cluster in $H$ at level $\estmergeon_{ij}$. Let $u,v \in C$, $u \neq v$ be
    such that $M(x_u, x_v) = \min_{u' \neq v' \in C} M(x_{u'}, x_{v'}) =
    \hat{M}_{ij}$. Then we
    have that $M(x_i, x_j) \geq M(x_u, x_v)$. On the other hand, $u$ and $v$ are
    members of $C$, which is a cluster at level $\estmergeon_{ij}$, so that
    $\estmergeon{uv} \estmergeon{ij}$. Hence $\estmergeon_{uv} > M(x_i, x_j) -
    \epsilon$. But $\estmergeon_{uv} < M(x_u, x_v) + \epsilon$. Therefore,
    $M(x_i, x_j) - M(x_u, x_v) < 2\epsilon$. Therefore, $M(x_i, x_j) -
    \hat{M}_{ij} < 2 \epsilon$. This holds for all $i$ and $j$ simultaneously,
    since $i$ and $j$ were arbitrary. Hence the merge distortion is less than
    $2\epsilon$.
\end{proof}

\claimpathuno*
\begin{proof}

    Let $A$ be an arbitrary measurable subset of $C$ such that $0 <
    \lebesguemeasure(A) < \lebesguemeasure(C)$. We will show that
    $W^{-1}[\lambda,1] \cap A \times (C \setminus A)$ has positive measure, and
    therefore $C$ is connected at level $\lambda$. Since $C$ is connected at
    level $\lambda$ in $W$, it must be part of some cluster at level $\lambda$,
    and so the mergeon is at least $\lambda$ almost everywhere on $C \times C$.
    
    There are two cases: Either 1) There exists a $j \in [t]$ such that $0 <
    \lebesguemeasure(R_j \cap A) < \lebesguemeasure(R_j)$, or 2) for all $i \in
    [t]$, either $\lebesguemeasure(R_i \cap A) = 0$ or $\lebesguemeasure(R_i
    \cap A) = \lebesguemeasure(R_i)$. 

    Assume the first case: there exists a $j$ such that $R_j$ contains some
    non-negligible part of $A$, but $\mu(A \cap R_j) < \mu(R_j)$.
    Since there are at least two elements in the path, there is a $j'$ such that
    $j' \in [t]$ and $|j - j'| = 1$, that is, $R_{j'}$ is either immediately
    before or after $R_j$ in the $\lambda$-path. There are two sub-cases:

    \begin{itemize}
        \item $\lebesguemeasure(R_{j'} \cap A) = 0$, such that $R_{j'} \subseteq
            C \setminus A$. Then $(R_j \cap A) \times R_{j'} \subseteq A \times
            (C \setminus A)$. Since $\lebesguemeasure(R_j \cap A) > 0$ and
            $\lebesguemeasure(R_{j'}) > 0$, we have that $\lebesguemeasure((R_j
                \cap A) \times R_{j'}) > 0$, and since $W$ is at least $\lambda$
                a.e. on $R_j \times R_{j'}$, we have that
            \[
                \lebesguemeasure(W^{-1}[\lambda, 1] \cap A \times (C \setminus
                A)) \geq 
                \lebesguemeasure(W^{-1}[\lambda, 1] \cap (R_j \cap A) \times
                R_{j'}) > 0.
            \]
        \item $\lebesguemeasure(R_{j'} \cap A) > 0$. Then 
            $(R_{j'} \cap A) \times (R_j \setminus A) \subseteq A \times (C
                \setminus A)$ is a set of positive measure. Since $W$ is at
                least $\lambda$ a.e. on $R_{j'} \times R_j$, we have:
            \[
                \lebesguemeasure(W^{-1}[\lambda, 1] \cap A \times (C \setminus
                A)) \geq 
                \lebesguemeasure\left(
                    W^{-1}[\lambda, 1] \cap 
                    (R_{j'} \cap A) \times (R_j \setminus A)
                \right) > 0.
            \]
    \end{itemize}

    Now consider the second case in which, for every $i \in [t]$,
    $\lebesguemeasure(R_i \cap A = 0)$ or $\lebesguemeasure(R_i \cap A) =
    \lebesguemeasure(R_i)$. There must exist a $j,j' \in [t]$ such that $|j -
    j'| = 1$, $\lebesguemeasure(R_j \cap A) = \lebesguemeasure(R_j)$, and
    $\lebesguemeasure(R_{j'} \cap A) = 0$. If this were not the case, then it
    would be that either $\mu(R_i \cap A) = \mu(R_i)$ for every $i \in [t]$, or 
    $\mu(R_i \cap A) = 0$ for every $i \in [t]$. But the former of these would
    imply that $\mu(A) = \mu(C)$, and the latter would imply $\mu(A) = 0$, which
    we have assumed not to be the case.

    Therefore, $R_{j} \times R_{j'} \subseteq A \times (C \setminus A)$, and
    this set is of positive measure. Since $W$ is at least $\lambda$ a.e. on
    $R_j \times R_{j'}$, we once again find
    \[
        \lebesguemeasure(W^{-1}[\lambda, 1] \cap A \times (C \setminus
        A)) \geq 
        \lebesguemeasure\left(
            W^{-1}[\lambda, 1] \cap R_j \times R_{j'}
        \right) > 0.
    \]

    Hence, in every case it is true that $\lebesguemeasure(W^{-1}[\lambda, 1]
    \cap A \times (C \setminus A))$ has positive measure. Since $A$ was
    arbitrary, $C$ is connected at level $\lambda$. Hence $M \geq \lambda$
    almost everywhere on $C \times C$.
\end{proof}

\claimpathdos*
\begin{proof}
    To be precise, let $C = \rho(\cluster)$ be any representative of the cluster
    $\cluster$.  Fix a $\lambda' < \lambda$. Let
    \[
        \mathcal G = \{ R'' \in \refinement : \lebesguemeasure(R'' \cap C) > 0
        \}.
    \]
    Then $\mathcal G$ contains, in particular, $R_1$ and $R_t$. Since $\cluster$
    is connected at level $\lambda$, it is true that
    \[
        \lebesguemeasure(W^{-1}[\lambda', 1] \cap (R_1 \cap C) \times (C \setminus
        R_1)) > 0.
    \]
    Since $C \setminus R_1$ is a subset of $\left(\bigcup \mathcal G\right)
    \setminus R_1$, there must exist an $R_2 \in \mathcal G$ such that
    \[
        \lebesguemeasure(W^{-1}[\lambda',1] \cap R_1 \times R_2) > 0.
    \]
    Consider $W$ on $R_2$. From above, we know that there is a non-negligible
    subset of $R_1 \times R_2$ on which $W \geq \lambda'$.  Hence there is some
    point in $R_1 \times R_2$ on which $W \geq \lambda'$. Therefore, due to the
    Lipschitz condition, we know that $W$ is at least $\lambda' - 2\Delta
    \lipschitz$ everywhere on $R_1 \times R_2$. 

    Now let $S_2 = R_1 \cup (R_2 \cap C)$. Now, since $\cluster$ is connected at
    level $\lambda$, it is true that 
    \[
        \lebesguemeasure(W^{-1}[\lambda', 1] \cap S_2 \times (C \setminus S_2))
        > 0.
    \]
    By the same logic as above, there must exist an $R_3 \in \mathcal G$, $R_3
    \neq R_2, R_1$ such that 
    \[
        \lebesguemeasure(W^{-1}[\lambda', 1] \cap S_2 \times R_3) > 0.
    \]
    Hence it must be the case that either 
    \[
        \lebesguemeasure(W^{-1}[\lambda, 1] \cap R_1 \times R_3) > 0,
    \]
    or
    \[
        \lebesguemeasure(W^{-1}[\lambda, 1] \cap R_2 \times R_3) > 0.
    \]
    In either case, it is true that between any pair chosen from $R_1, R_2,
    R_3$, there is a $\lambda - 2\Delta \lipschitz$ path. The process continues,
    choosing $R_4, R_5, \ldots$ and so on. This process must complete in a
    finite number of steps, since $\mathcal G$ is a finite set. At every step,
    there exists a $\lambda$-path between any two of the $R_i$. Hence we
    eventually construct a $\lambda$-path between $R$ and $R'$.

\end{proof}

\claimboundmergeon*
\begin{proof}
    By the definition of the mergeon it must be that $M \leq \lambda$ almost
    everywhere on $R \times R'$, since if there existed a $\lambda' > \lambda$
    for which $M^{-1}[\lambda',1] \cap R \times R'$ is not-null, this would
    imply that there exists some cluster at level $\lambda'$ containing a
    non-negligible part of both $R$ and $R'$.

    Now, by Claim~\ref{claim:path2}, for any $\lambda' < \lambda$ there exists a
    $(\lambda' - 2\Delta\lipschitz)$ path between $R$ and $R'$. Hence, by
    Claim~\ref{claim:path1}, $M \geq \lambda' - 2\Delta\lipschitz$ almost
    everywhere on $R \times R'$ for any $\lambda' < \lambda$.
\end{proof}

\claimmergefixed*
\begin{proof}

    Consider an arbitrary $x_i,x_j \in S$. Let $R_i$ and $R_j$ be the blocks in
    $\refinement$ which contain $x_i$ and $x_j$, respectively. 
    Let $\lambda^*$ be the greatest level at which there exists some cluster
    containing non-negligible parts of both $R_i$ and $R_j$. Therefore, by
    Claim~\ref{claim:boundmergeon}, $M$ is bounded below by $\lambda^*
    -2\Delta\lipschitz$ and above by $\lambda^*$ almost everywhere on $R_i \times
    R_j$.
    
    First we bound $\estmergeon_{ij}$ from below.  By Claim~\ref{claim:path2}
    there exists a $(\lambda' - \Delta \lipschitz)$-path $\langle R_i = R_1,
    \ldots, R_t = R_j \rangle$ between $R_i$ and $R_j$, for any $\lambda' <
    \lambda^*$. By the assumption on $S$, there exists a sample from each
    element of the path, so that there is a path of samples $\langle x_i = x_1,
    \ldots, x_t = x_j \rangle$ with the property that, between any two
    consecutive elements in the path, we have $W(x_k, x_{k+1}) \geq \lambda' - 2
    \Delta \lipschitz$ for all $\lambda' < \lambda^*$.  Hence $\estedgeprob(x_k,
    x_{k+1}) \geq \lambda^* - 2\Delta \lipschitz - \epsilon$. Therefore, there
    exists a path $p$ from $x_i$ to $x_j$ such that $\min_{1\leq k \leq \ell(p)}
    \estedgeprob_{p_k p_{k+1}} \geq \lambda^* - 2\Delta \lipschitz$. As a
    result, $\estmergeon_{ij} \geq \lambda^* - 2\Delta \lipschitz - \epsilon$.

    We now bound $\estmergeon_{ij}$ from above.  Let $p = \langle x_i = x_1,
    \ldots, x_t = x_j \rangle$ be a path with cost $\estmergeon_{ij}$. Let
    $\langle R_1, \ldots, R_t \rangle$ be the corresponding path of blocks from
    $\refinement$, such that $x_k \in R_k$. Then we have $\estedgeprob_{x_k
    x_{k+1}} \geq \estmergeon_{ij}$, so that $W(x_k, x_{k+1}) \geq
    \estmergeon_{ij} - \epsilon$.  Hence there is a point in $R_k \times
    R_{k+1}$ which is at least $\estmergeon_{ij} - \epsilon$, and by smoothness
    it follows that $W \geq \estmergeon_{ij} - 2\Delta c - \epsilon$ almost
    everywhere on $R_{k} \times R_{k+1}$. That is, $\langle R_1, \ldots, R_t
    \rangle$ is a $(\estmergeon_{ij} - 2\Delta \lipschitz - \epsilon)$-path.
    Therefore, Claim~\ref{claim:path1} implies that the mergeon $M$ is at least
    $\estmergeon_{ij} - 2\Delta\lipschitz - \epsilon$ almost everywhere on $R_i
    \times R_j$. However, by Claim~\ref{claim:boundmergeon}, $M \leq \lambda^*$
    almost everywhere on $R_i \times R_j$. Therefore $\estmergeon_{ij} \leq
    \lambda^* + 2\Delta\lipschitz + \epsilon$.
    
    Combining the above bounds, we find that
    \[
        |\estmergeon_{ij} - \lambda^*| 
        \leq 2 \Delta \lipschitz + \epsilon.
    \]
    The true merge height $M(x_i,x_j)$ is bounded between $\lambda^* - 2\Delta
    \lipschitz$ and $\lambda^*$, and so we have
    \[
        |\estmergeon_{ij} - M(x_i,x_j)| 
        \leq 4 \Delta \lipschitz + \epsilon.
    \]

\end{proof}

\thmconditions*

\begin{proof}
    As stated, $f$ is the clustering method which takes a graph $G$ and returns
    the clustering $\clustering_\estmergeon$ described at the beginning of the
    section -- the single linkage clustering of the estimated edge probability
    matrix $\estedgeprobrv$. Let $M$ be a mergeon of $W$. We will show that,
    for any $\epsilon > 0$.
    \[
        \labelmeasure\left(
            \left\{
                (G,S) : \mergedist_S(M, \inducedheights_{f(G)\relabeledby S}) > \epsilon
            \right\}
        \right) \to 0,
    \]
    where $\hat{M}_{f(G)\relabeledby S}$ is the merge height function induced on
    the clustering $f(G)\relabeledby S$ by the mergeon $M$.

    First, fix any $\epsilon > 0$. \newcommand{\perr}{\tilde\epsilon} Let $\perr
    = \nicefrac{\epsilon}{4}$. Define
    \[
        H_n = \left\{
            (G, S) \in \allgraphs_n \times [0,1]^n :
            \max_{i \neq j} |\estedgeprob - \edgeprob| < \perr
        \right\},
    \]
    where $P$ is the edge probability matrix induced by $S$ and $\estedgeprob$
    is the estimate of $P$ computed from $G$. By the assumption that
    $\estedgeprobrv$ is consistent in $\infty$-norm, we have
    $\labelmeasure(H_n) \to 1$ as $n \to \infty$.

    Now let $\Delta = \nicefrac{\epsilon}{16\lipschitz}$. Let $\blocks$ be the
    block partition on which $W$ is piecewise $\lipschitz$-Lipschitz, and let
    $\refinement$ be an arbitrary $\Delta$-refinement of $\blocks$.
    In order to apply Claim~\ref{claim:mergefixed}, we require that the labeling
    $S$ satisfies the property that every block $R$ in the refinement contains at
    least one point from $S$. The probability that a block $R$ contains no
    points from a random sample $\rv{S}$ is $(1 - |R|)^n \leq (1 - \Delta/2)^n$,
    since $|R| \geq \Delta/2$. Now take a union bound over all blocks in the
    partition, of which there are at most $2/\Delta$. Hence the probability that
    there exists a block in the partition that does not have a sample from $S$
    is $\frac{2}{\Delta}(1 - \Delta/2)^n$. Let
    \[
        F_n = \left\{
            (G, S) \in \allgraphs_n \times [0,1]^n : 
            |R \cap S| > 1 \text{ for all } R \in \refinement
        \right\}.
    \]
    As per above, we have $\labelmeasure(F_n) = \frac{2}{\Delta}(1-
    \Delta/2)^n$, which tends to 0 as $n \to \infty$.

    By \Cref{claim:mergefixed}, for every $(G,S) \in H_n \setminus F_n$, we have
    that, for all $i \neq j \in [n] \times [n]$, writing $S = (x_1, \ldots,
    x_n)$:
    \[
        |\estmergeon_{ij} - M(x_i, x_j)| \leq 4 \Delta \lipschitz + \perr
        = \epsilon/2,
    \]
    where $\estmergeon$ is the merge estimate between nodes $i$ and $j$,
    described at the beginning of the section. The clustering method $f$ uses
    $\estmergeon$ to construct the clustering $\clustering_\estmergeon$
    Therefore,  by Claim~\ref{claim:mergedist}, the merge distortion
    $\mergedist(M, \hat{M}_{f(G)\relabeledby S})$ is bounded above by $\epsilon$
    on the set $H_n \setminus F_n$. Since $\labelmeasure(H_n) \to 1$ and
    $\labelmeasure(F_n) \to 0$ as $n \to \infty$, we have $\labelmeasure(H_n
    \setminus F_n) \to 1$ as $n \to \infty$ and have thus proven the claim.
\end{proof}

\section{Neighborhood smoothing methods}
\label{apx:smoothing}

\newcommand{\frobnorm}[1]{\|#1\|_F}
\newcommand{\given}[1][]{\:#1|\:}
\newcommand{\mats}[1]{\left[#1\right]}
\newcommand{\mat}[1]{\left(#1\right)}
\newcommand{\neighborhood}{\mathcal{N}}
\newcommand{\powerset}[1]{2^{#1}}
\newcommand{\twonorm}[1]{\left\|#1\right\|_2}
\newcommand{\without}{\setminus}

\Cref{thm:conditions} states a sufficient condition under which an estimator
$\estedgeprobrv$ of the edge probability matrix leads to a consistent clustering
algorithm. In particular, if the graphon $W$ is piecewise Lipschitz, and if for
any $\epsilon > 0$,
\[
    \lim_{n \to \infty}
    \prob(\max_{i \neq j} |\rv{P}_{ij} - \hat{\rv{P}}_{ij}| > \epsilon) = 0
\]
then one consistent clustering algorithm is that which applies single linkage
clustering to the estimate $\estedgeprobrv$. In this section, we analyze a
modification of the edge probability estimator introduced in\cite{Zhang2015-ik}
and show that it satisfies the above condition. Combining this result with
\Cref{thm:conditions} shows that the single linkage clustering applied this
estimate of the edge probability matrix is a consistent clustering algorithm.

\subsection{The method of \citet{Zhang2015-ik}}

The aim of the neighborhood smoothing method of \citet{Zhang2015-ik} is to
estimate the random edge probability matrix $\edgeprobrv$. In particular, the
method defines a distance $d(i,i')$ between the columns of the random adjacency
matrix $\adjacencyrv$ as such:

\[
    d(i,i') = \frac{1}{n} \max_{k \neq i,i'} 
        |\langle 
        \adjacencyrv_i - \adjacencyrv_{i'},
        \adjacencyrv_k
        \rangle|
        =
        \max_{k \neq i,i'}
        |(\adjacencyrv^2/n)_{ik} - (\adjacencyrv^2/n)_{i'k}|.
\]
The neighborhood $\neighborhood_i(\adjacencyrv)$ of node $i$ then consists of
all nodes $i'$ such that $d(i,i')$ is below the $h$-th quantile of
$\{d(i,k)\}_{k \neq i}$, where $h$ is a parameter of the algorithm. Note that
$\neighborhood_i(\adjacencyrv)$ is a random variable, as the neighborhood around
node $i$ depends on the random adjacency matrix $\adjacencyrv$. For simplicity,
however, we will often omit the explicit dependence on $\adjacencyrv$. 

The estimate of the probability of the edge $(i,j)$, written
$\estedgeprobrv_{ij}$, is then computed by smoothing over the neighborhoods
$\neighborhood_i$ and $\neighborhood_j$:
\[
    \estedgeprobrv_{ij} = 
    \frac{1}{2}
    \left(
    \frac{1}{|\neighborhood_i|} \sum_{i' \in \neighborhood_i} \adjacencyrv_{i'j}
    +
    \frac{1}{|\neighborhood_j|} \sum_{j' \in \neighborhood_j} \adjacencyrv_{ij'}
    \right).
\]
If it is assumed that $W \in \smoothgraphons$, and $h$ is set to be $C_0
\sqrt{\log n / n}$ for arbitrary constant $C_0$, where $n$ is the size of the
sampled graph, then the method is consistent in the sense that, for any
$\epsilon > 0$, as $n \to \infty$
\[
    \prob\left(\frac{1}{n^2} \frobnorm{\estedgeprobrv - \edgeprobrv}^2  >
    \epsilon\right) \to 0
\]

\subsection{Our modification}
\label{apx:smoothing:modifications}

In order to construct an algorithm which is a consistent estimator of the
graphon cluster tree in the sense made precise above, we need for the edge
probability estimator to be consistent in a stronger sense. In particular, we
need that for any $\epsilon > 0$, as $n \to \infty$
\[
    \prob\left(
        \max_{i \neq j} |\estedgeprobrv_{ij} - \edgeprobrv_{ij}|  >
    \epsilon\right) \to 0.
\]
In order to show that the neighborhood smoothing method satisfies such a notion
of consistency, one might attempt to apply a concentration inequality to bound
the difference between $\frac{1}{|\neighborhood_i|} \sum_{i' \in
\neighborhood_i} \adjacencyrv_{i'j}$ and $\edgeprobrv_{ij}$. The problem with
this approach, however, is that such concentration results require an assumption
of statistical independence that is not satisfied by the neighborhoods as
defined; that is, the terms of the sum $\sum_{i' \in \neighborhood_i}
\adjacencyrv_{i'j}$ are not statistically independent. It is true that,
unconditioned, $\adjacencyrv_{i'j}$ and $\adjacencyrv_{i''j}$ are independent Bernoulli random
variables. However, once we condition on the event $i' \in \neighborhood_i$ and $i'' \in
\neighborhood_i$, the random variables $\adjacencyrv_{i'j}$ and $\adjacencyrv_{i''j}$ are no longer independent.

More precisely, we are interested in
\begin{equation}
    \prob(
        \adjacencyrv_{i'j},
        \adjacencyrv_{i''j}
        \given
        i', i'' \in \neighborhood_i
    )
    =
    \frac{
    \prob(
        i', i'' \in \neighborhood_i
        \given
        \adjacencyrv_{i'j},
        \adjacencyrv_{i''j}
        )
    \prob(
        \adjacencyrv_{i'j},
        \adjacencyrv_{i''j}
        )
    }{
    \prob(
        i', i'' \in \neighborhood_i
        )
    }.
\end{equation}
The denominator of the RHS is a normalization constant which does not
depend on $\adjacencyrv_{i'j}$ or $\adjacencyrv_{i''j}$. Moreover, the entries
of $\adjacencyrv$ are
independent when unconditioned, and so $\prob(\adjacencyrv_{i'j},\adjacencyrv_{i''j}) =
\prob(\adjacencyrv_{i'j}) \prob(\adjacencyrv_{i''j})$. The difficulty is in 
\[
    \prob(
        i', i'' \in \neighborhood_i
        \given
        \adjacencyrv_{i'j},
        \adjacencyrv_{i''j}
        ).
\]
Intuitively, the event $i' \in \neighborhood_i$ depends on $\adjacencyrv_{i'j}$, and, likewise,
$i'' \in \neighborhood_{i}$ depends on $\adjacencyrv_{i''j}$. The reason is that $i' \in
\neighborhood_i$ when
$d(i,i')$ is small. But $d(i,i')$ depends on $\adjacencyrv_{i'j}$, since
\begin{align*}
    d(i,i')
    &=
        \max_{k \neq i,i'} |(\adjacencyrv^2/n)_{ik} - (\adjacencyrv^2/n)_{i'k}|,\\
    &=
        \frac{1}{n} 
        \max_{k \neq i,i'} 
        \left|
            \sum_{\ell = 1}^n 
            \adjacencyrv_{k \ell}
            \left(
                \adjacencyrv_{i\ell} 
                -
                \adjacencyrv_{i'\ell}
            \right)
        \right|.\\
\end{align*}
and so $\adjacencyrv_{i'j}$ enters the sum and $d(i,i')$ depends on it. In the extreme
case, suppose there are two nodes $i'$ and $i''$ such that the row vectors
$\adjacencyrv_{i'}$ and $\adjacencyrv_{i''}$ are identical except in their $j$th component. Then
the only difference between $d(i,i')$ and $d(i,i'')$ comes from the
difference in $\adjacencyrv_{i'j}$ and $\adjacencyrv_{i''j}$. Hence it is clear that $d(i,i')$ and
$d(i,i'')$ depend on the values of $\adjacencyrv_{i'j}$ and $\adjacencyrv_{i''j}$, and, by
extension, the events $i' \in \neighborhood_i$ and $i'' \in \neighborhood_i$ are \emph{not}
independent of $\adjacencyrv_{i'j}$ and $\adjacencyrv_{i''j}$.

Our modification of the algorithm is to change the way in which neighborhoods
are constructed so that statistical independence is preserved.  Instead of
constructing a neighborhood for each node $i$, we construct a neighborhood
$\neighborhood_{i \without j}$ for each \emph{ordered pair} $(i,j)$ by using a
parameterized distance function $d_j$ which ignores all information about node
$j$. More precisely, let $\partial_j \adjacencyrv$ represent the matrix obtained
by setting the $j$th row and column of $\adjacencyrv$ to zero. Then for every
node $j$ we define \begin{align*} d_j(i,i') &= \max_{k \neq i,i',j} |
([\partial_j \adjacencyrv]^2/n)_{ik} - ([\partial_j \adjacencyrv]^2/n)_{i'k}
|,\\ &= \frac{1}{n} \max_{k \neq i,i',j} \left| \sum_{\substack{ \ell = 1\\ \ell
\neq j }}^n \adjacencyrv_{k \ell} \left( \adjacencyrv_{i\ell} -
\adjacencyrv_{i'\ell} \right) \right|.\\ \end{align*} The important thing to
note here is that $\adjacencyrv_{i'j}$ does not appear in $d_j(i,i')$, and,
since the other entries of $\adjacencyrv$ are independent of
$\adjacencyrv_{i'j}$, we have that $d_j(i,i')$ is statistically independent of
$\adjacencyrv_{i'j}$. Therefore the event $i' \in \neighborhood_{i \without j}$ is
independent of $\adjacencyrv_{ij}$.

We are now interested in the quantity:
\begin{equation}
    \prob(
        \adjacencyrv_{i'j},
        \adjacencyrv_{i''j}
        \given
        i', i'' \in \neighborhood_{i \without j}
    )
    =
    \frac{
    \prob(
        i', i'' \in \neighborhood_{i \without j}
        \given
        \adjacencyrv_{i'j},
        \adjacencyrv_{i''j}
        )
    \prob(
        \adjacencyrv_{i'j},
        \adjacencyrv_{i''j}
        )
    }{
    \prob(
        i', i'' \in \neighborhood_{i \without j}
        )
    },
\end{equation}
where we are using the parameterized distance $d_j$ to build the
neighborhood $\neighborhood_{i \without j}$. In this case, we apply the independence argument
above to see that
\[
    \prob(
        i', i'' \in \neighborhood_{i \without j}
        \given
        \adjacencyrv_{i'j},
        \adjacencyrv_{i''j}
        ) 
    = 
    \prob(
        i', i'' \in \neighborhood_{i \without j}
        ).
\]
Therefore the denominator cancels with the term in the numerator, and we
have
\begin{equation}
    \prob(
        \adjacencyrv_{i'j},
        \adjacencyrv_{i''j}
        \given
        i', i'' \in \neighborhood_{i \without j}
    )
    =
    \prob(
        \adjacencyrv_{i'j},
        \adjacencyrv_{i''j}
        )
    =
    \prob(\adjacencyrv_{i'j})
    \prob(\adjacencyrv_{i''j}).
\end{equation}
Therefore $\adjacencyrv_{i'j}$ and $\adjacencyrv_{i''j}$ are independent even
when conditioning on the event $i' \in \neighborhood_{i \without j}$ and $i''
\in \neighborhood_{i \without j}$. This allows us to apply a concentration
inequality to bound each entry of $\hat P - P$, and a max norm result follows
after a simple union bound.

In total, the modified neighborhood smoothing procedure is as follows:
Fix some neighborhood size parameter $h$ and let $q_{i \without j}(h)$ denote
the $h$-th quantile of the set $\{d_{j}(i,i') : i' \neq i,j \}$.  Construct the
neighborhood $\neighborhood_{i \without j}$ by setting
\[
    \neighborhood_{i \without j} = 
        \{ i' \neq i,j : d_{j}(i,i') \leq q_{i \without j}(h) \}.
\]
Then set
\[
    \hat{P}_{ij} = 
    \frac{1}{2}
    \left(
        \frac{1}{|\neighborhood_{i \without j}|} 
        \sum_{i' \in \neighborhood_{i \without j}} A_{i'j}
    +
        \frac{1}{|\neighborhood_{j \without i}|} 
        \sum_{j' \in \neighborhood_{j \without i}} A_{ij'}
    \right).
\]
We will show that this estimator of the edge probability matrix is consistent in
max-norm.

\subsection{Claims}

There are two major components to the analysis. First, we show that, with high
probability, each neighborhood $\neighborhood_{i \without j}$ consists only of
nodes $i'$ for which $\inftynorm{\edgeprobrv_i - \edgeprobrv_{i'}} < \epsilon$,
with $\epsilon \to 0$ as $n \to \infty$; The formal statement of this result is
made in Claims~\ref{claim:two_neighborhood}~and~\ref{claim:infty_neighborhood}
below.  This is an extension of the analysis in \cite{Zhang2015-ik}, where it is
shown that the neighborhood $\neighborhood_i$ consists only of nodes $i'$ for
which $\nicefrac{1}{n}\twonorm{\edgeprobrv_i - \edgeprobrv_{i'}} < \epsilon$,
with $\epsilon \to 0$ as $n \to \infty$. The procedure for proving this result
parallels that of \cite{Zhang2015-ik}, however, the modifications we make to the
algorithm -- namely, the deletion of a node from the graph -- mean that the
claims in that paper do not directly transfer. Much of the analysis consists of
making the minor changes necessary to show that analogous versions of the
claims in \cite{Zhang2015-ik} hold for our modified algorithm.

The second part of the analysis uses concentration inequalities to derive the
consistency result. In particular, Claim~\ref{claim:smoothing_close}
shows that smoothing within neighborhoods produces an estimate of the edge
probability matrix which is close within max-norm, provided that each
neighborhood consists only of nodes which are sufficiently similar in the sense
described above. Theorem~\ref{thm:smoothing} puts these two claims together to
derive the main result.

The technical details of the analysis are in \Cref{apx:smoothing:proofs}. 
In particular, all proofs of the following claims can be found there.

\subsubsection{Sample requirements}

\newcommand{\msq}[2]{\left[ \left(\partial_{#2} {#1} \right)^2 \right]}
\newcommand{\nmsq}[2]{\left[ \left(\partial_{#2} {#1} \right)^2 / n \right]}
\newcommand{\mna}[1][j]{\nmsq{\adjacencyrv}{#1}}
\newcommand{\mnp}[1][j]{\nmsq{P}{#1}}

The analysis will require the notion of a block partition and
$\Delta$-refinement as defined in \Cref{defn:block_partition} and
\Cref{defn:delta_refinement}, respectively, both in
\Cref{apx:estprob:definitions}. If $\refinement$ is a block partition and $x \in
[0,1]$, we write $\refinement(x)$ to denote the block $R \in \refinement$ which
contains $x$. Some of the following results will include an assumption that
there are ``enough'' samples in each block of a partition. We formalize this
notion as follows:

\begin{defn}
    If $\rv{S}$ is an ordered set of random samples from the uniform
    distribution on the unit interval, and $\blocks$ is any block partition, we
    say that $\rv{S}$ is a \term{$\rho$-dense sample in $\blocks$} if for any
    block $B \in \blocks$, 
    \[
        \frac{|B \cap \rv{S}|}{n} > (1 - \rho)\mu(B).  
    \]
\end{defn}

If we fix any $\rho$ and a $\Delta$-block partition, a random sample $\rv{S}$
will be $\rho$-dense with high probability as the size of the sample $n \to
\infty$, as the following result shows:

\begin{restatable}{claim}{claimsampledensity}
    \label{claim:sample_density}
    Let $\blocks$ be a $\Delta$-block partition. Let $\rho < 1$. Then with
    probability $1 - \frac{2}{\Delta}e^{-2n\rho^2\Delta^2}$, $\samplerv$ is a
    $\rho$-dense sample of $\blocks$. That is, for all $B \in \blocks$
    simultaneously,
    \[
        \frac{|B \cap \samplerv|}{n} > (1 - \rho)|B|.
    \]
\end{restatable}

\subsubsection{The adjacency column distance}

In the previous section detailing our modified neighborhood smoothing algorithm,
we introduced the following distance $d_j(i,i')$ between columns of the
adjacency matrix.  For a square matrix $M$, let $\partial_v M$ denote the matrix
obtained by replacing the $v$-th row and column of the matrix $M$ with zeros.
Then we define

\[
    d_j(i,i') = \max_{k \neq i,i'} \left| \mna_{ik} - \mna_{i'k} \right|.
\]

This pattern -- the maximum elementwise difference of normalized squared
matrices -- will reoccur in the analysis. We therefore define:
\begin{defn}
    Let $M_1$ and $M_2$ be $n \times n$ matrices. We define
\[
    D(M_1, M_2) = 
    \max_{i,j}
    \left|
    \left[ M_1^2 / n \right]_{ij} -
    \left[ M_2^2 / n \right]_{ij}
    \right|.
\]
\end{defn}

A key observation in the analysis of \cite{Zhang2015-ik} is that if
$\adjacencyrv$ is sampled from $P$, then $\prob(D(\adjacencyrv, P) < \epsilon)
\to 0$ as $n \to \infty$. In our analysis, however, we will work with
$\partial_k \adjacencyrv$ and $\partial_k \edgeprob$, which are the adjacency
and edge probability matrices with the $k$th row and column set to zero. We
therefore have a slightly modified claim:

\newcommand{\rtlog}{\sqrt{\frac{(C_2 + 2) \log n}{n}}}
\begin{restatable}{claim}{claimdifferenceadjprob}
    \label{claim:difference_adj_prob}
    Let $P$ be an arbitrary $n \times n$ edge probability matrix.
    Let $C_2 > 0$ be an arbitrary constant and suppose $n$ is large enough that
    $\rtlog \leq 1$. Then, with probability $1 - 2n^{-C_2/4}$ over random
    adjacency matrices $\adjacencyrv$ sampled from $P$, for all $k \in [n]$
    simultaneously,
    \[
        D(\partial_k \adjacencyrv, \partial_k P) = 
        \max_{i \neq j}
        \left| \mna[k]_{ij} - \mnp[k]_{ij} \right| \leq 
        \rtlog + \frac{6}{n}.
    \]
\end{restatable}

\subsubsection{Composition of neighborhoods}

Another key step in the analysis of \cite{Zhang2015-ik} is that, with high
probability, for any $i'$ in the neighborhood of node $i$, $\nicefrac{1}{n}
\twonorm{P_{i'} - P_{i}^2} = O(\sqrt{\nicefrac{\log n}{n}})$. We derive a
similar result for our modified neighborhoods:

\begin{restatable}{claim}{claimtwoneighborhood}
    \label{claim:two_neighborhood}
    Let $W \in \smoothgraphons$ and let $\refinement$ be a $\Delta$-refinement
    of $\blocks$. Suppose $S$ is a $\rho$-dense sample of $\refinement$ and let
    $P$ be the induced edge probability matrix.  Suppose $A$ is an adjacency
    matrix such that $D(\partial_k A, \partial_k P) < \epsilon$ for every $k \in
    [n]$.  Pick $0 < h \leq \rho \Delta$, and construct for every pair $i,j$ a
    neighborhood $\neighborhood_{i \without j}$ as described above, including
    all nodes within the $h$-th quantile.  Then for all $i,j$ and any $i' \in
    \neighborhood_{i \without j}$ we have
    \[
        \frac{1}{n} \twonorm{P_i - P_{i'}}^2 \leq
        6\lipschitz\Delta + 8\epsilon + \frac5n.
    \]
\end{restatable}

Additionally, we prove that neighborhoods are composed of nodes whose
corresponding columns of $P$ are close in $\infty$-norm. This follows from the
previous claim after leveraging the piecewise Lipschitz condition.

\begin{restatable}{claim}{claiminftyneighborhood}
    \label{claim:infty_neighborhood}
    Let $W \in \smoothgraphons$ and let $\refinement$ be a $\Delta$-refinement
    of $\blocks$. Suppose $S$ is a $\rho$-dense sample of $\refinement$. Then
    for any $\epsilon \geq 4 \rho \Delta^3 \lipschitz^2$, if $i \neq j$ are
    such that 
    $
        \frac{1}{n}\twonorm{P_i - P_j}^2 \leq \epsilon,
    $
    then 
    $
    \inftynorm{P_i - P_j}^2 \leq \frac{4\epsilon}{\rho \Delta}.
    $
\end{restatable}

\subsubsection{Main result}

Intuitively, if every neighborhood $\neighborhood_{i \without j}$ is composed of
nodes whose corresponding columns of $P$ are close in $\infty$-norm, and whose
$j$th elements are statistically independent, we may apply a concentration
inequality to conclude that the estimate $\estedgeprobrv_{ij}$ is close to
$\edgeprob_{ij}$. The following claim makes this precise.

\newcommand{\toterr}{\tilde{\epsilon}}
\newcommand{\allneighborhoods}[2]{\mathscr{N}_{#1 \without #2}}
\newcommand{\allcloseneighborhoods}[2]{\mathscr{N}_{#1 \without #2}^\epsilon}
\newcommand{\error}[2]{\ell_{#1 \without #2}}
\newcommand{\ourneighborhood}[2]{\neighborhood_{#1 \without #2}}

\begin{restatable}{claim}{claimsmoothingclose}
    \label{claim:smoothing_close}
    Let $W \in \smoothgraphons$ and let $\refinement$ be a $\Delta$-refinement
    of $\blocks$. Let $S \in [0,1]^n$ be fixed, and let $P$ be
    the edge probability matrix induced by $S$.  Assume that with probability $1
    - \delta$ over graphs generated from $P$, that for all $i \neq j$
    simultaneously, $\inftynorm{P_i - P_{i'}} < \epsilon$ for all $i' \in
    \ourneighborhood{i}{j}$. Then with probability at least $(1 - \delta)\left[1
    - 2n(n-1)e^{-2hnt^2}\right]$,
    \[
        \max_{ij} \left| \estedgeprobrv_{ij} - P_{ij} \right| < \epsilon + t.
    \]
\end{restatable}

We combine all of the previous claims to derive our main result.

\begin{restatable}{theorem}{thmsmoothing}
    \label{thm:smoothing}
    Let $W \in \smoothgraphons$.
    Let $\edgeprobrv$ be the random edge probability matrix arising by sampling
    a graph of size $n$ from $W$ according to the graphon sampling procedure,
    and denote by $\estedgeprobrv$ the estimated edge probability using our
    modified neighborhood smoothing method. Then
    \[
        \max_{i \neq j}
        \left|
        \estedgeprobrv_{ij} - \edgeprobrv_{ij}
        \right| = 
        O_\text{p}\left(
            \left[
                \frac{\log n}{n}
            \right]^{1/6}
        \right).
    \]
\end{restatable}

\subsection{Supplementary claims}

The following claims will be used in the proofs of \Cref{apx:smoothing:proofs},
and are gathered here for convenience. The proofs of these claims are located in
\Cref{apx:smoothing:proofs} as well.

\begin{restatable}{claim}{claimdenserefinement}
    \label{claim:sub_refinement}
    Let $\refinement_2$ be a block partition.  Suppose $\refinement_1$ is a
    $\Delta$-refinement of $\refinement_2$. If a $S$ is a $\rho$-dense sample of
    $\refinement_1$, then it is also a $\rho$-dense sample of $\refinement_2$.
\end{restatable}

\begin{restatable}{claim}{claimwithoutrownorm}
    \label{claim:without_row_norm}
    Let $M$ be an $n \times n$ matrix with values in $[0,1]$. Then for any
    distinct $u,u',v \in [n]$,
    \[
        \twonorm{(\partial_v M)_u - (\partial_v M)_{u'}}^2 \geq
        \twonorm{M_u - M_{u'}}^2 - 1
    \]
\end{restatable}

\begin{restatable}{claim}{claimwithoutrowsquare}
    \label{claim:without_row_square}
    Let $M$ be an $n \times n$ symmetric matrix with values in $[0,1]$. Then for
    any distinct $i,j,k \in [n]$,
    \[
        \mats{M^2}_{ij} - 1 
        \leq \msq{M}{k}_{ij} 
        \leq \mats{M^2}_{ij}.
    \]
\end{restatable}

\begin{restatable}{claim}{claimnormexpansion}
    \label{claim:norm_expansion}
    Let $M$ be an $n \times n$ symmetric matrix. Then for any distinct $i,i' \in
    [n]$,
    \[
        \twonorm{M_i - M_{i'}}^2 = 
        \mat{M^2}_{ii} - 
        2\mat{M^2}_{ii'} + 
        \mat{M^2}_{i'i'}.
    \]
\end{restatable}

\begin{restatable}{claim}{claimdistancelowerbound}
    \label{claim:distance_lower_bound}
    Let $W \in \smoothgraphons$. Suppose $\refinement$ is a $\Delta$-refinement
    of $\blocks$.
    Let $S = (x_1, \ldots, x_n)$ be a fixed sample. Fix $\Delta > 0$ and assume
    that $|\refinement(x_i) \cap S| \geq 4$ for every $i \in [n]$. Let $P$ be
    the edge probability matrix induced by $S$. Let $A$ be an adjacency matrix,
    and suppose that $A$ is such that $D(\partial_k A, \partial_k P) <
    \epsilon$ for all $k \in [n]$. Then for all $i \neq j \neq k$
    simultaneously, 
    \[
        2 d_k(i,j) + \frac1n + 4 \lipschitz \Delta + 4 \epsilon
        \geq \frac{1}{n} \twonorm{\edgeprob_i - \edgeprob_j}^2
    \]
    for $d_k$ computed w.r.t. $A$.
\end{restatable}

\begin{restatable}{claim}{claimdistanceupperbound}
    \label{claim:distance_upper_bound}
    Let $W \in \smoothgraphons$. Fix a sample $S = (x_1, \ldots, x_n)$ and let
    $P$ be the induced edge probability matrix. Suppose that $\refinement$ is a
    $\Delta$-refinement of $\blocks$. Now suppose that nodes $x_i$ and $x_{i'}$
    are from the same $\refinement(x_{i''})$ for some $i''$. Furthermore,
    suppose that $A$ is an adjacency matrix with the property that 
    $D(\partial_j A, \partial_j P) \leq \epsilon$ for all $j \in [n]$. Then for
    all $j \neq i,i'$,
    \[
        d_j(i,i') \leq \lipschitz \Delta + 2 \epsilon + \frac2n.
    \]
\end{restatable}

\subsection{Proofs}
\label{apx:smoothing:proofs}

\claimsampledensity*
\begin{proof}
    Let $B$ be an arbitrary block in the partition $\blocks$. Since
    $\blocks$ is a $\Delta$-partition, the size of any block is between
    $\Delta/2$ and $\Delta$. Therefore there are at most $2/\Delta$ blocks in
    $\blocks$.

    The membership of any given sample in $B$ is a Bernoulli trial with
    probability $|B|$ of success. Applying Hoeffding's inequality:
    \[
        \prob\left(
            \left|
            \frac{1}{n} \left| B \cap \samplerv \right|
            -
            |B|
            \right| > \epsilon
            \right)
            < e^{-2n\epsilon^2}.
    \]
    Choose $\epsilon = \rho \Delta$. This gives
    \[
        \prob\left(
            \left|
            \frac{1}{n} \left| B \cap \samplerv \right|
            -
            |B|
            \right| > \rho \Delta
            \right)
            < e^{-2n\rho^2\Delta^2},
    \]
    which implies
    \[
        \prob\left(
            \frac{1}{n} \left| B \cap \samplerv \right|
            > |B| - \rho \Delta
            \right)
            < e^{-2n\rho^2\Delta^2}.
    \]
    Now, $|B| \leq \Delta$, so that for any arbitrary $B$ it is true that
    \[
        \prob\left(
            \frac{1}{n} \left| B \cap \samplerv \right|
            > |B| - \rho |B|
            \right)
            < e^{-2n\rho^2\Delta^2}.
    \]
    The result follows by applying a union bound over all blocks of the
    partition, of which there are at most $2/\Delta$.
\end{proof}

% \claimaltdistance*
% 
% \begin{proof}
%     \begin{align*}
%         d_j(i,i') &= \frac{1}{n} \max_{k \neq i,i',j} | 
%             \langle 
%                 \adjacencyrv_i - \adjacencyrv_{i'}, 
%                 (1 - e_j) \hadamard \adjacencyrv_{k} 
%             \rangle|\\
%         &= \frac{1}{n} \max_{k \neq i,i',j} | 
%             \langle \adjacencyrv_i, (1 - e_j) \hadamard \adjacencyrv_k \rangle - 
%             \langle \adjacencyrv_{i'}, (1 - e_j) \hadamard \adjacencyrv_k \rangle|\\
%         &= \max_{k \neq i,i',j}
%             \left|
%                 \frac{1}{n}\sum_{t \neq j} \adjacencyrv_{it} \adjacencyrv_{kt} - 
%                 \frac{1}{n}\sum_{t \neq j} \adjacencyrv_{i't} \adjacencyrv_{kt}
%             \right|\\
%         &= \max_{k \neq i,i',j} \left|\mna_{ik} - \mna_{i'k}\right|
%     \end{align*}
% \end{proof}

\claimdifferenceadjprob*

\begin{proof}
    The proof of Lemma 5.2 in \cite{Zhang2015-ik} establishes that, given the
    above assumptions, with probability $1 - 2n^{C_2/4}$,
    \[
        D(\adjacencyrv, P) = 
        \max_{i \neq j} \left| \mats{\adjacencyrv^2/n}_{ij} -
        \mats{\edgeprob^2/n}_{ij} \right| \leq \rtlog + \frac{4}{n}.
    \]
    From Claim~\ref{claim:without_row_square}, for all $k$,
    \begin{align*}
        \left| \mats{\adjacencyrv^2 /n}_{ij} - \mna[k]_{ij} \right| &\leq
        \frac{1}{n},\\
        \left| \mats{\edgeprob^2 /n}_{ij} - \mnp[k]_{ij} \right| &\leq \frac1n,
    \end{align*}
    and so, with probability $1 - 2n^{-C_2/4}$ ,
    \begin{align*}
        \max_{i \neq j} \left| \mna[k]_{ij} -
            \mnp[k]_{ij} \right|
            &\leq \max_{i \neq j} 
            \left|\mats{\adjacencyrv^2 /n}_{ij} -
            \mats{\edgeprob^2 / n}_{ij} \right| + \frac{2}{n}\\
        &\leq 
        \rtlog + \frac{6}{n}.
    \end{align*}
\end{proof}

\claimtwoneighborhood*

\begin{proof}
    We start by applying Claim~\ref{claim:distance_lower_bound}, which yields
    \[
        \frac{1}{n} \twonorm{P_i - P_{i'}}^2 
        \leq 
        2 d_j(i,i') + 
        \frac1n
        +
        4\lipschitz \Delta + 4 \epsilon
    \]
    Now we upper bound $d_j(i,i')$. Since we have assumed that $h \leq \rho$,
    at least a fraction $h$ of the nodes are within $i$'s partition in the
    refinement. Therefore the distance between any two nodes in the neighborhood
    is bounded above by the maximum distance between two nodes in this
    partition. This was computed in Claim~\ref{claim:distance_upper_bound}, such
    that:
    \begin{align*}
        \frac{1}{n} \twonorm{P_i - P_{i'}}^2 &\leq 2 \left(
            \lipschitz \Delta + 2 \epsilon + \frac2n
        \right)
        + \frac1n + 4\lipschitz\Delta + 4\epsilon\\
        &= 6\lipschitz\Delta + 8\epsilon + \frac5n.
    \end{align*}
\end{proof}

\claiminftyneighborhood*

\begin{proof}
    Suppose $\epsilon \geq 4 \rho \Delta^3 \lipschitz^2$.
    Define $\alpha = \sqrt{4 \epsilon / (\rho \Delta)}$ and suppose that
    $\inftynorm{P_i - P_j} > \alpha$. This implies that there exists a $k$ such
    that $|P_{ik} - P_{jk}| > \alpha$. Consider any $k' \in \refinement(k)$.
    Then
    \begin{align*}
        |P_{ik'} - P_{jk'}| &= 
        \left| (P_{ik'} - P_{ik}) + P_{ik} - (P_{jk'} - P_{jk}) - P_{jk}
            \right|\\
        &= \left| (P_{ik'} - P_{ik}) + (P_{jk} - P_{jk'}) + (P_{ik} -
            P_{jk}) \right|
        \intertext{Since $|x_k - x_{k'}| < \Delta$ by virtue of being in the
            same block $\refinement(x_k)$, we have $|P_{ik'} - P_{ik}| \leq
            \Delta \lipschitz$. But by assumption, $\Delta \leq \alpha / (4
            \lipschitz)$. Therefore, $|P_{ik'} - P_{ik}| \leq \alpha/4$.
            Similarly, $|P_{jk'} - P_{jk}| \leq \alpha/4$. The last term
            satisfies $|P_{ik} - P_{jk}| > \alpha$. Therefore the entire
            quantity must be at least:} 
            &> \alpha / 2.
    \end{align*}
    Now consider
    \begin{align*}
        \frac{1}{n} \twonorm{P_i - P_j}^2 
        &= \frac{1}{n} \sum_l (P_{il} - P_{jl})^2\\
        &\geq \frac{1}{n} \sum_{k' \in \refinement(k)} (P_{ik'} - P_{jk'})^2\\
        \intertext{But, as established above, each term in the sequence is at
        least $\alpha/2$, and so:}
        &> \frac{1}{n} \sum_{k' \in \refinement(k)} \alpha^2/4
        \intertext{Since $S$ is assumed to be a $\rho$-dense sample of
        $\refinement$, there are at least $\rho \Delta n$ elements in
        $\refinement(k)$. Therefore:}
        &\geq \frac{\rho \Delta \alpha^2}{4} = \epsilon
    \end{align*}
    The claim follows from the contrapositive.
\end{proof}

\claimsmoothingclose*

\begin{proof}
    Consider an arbitrary ordered pair of nodes $i \neq j$. The neighborhood
    $\ourneighborhood{i}{j}$ is a random variable, since it depends on the
    random adjacency matrix $\adjacencyrv$. Define $\error{i}{j}$ to be
    the amount by which our smoothed estimate computed using
    $\ourneighborhood{i}{j}$ differs from $P_{ij}$:
    \[
        \error{i}{j} = \left|
        P_{ij} - 
        \frac{1}{\ourneighborhood{i}{j}} 
        \sum_{i' \in \ourneighborhood{i}{j}} \adjacencyrv_{i'j}
        \right|.
    \]
    Note that $\error{i}{j}$ is itself a random variable, and we seek to compute
    \[
        \prob\left( \max_{i \neq j} \error{i}{j} < \toterr \right),
    \]
    where it will be assumed that $\toterr > \epsilon$. 

    Denote by $\allneighborhoods{i}{j}$ the subset of $\powerset{[n]}$
    consisting of all possible values of the neighborhood
    $\ourneighborhood{i}{j}$ over all graphs on $[n]$. Denote by
    $\allcloseneighborhoods{i}{j}$ the subset of $\allneighborhoods{i}{j}$
    consisting of neighborhoods with the property that that if $i'$ is in the
    neighborhood, then $\inftynorm{P_i - P_{i'}} < \epsilon$. Then
    \begin{align*}
        \prob\left( \max_{i \neq j} \error{i}{j} < \toterr \right)
        &\geq
        \prob\left( \max_{i \neq j} \error{i}{j} < \toterr \given[\middle]
            \forall \, i \neq j, \:
            \ourneighborhood{i}{j} \in \allcloseneighborhoods{i}{j}
            \right)
            \prob\left(
                \forall \, i \neq j, \:
                \ourneighborhood{i}{j} \in \allcloseneighborhoods{i}{j}
            \right)\\
        &\geq
            \prob\left( \max_{i \neq j} \error{i}{j} < \toterr \given[\middle]
            \forall \, i \neq j, \:
            \ourneighborhood{i}{j} \in \allcloseneighborhoods{i}{j}
            \right)
            (1 - \delta).
    \end{align*}
    
    We now lower bound the probability that an arbitrary pair $u \neq v$ is such
    that $\error{u}{v} < \toterr$. The result will then follow from a union
    bound. That is, we would like to compute, for arbitrary $u \neq v$, the
    probability
    \begin{align*}
        \prob\left( \error{u}{v} < \toterr \given[\middle]
        \forall \, i \neq j, \:
        \ourneighborhood{i}{j} \in \allcloseneighborhoods{i}{j}
        \right)
        &=
            \prob\left( \error{u}{v} < \toterr \given[\middle]
            \ourneighborhood{u}{v} \in \allcloseneighborhoods{u}{v}
            \right)
        \intertext{We decompose this quantity as a sum over all neighborhoods in
            $\allcloseneighborhoods{u}{v}$:}
        &=
            \sum_{N \in \allcloseneighborhoods{u}{v}}
            \prob\left( \error{u}{v} < \toterr \given[\middle]
            \ourneighborhood{u}{v} = N
            \right)
            \prob\left(
            \ourneighborhood{u}{v} = N \given[\middle]
            \ourneighborhood{u}{v} \in \allcloseneighborhoods{u}{v}
            \right)
    \end{align*}

    We now claim that, conditioned on a particular neighborhood $N$, the random
    variables $\adjacencyrv_{u_1v}$ and $\adjacencyrv_{u_2v}$ are independent.
    We may then apply Hoeffding's inequality to conclude:
    \[
        \prob\left(
        \left|
        \frac{1}{|N|} \sum_{u' \in N} 
        (\adjacencyrv_{u'v} - P_{u'v})
        \right| > t
        \right)
        < e^{-2hnt^2}.
    \]
    Where we have used the fact that there are at least $hn$ nodes in the
    neighborhood $N$. By the assumption that $|P_{uv} - P_{u'v}| < \epsilon$
    for any $u \in N$, we have:
    \[
        \prob\left(
        \left|
        P_{uv} - 
        \frac{1}{|N|} \sum_{u' \in N} 
        \adjacencyrv_{u'v}
        \right| > t + \epsilon
        \right)
        < e^{-2hnt^2}.
    \]
    So that
    \begin{align*}
        \prob\left( \error{u}{v} < \toterr \given[\middle]
        \forall \, i \neq j, \:
        \ourneighborhood{i}{j} \in \allcloseneighborhoods{i}{j}
        \right)
        &=
            \sum_{N \in \allcloseneighborhoods{u}{v}}
            \prob\left( \error{u}{v} < \toterr \given[\middle]
            \ourneighborhood{u}{v} = N
            \right)
            \prob\left(
            \ourneighborhood{u}{v} = N \given[\middle]
            \ourneighborhood{u}{v} \in \allcloseneighborhoods{u}{v}
            \right)\\
        &>
            \left(1 - e^{-2hnt^2}\right) 
            \sum_{N \in \allcloseneighborhoods{u}{v}}
            \prob\left(
            \ourneighborhood{u}{v} = N \given[\middle]
            \ourneighborhood{u}{v} \in \allcloseneighborhoods{u}{v}
            \right)\\
        &=
            \left(1 - e^{-2hnt^2}\right) 
    \end{align*}
    Now, returning to:
    \begin{align*}
        \prob\left( \max_{i \neq j} \error{i}{j} < \toterr \right)
        &\geq
            \prob\left( \max_{i \neq j} \error{i}{j} < \toterr \given[\middle]
            \forall \, i \neq j, \:
            \ourneighborhood{i}{j} \in \allcloseneighborhoods{i}{j}
            \right)
            (1 - \delta)
        \intertext{We apply a union bound over all $2n(n-1)$ ordered pairs to
            obtain:}
        &>
            (1 - \delta) \left(1 - 2n(n-1)\left[1 - e^{-2hnt^2}\right]\right)
    \end{align*}
\end{proof}

\thmsmoothing*

\begin{proof}

    The mechanism of the proof involves a translation from the $L^2$ result of
    \cite{Zhang2015-ik} to our desired max-norm result. To accomplish this, we
    will make use of two discretizations at different scales. First, define
    arbitrary constants $\alpha_2, \alpha_\infty > 0$ and $0 < \rho < 1$ such
    that $\rho \cdot \alpha_2 > \frac{1}{2}$, and let
    \[
        \Delta_2(n) = \alpha_2 \sqrt{\frac{\log n}{n}},
        \qquad
        \Delta_\infty(n) = \alpha_\infty \left(\frac{\log n}{n}\right)^{1/6},
    \]
    for any $n \geq 2$. For each $n \geq 2$, let $\refinement_\infty(n)$ be an
    arbitrary $\Delta_\infty(n)$-refinement of $\blocks$, and let
    $\refinement_2(n)$ be an arbitrary $\Delta_2(n)$-refinement of
    $\refinement_\infty(n)$. In what follows we will drop the functional
    notation, as the dependence of these quantities on $n$ should be clear.

    Let $\samplerv$ be a random sample of $[0,1]$. Then, according to
    Claim~\ref{claim:sample_density}, $\samplerv$ is $\rho$-dense in
    $\refinement_2$ with probability
    \begin{align*}
        1 - \frac{2}{\Delta_2} e^{-2n\rho^2\Delta_2^2}
        &=
            1 - 2 \alpha_2 \sqrt{\frac{n}{\log n}}
            e^{-2 n \rho^2 \alpha_2^2 \frac{\log n}{n}}\\
        &= 
            1 - 2 \alpha_2 \sqrt{\frac{n}{\log n}} n^{-2 \alpha_2^2 \rho^2}\\
        &\geq
            1 - 2 \alpha_2 \sqrt{n} \cdot n^{-2 \alpha_2^2 \rho^2}\\
        &= 1 - 2 \alpha_2 n^{\frac{1}{2} - 2 \alpha_2^2 \rho^2}
    \end{align*}
    Since $\rho \cdot \alpha_2 > 1/2$ by assumption, this is a decreasing
    function in $n$. 
    
    We have so-far shown that a sample is ``good'' with high probability in the
    sense that it is $\rho$-dense in $\refinement_2$. We now show that, assuming
    the sample $\samplerv = S$ is a fixed, $\rho$-dense sample of
    $\refinement_2$, the estimate $\estedgeprob$ is good in max-norm with
    high probability over random graphs sampled according to the distribution
    induced by $S$.

    We begin by showing that, with high probability, the neighborhood around
    node $i$ contains only nodes $i'$ such that $\edgeprob_i$ and
    $\edgeprob_{i'}$ are close in 2-norm, which will follow from combining
    Claims~\ref{claim:two_neighborhood}~and~\ref{claim:difference_adj_prob}.
    We will use this result to invoke Claim~\ref{claim:infty_neighborhood},
    which says that, for $i'$ in the neighborhood of $i$, $\edgeprob_i$ and
    $\edgeprob_{i'}$ are close in $\infty$-norm. This will in turn satisfy the
    assumptions of Claim~\ref{claim:smoothing_close}, which shows that
    $\estedgeprob$ is close to $\edgeprob$.

    First, we combine
    Claims~\ref{claim:two_neighborhood}~and~\ref{claim:difference_adj_prob} to
    show that, with high probability, $\frac{1}{n}\twonorm{P_i - P_{i'}}^2$ is
    small when $i'$ is in $\neighborhood_{i \without j}$.
    Fix an arbitrary constant $C_2 > 0$ and suppose that $n$ is large enough
    that $\sqrt{(C_2 + 2) \log n / n} \leq 1$. Then
    Claim~\ref{claim:difference_adj_prob} says that, with probability $1 -
    2n^{-C_2/4}$ over random adjacency matrices $\adjacencyrv$ generated by $P$,
    for all $k \in [n]$ simultaneously,
    \[
        D(\partial_k \adjacencyrv, \partial_k P) \leq 
        \sqrt{\frac{(C_2 + 2) \log n}{n}} + \frac{6}{n}.
    \]
    Using $\refinement_2$ as the partition in
    Claim~\ref{claim:two_neighborhood}, we find that this implies that the
    adjacency matrix $\adjacencyrv$ is such that for any $i,j$ and $i' \in
    \neighborhood_{i \without j}$,
    \begin{align*}
        \frac{1}{n} \twonorm{P_i - P_{i'}}^2 
        &\leq 
            6 \lipschitz \Delta_2 + 8 \left( 
                \sqrt{\frac{(C_2 + 2) \log n}{n}} + \frac6n
                \right) +
            \frac5n\\
        &\leq
            6 \lipschitz \alpha_2 \sqrt{\frac{\log n}{n}} + 
                8 \sqrt{\frac{(C_2 + 2) \log n}{n}} +
                \frac{53}n\\
        &=
                \left(6\lipschitz \alpha_2 + 8 \sqrt{C_2 + 2}\right)
                \sqrt{\frac{\log n}{n}} + \frac{53}{n}\\
        &\leq
                \tilde{\alpha}_2 \sqrt{\frac{\log n}{n}}
    \end{align*}
    where $\tilde{\alpha}_2$ is an arbitrary constant greater than $6\lipschitz
    \alpha_2 + 8\sqrt{C_2 + 2}$, and assuming that $n$ is large enough that
    $53/n \leq \tilde{\alpha}_2 - 6\lipschitz \alpha_2 + 8\sqrt{C_2 + 2}$.  
    
    Now we may invoke Claim~\ref{claim:infty_neighborhood} using
    $\refinement_\infty$ as the refinement of $\blocks$. Define $\gamma = \max
    \{ \tilde{\alpha}_2,  4 \rho \alpha_\infty^3 c^2 \}$ and let
    $\tilde{\epsilon} = \gamma \sqrt{\frac{\log n}{n}}$. Then, from the previous
    result, for any $i' \in \neighborhood_{i \without j}$, $\frac{1}{n}
    \twonorm{P_i - P_{i'}}^2 \leq \gamma \sqrt{\frac{\log n}{n}}$. Furthermore,
    \[
        \tilde{\epsilon} = \gamma \sqrt{\frac{\log n}{n}} 
        \geq
            4 \rho \alpha_\infty^3 c^2 \sqrt{\frac{\log n}{n}}
        =
            4 \rho c^2 \left[ 
                \alpha_\infty \left(\frac{\log n}{n} \right)^{1/6} \right]^3
        = 4 \rho c^2 \Delta_\infty^3
    \]
    and so we may use the claim to conclude that, with probability at least $1 -
    2n^{-C_2/4}$ over graphs generated from $P$, for all $i,j \in [n]$ and any
    $i' \in \neighborhood_{i \without j}$,
    \[
        \inftynorm{P_i - P_{i'}}^2 
        \leq
            \frac{4 \gamma}{\rho \Delta_\infty} \sqrt{\frac{\log n}{n}}
        =
            \frac{4 \gamma}{\rho \cdot \alpha_\infty} 
            \left( \frac{\log n}{n} \right)^{1/3}.
    \]

    Now we may apply Claim~\ref{claim:smoothing_close}. Let $\alpha_t$ be an
    arbitrary constant, and choose
    \[
        t = \alpha_t \left( \frac{\log n}{n} \right)^{1/6}.
    \]
    Then, with probability
    \[
        \left( 1 - 2n^{-C_2/4} \right) 
        \left( 
            1 - n^{2 - 2 h \alpha_t \left(\frac{n}{\log n}\right)^{2/3}}
        \right),
    \]
    it holds that
    \[
        \max_{ij} \left| \estedgeprobrv_{ij} - P_{ij} \right|
        <
        \left(
                \alpha_t + \sqrt{\frac{4 \gamma}{\rho \cdot \alpha_\infty}}
                \,
            \right)
            \left(
                \frac{\log n}{n}
            \right)^{1/6}.
    \]

    The probability over all samples and graphs is therefore
    \[
        \left( 1 - 2n^{-C_2/4} \right) 
        \left( 
            1 - n^{2 - 2 h \alpha_t \left(\frac{n}{\log n}\right)^{2/3}}
        \right)
        \left(
            1 - 2 \alpha_2 \sqrt{\frac{n}{\log n}} n^{-2 \alpha_2^2 \rho^2}
        \right).
    \]
\end{proof}

\claimdenserefinement*
\begin{proof}
    Suppose $S$ is a $\rho$-dense sample of $\refinement_1$. Take any block $R
    \in \refinement_2$. Then $R$ is the disjoint union of blocks in
    $\refinement_1$:
    \[
        R = R_1 \cup \ldots \cup R_t
    \]
    where $R_i \in \refinement_1$. Each $R_i$ is such that $|R_i| \leq \Delta$.
    Therefore:
    \[
        \frac{|R \cap S|}{n} = \sum_i \frac{|R_i \cap S|}{n} 
        \geq \sum_i (1 - \rho)|R_i| = (1 - \rho) \sum_i |R_i|
        = (1 - \rho) |R|.
    \]
    Therefore $S$ is a $\rho$-dense sample of $\refinement_2$.
\end{proof}

\claimwithoutrownorm*

\begin{proof}
    \begin{align*}
        \twonorm{
            (\partial_v M)_u - 
            (\partial_v M)_{u'}}^2
        &= \sum_t \left(
            (\partial_v M)_{ut} -
            (\partial_v M)_{u't}
            \right)^2\\
        &= \sum_t \left( M_{ut} - M_{u't} \right)^2 -
            \left( M_{uv} - M_{u'v} \right)^2\\
        &= \twonorm{M_u - M_{u'}}^2 - 
            \left( M_{uv} - M_{u'v} \right)^2\\
        &\geq \twonorm{M_u - M_{u'}}^2 - 1
    \end{align*}
\end{proof}

\claimwithoutrowsquare*

\begin{proof}
    We have
    \begin{align*}
        \msq{M}{k}_{ij}
        &= \sum_{l \neq k} M_{il} M_{lj} \\
        &= \sum_{l} M_{il} M_{lj} - M_{ik} M_{kj}\\
        &= \mats{M^2}_{ij} - M_{ik} M_{kj}
    \end{align*}
    The product $M_{ik}M_{kj}$ is at most one and at least zero, which proves
    the claim.
\end{proof}

\claimnormexpansion*

\begin{proof}
    For any $u,v$ we have
    \[
        \mat{M^2}_{uv} = \sum_k M_{uk} M_{vk}.
    \]
    Therefore,
    \begin{align*}
        \mat{M^2}_{ii} - 2\mat{M^2}_{ii'} + \mat{M^2}_{i'i'}
        &= \sum_k M_{ik}^2 - 2 \sum_k M_{ik}M_{i'k} + \sum_k M_{i'k}^2\\
        &= \sum_k \left( M_{ik} - M_{i'k}\right)^2\\
        &= \twonorm{M_i - M_{i'}}^2.
    \end{align*}
\end{proof}

\claimdistancelowerbound*

\begin{proof}
    We may apply Claim~\ref{claim:without_row_norm} to obtain
    \renewcommand{\mna}{\nmsq{A}{k}}
    \begin{align*}
        &\frac{1}{n} \twonorm{\edgeprob_i - \edgeprob_j}^2
        \leq \frac{1}{n} \twonorm{ 
            (\partial_k \edgeprob)_i - (\partial_k \edgeprob)_{j}}^2 + \frac1n
        \intertext{which may be expanded using Claim~\ref{claim:norm_expansion},
            yielding:}
        &\qquad = \mnp[k]_{ii} - 2 \mnp[k]_{ij} + \mnp[k]_{jj} + \frac1n\\
        &\qquad \leq \Bigg| \mnp[k]_{ii} - \mnp[k]_{ij} \Bigg| + 
            \Bigg| \mnp[k]_{jj} - \mnp[k]_{ij} \Bigg| + \frac1n
        \intertext{By virtue of the fact that every block $\refinement(x_i)$ in
            the refinement contains at least 4 points, we may find an $x_{\tilde
            i} \in \refinement(x_i) \cap S$ and $\tilde j \in \refinement(x_j)
            \cap S$ such that $\tilde i \neq i,k$ and $ \tilde j \neq j,k$. It
            is clear that $\mnp[k]_{ii}$ differs from $\mnp[k]_{i\tilde i}$ by
            at most $\lipschitz\Delta$, and similarly for the other terms.
            Hence}
        &\qquad \leq \Bigg| \mnp[k]_{i \tilde{i}} - \mnp[k]_{\tilde i j} \Bigg|
            + \Bigg| \mnp[k]_{j\tilde{j}} - \mnp[k]_{i\tilde{j}} \Bigg| +
            \frac1n + 4\lipschitz\Delta
        \intertext{Next we apply the assumption that $D(\partial_k A, \partial_k
            P) < \epsilon$:}
        &\qquad \leq \Bigg| \mna_{i \tilde{i}} - \mna_{\tilde i j} \Bigg| + 
            \Bigg| \mna_{j\tilde{j}} - \mna_{i\tilde{j}} \Bigg|
            + \frac{1}{n} + 4\lipschitz\Delta + 4\epsilon\\
        &\qquad\leq 2 \max_{l \neq i,j} 
            \left| \mna_{il} - \mna_{jl} \right|
            + \frac1n + 4 \lipschitz\Delta + 4\epsilon
        \intertext{We recognize this as:}
        &\qquad= 2 d_k(i,j)
            + \frac1n + 4 \lipschitz\Delta + 4\epsilon.
    \end{align*}
\end{proof}

\claimdistanceupperbound*

\begin{proof}
    We have
    \begin{align*}
        d_j(i,i') &= \max_{k \neq i, i'} 
            \left| \nmsq{A}{j}_{ik} - \nmsq{A}{j}_{i'k} \right|
        \intertext{Applying the fact that $D(\partial_j A, \partial_j P) \leq
            \epsilon$:}
        &\leq \max_{k \neq i,i'}
            \left| \nmsq{P}{j}_{ik} - \nmsq{P}{j}_{i'k} \right|
            + 2 \epsilon
        \intertext{Applying Claim~\ref{claim:without_row_square} yields an
        additional two terms of $1/n$:}
        &\leq \max_{k \neq i,i'}
        \left| \mats{P^2/n}_{ik} - \mats{P^2/n}_{i'k} \right|
            + 2\epsilon + \frac{2}{n}
        \intertext{The fact that $x_i$ and $x_{i'}$ are from the same block of
            the $\Delta$-refinement implies that $|x_i - x_{i'}| \leq \Delta$.
            Hence, by smoothness of $W$, we have that $|P_{ik} - P_{i'k}| \leq
            \lipschitz \Delta$ for every $k$. It is therefore the case that for
            any $k$ $\left|\mats{P^2/n}_{ik} - \mats{P^2/n}_{i'k}\right| \leq
            \lipschitz\Delta$, as is shown in the proof of Lemma~5.2 in
            \cite{Zhang2015-ik}.  Therefore:}
        &\leq 
            \lipschitz\Delta 
            + 2 \epsilon + \frac2n.
    \end{align*}
\end{proof}

\section{Experiments}
\label{apx:experiments}

In this section we apply the graph clustering method proposed in
\Cref{fig:algorithm} to real and synthetic data and discuss the results. The
purpose of these experiments is to help the reader develop an intuition for how
the clustering method works, and not necessarily to demonstrate superior
practical performance. As such, only limited comparisons are made to existing
clustering methods.

\subsection{Football dataset}

We first apply \Cref{fig:algorithm} to the football network from
\cite{Girvan2002-mo}. This is a undirected, unweighted graph representing the
games played between all NCAA Division I-A American college football teams
during the regular season in the year 2000. Each team appears as a node in the
graph; an edge exists between two teams if and only if they played one another.
The graph, shown in \Cref{fig:football_graph}, includes 115 nodes (teams) and
613 edges (games).

\begin{figure}[b!]
    \centering
    \includegraphics[scale=.3,angle=90,trim={1.8in 1.6in 1.8in 1.6in},clip]{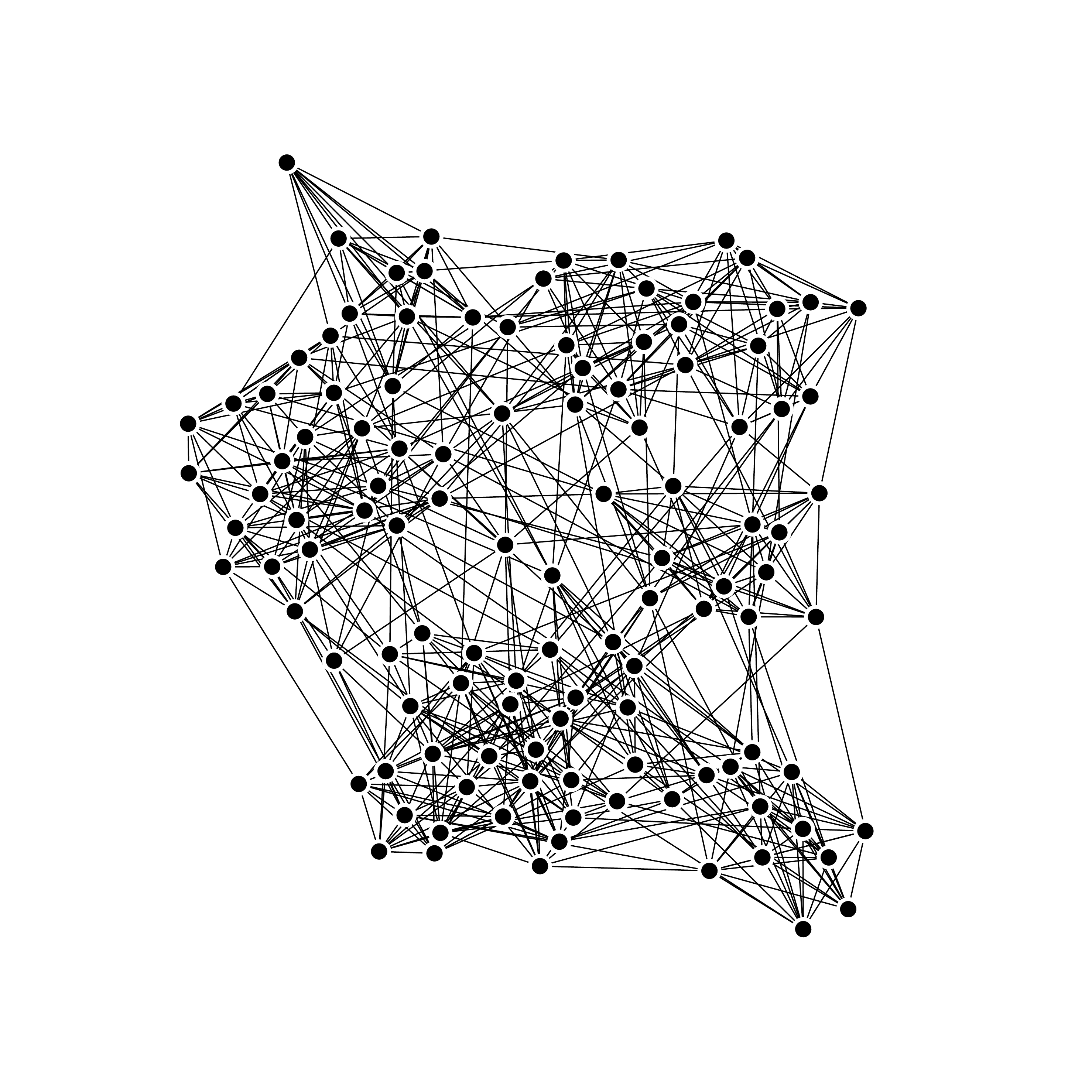}
    \caption{\label{fig:football_graph} The network of American college football
        games played during the 2000 regular season. Each node in the graph
        represents a team, and each edge represents a game played.}
\end{figure}

In this year, the teams in Division I-A were divided into eleven football
conferences, excepting five ``independent'' teams which belonged to no
conference in particular. The conferences and their associated teams are shown
in \Cref{table:conferences}. In general, an American college football team will
play the majority of its games against opponents belonging to its own conference
-- though the team will not usually play every other conference member in the
same season. The remaining games on the team's schedule are against
out-of-conference opponents. For instance, Ohio State belongs to the Big 10
conference, and in this particular season played conference opponents Iowa,
Illinois, Purdue, Michigan, Minnesota, Wisconsin, Michigan State, and Penn
State, as well as out-of-conference opponents Miami of Ohio, Arizona, and Fresno
State. Because of this connection between conference membership and the
scheduling of games, it is reasonable to assume that the graph of football games
will exhibit cluster structure. In particular, the clusters of the graph should
roughly correspond to the eleven football conferences. As such, we apply the
neighborhood smoothing and clustering method to this network and compare the
resulting clusters to the eleven football conferences.

\begin{table}[t]
\resizebox{\linewidth}{!}{%
\begin{tabular}{cccccc}
{\bf ACC} & {\bf Big 10} & {\bf Big 12} & {\bf Big East} & {\bf C-USA} & {\bf Independent}\\
Clemson & Illinois & Baylor & BostonCollege & AlabamaBirmingham & CentralFlorida\\
Duke & Indiana & Colorado & MiamiFlorida & Army & Connecticut\\
FloridaState & Iowa & IowaState & Pittsburgh & Cincinnati & Navy\\
GeorgiaTech & Michigan & Kansas & Rutgers & EastCarolina & NotreDame\\
Maryland & MichiganState & KansasState & Syracuse & Houston & UtahState\\
NorthCarolina & Minnesota & Missouri & Temple & Louisville & \\
NorthCarolinaState & Northwestern & Nebraska & VirginiaTech & Memphis & \\
Virginia & OhioState & Oklahoma & WestVirginia & SouthernMississippi & \\
WakeForest & PennState & OklahomaState &  & Tulane & \\
 & Purdue & Texas &  &  & \\
 & Wisconsin & TexasA\&M &  &  & \\
    &  & TexasTech &  &  & \\[2em]
{\bf MAC} & {\bf MW} & {\bf Pac 10} & {\bf SEC} & {\bf Sunbelt} & {\bf WAC}\\
Akron & AirForce & Arizona & Alabama & ArkansasState & BoiseState\\
BallState & BrighamYoung & ArizonaState & Arkansas & Idaho & FresnoState\\
BowlingGreenState & ColoradoState & California & Auburn & LouisianaLafayette & Hawaii\\
Buffalo & NevadaLasVegas & Oregon & Florida & LouisianaMonroe & LouisianaTech\\
CentralMichigan & NewMexico & OregonState & Georgia & MiddleTennesseeState & Nevada\\
EasternMichigan & SanDiegoState & SouthernCalifornia & Kentucky & NewMexicoState & Rice\\
Kent & Utah & Stanford & LouisianaState & NorthTexas & SanJoseState\\
Marshall & Wyoming & UCLA & Mississippi &  & SouthernMethodist\\
MiamiOhio &  & Washington & MississippiState &  & TexasChristian\\
NorthernIllinois &  & WashingtonState & SouthCarolina &  & TexasElPaso\\
Ohio &  &  & Tennessee &  & Tulsa\\
Toledo &  &  & Vanderbilt &  & \\
WesternMichigan &  &  &  &  &\\[1em]
\end{tabular}
}
\caption{\label{table:conferences} The teams belonging to each conference. Note
that the dataset from \cite{Girvan2002-mo} erroneously assigns Texas Christian
to C-USA. Texas Christian was in fact in the WAC in the year 2000, and we have
made this correction before performing our analysis.}
\end{table}

The input to the algorithm is the adjacency matrix of the football graph, shown
in \Cref{fig:football_smoothing}(a). Rearranging the rows and columns of the
adjacency matrix according to conference membership as shown in
\Cref{fig:football_smoothing}(b) reveals the network's cluster structure. Note
that the algorithm \emph{does not} have access to this rearranged adjacency or
the conference membership of each team; it is shown here only for the
convenience of the reader. Smoothing was performed with the neighborhood size
parameter $C = 0.09$; the parameter was chosen by hand to produce a good
clustering. The output $\hat P$ of the network smoothing step is shown in
\Cref{fig:football_smoothing}(c); this matrix after rearranging by conference
membership is shown in \Cref{fig:football_smoothing}(d).

The effect of neighborhood smoothing is to propagate trends in scheduling to all
teams within a conference. For instance, consider the ACC and Big East
conferences. As \Cref{fig:football_smoothing}(b) shows, in this season there
were five games played between these conferences. Most ACC teams played at least
one Big East opponent, but some ACC teams played no Big East opponent. After
applying neighborhood smoothing, however, the estimated probability that any ACC
team should play any Big East team is uniformly nonzero, as shown in
\Cref{fig:football_smoothing}. That is, even if an ACC team played no Big East
opponent, the algorithm smooths the estimate of the probability of such a game
to be consistent with the other teams in the conference.

In the clustering step, single-linkage clustering is applied to $\hat P$,
interpreting it as a similarity matrix. The resulting dendrogram is shown in
\Cref{fig:football_clustering}. In general the clustering recovers the
conferences with high accuracy. In addition, because the clustering is a tree
and not a flat partitioning of the teams, more structure is evident. For
instance, the clusters corresponding to the MW (Mountain West) conference and
the Pac 10 are joined at a high level. This is because the Mountain West and Pac
10 are comprised of teams from the western U.S. and who play one another
frequently as such.

\begin{figure}[p!]
    \centering
    \begin{subfigure}[t]{.45\textwidth}
        \centering
        \includegraphics[scale=.25]{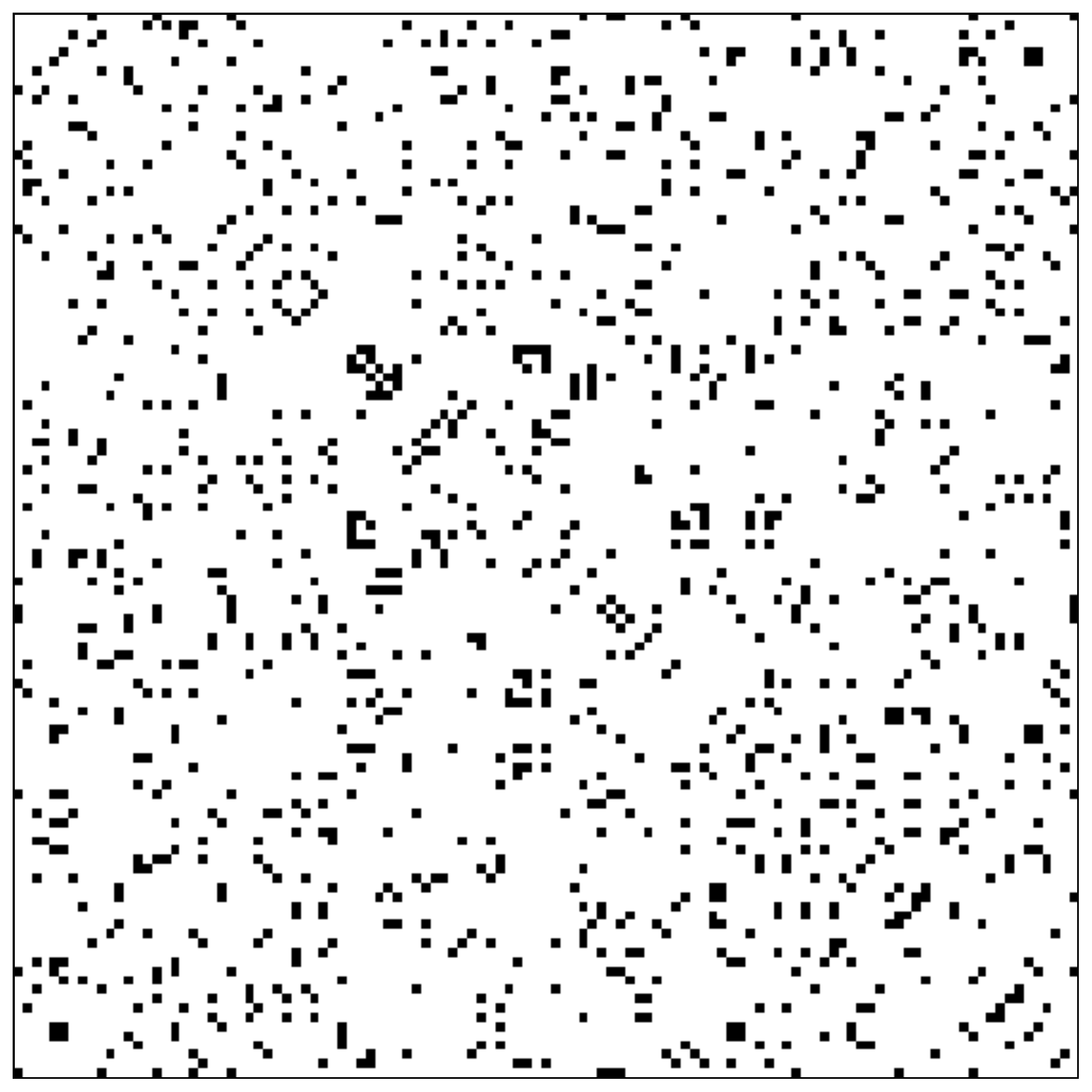}
        \caption{The input adjacency matrix.}
    \end{subfigure}\hfill%
    \begin{subfigure}[t]{.45\textwidth}
        \centering
        \includegraphics[scale=.25]{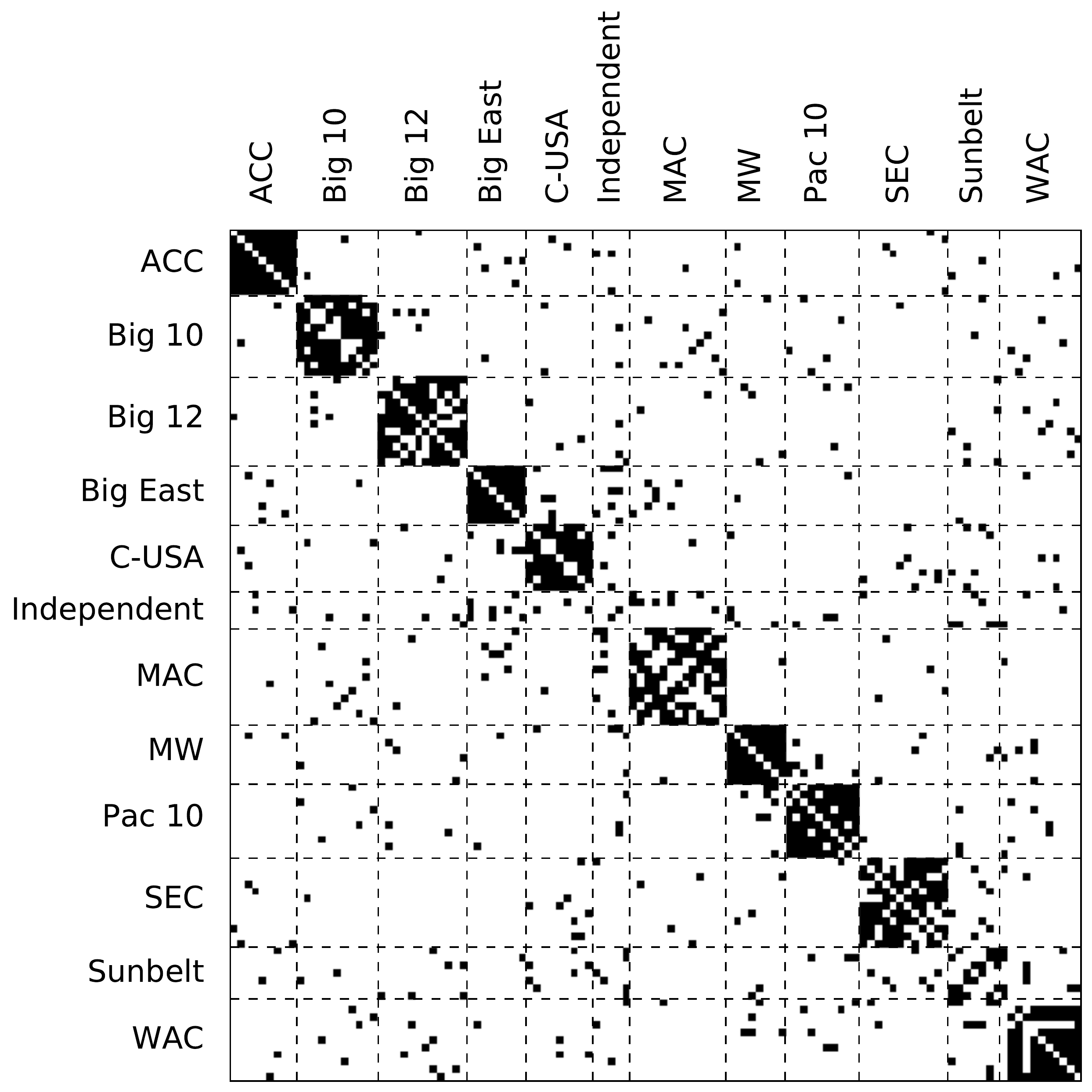}
        \caption{The input adjacency matrix, rearranged according to conference
        membership.}
    \end{subfigure}\\[1em]
    \begin{subfigure}[t]{.45\textwidth}
        \centering
        \includegraphics[scale=.25]{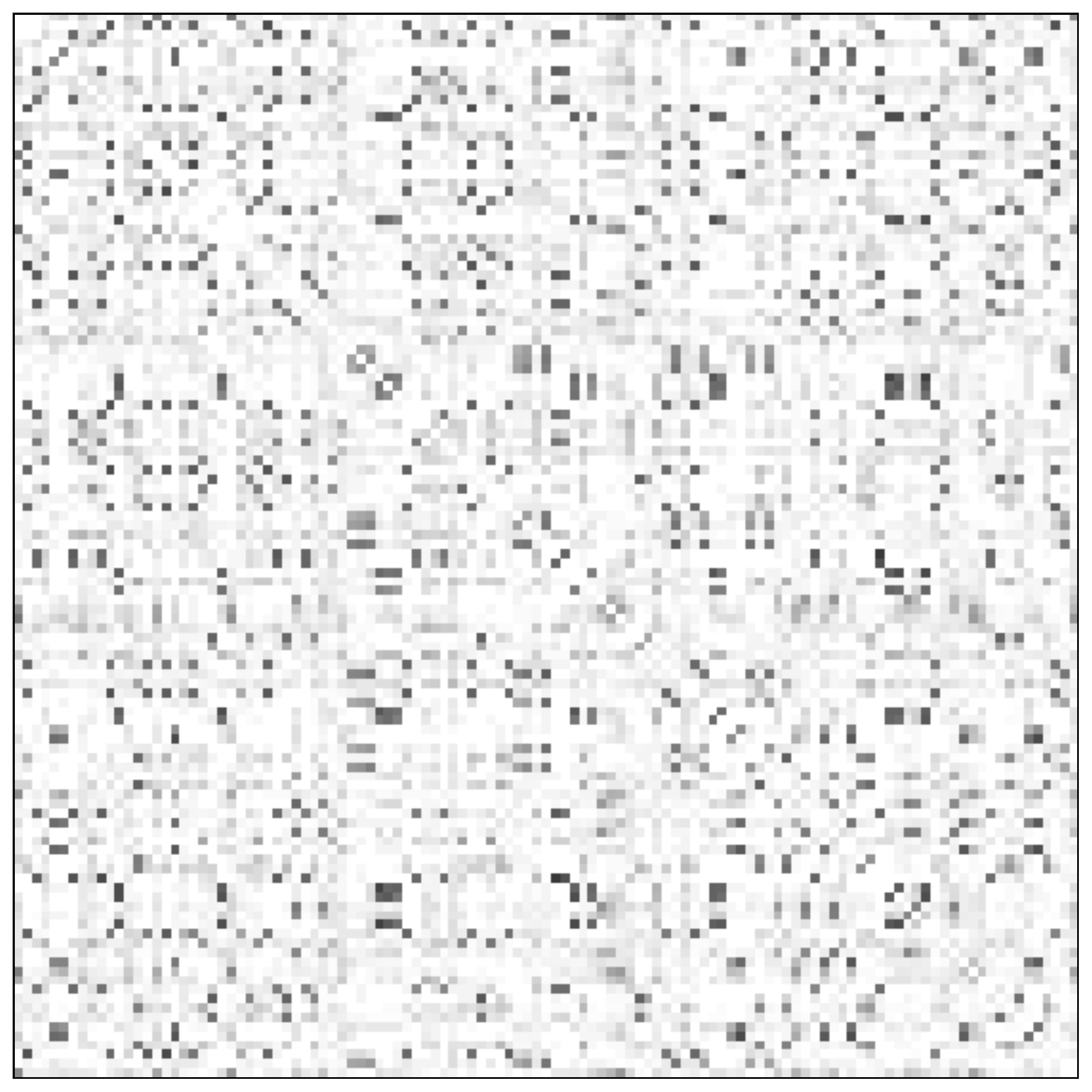}
        \caption{The result of neighborhood smoothing.}
    \end{subfigure}\hfill%
    \begin{subfigure}[t]{.45\textwidth}
        \centering
        \includegraphics[scale=.25]{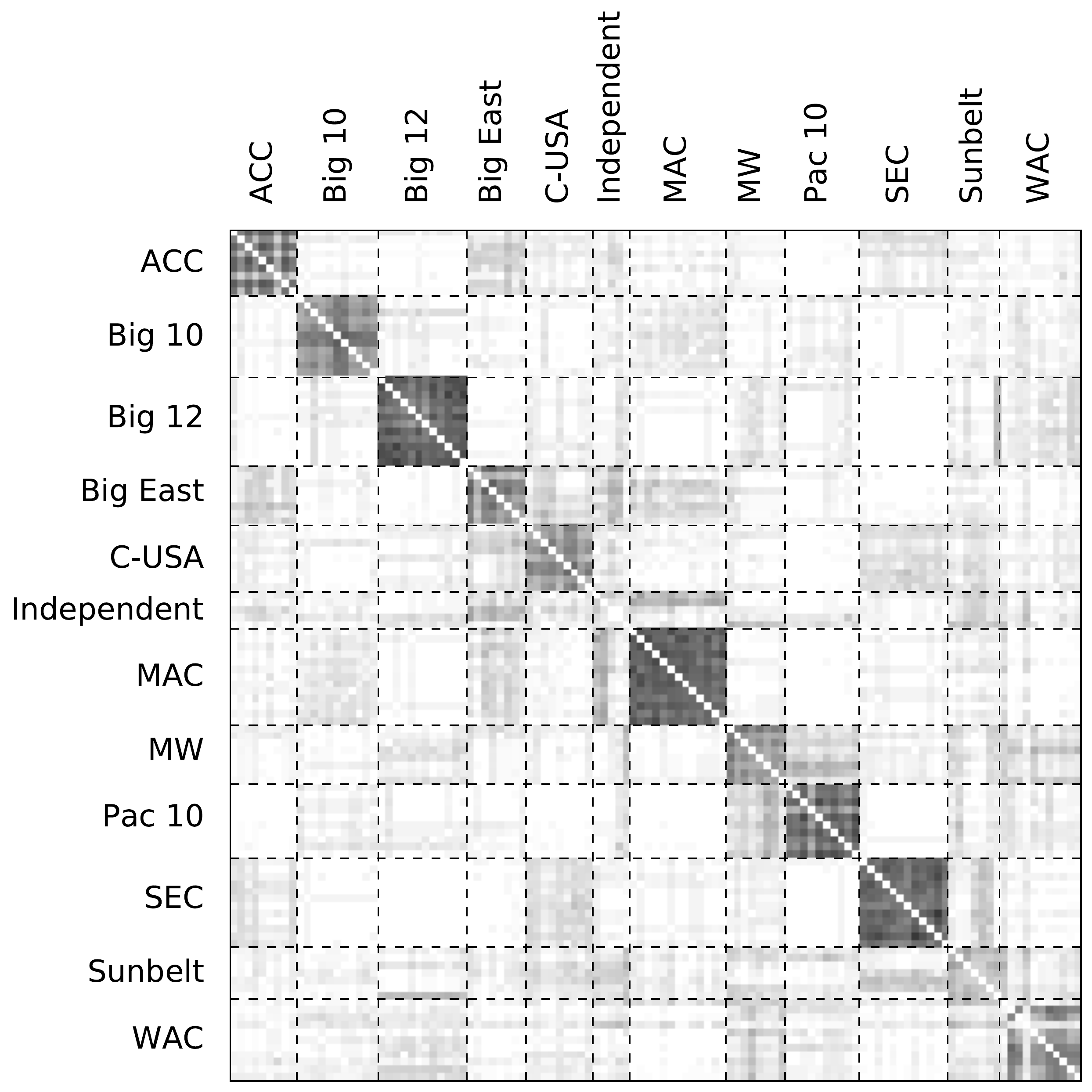}
        \caption{The result of neighborhood smoothing, rearranged according to
        conference membership.}
    \end{subfigure}

    \caption{\label{fig:football_smoothing} The neighborhood smoothing step as
    applied to the football network. The smoothing algorithm only has access to
    the input adjacency matrix as shown in (a), and not to the re-arranged matrix
    shown in (b).}
\end{figure}

\begin{figure}[h!]
    \centering
    \includegraphics[width=\textwidth]{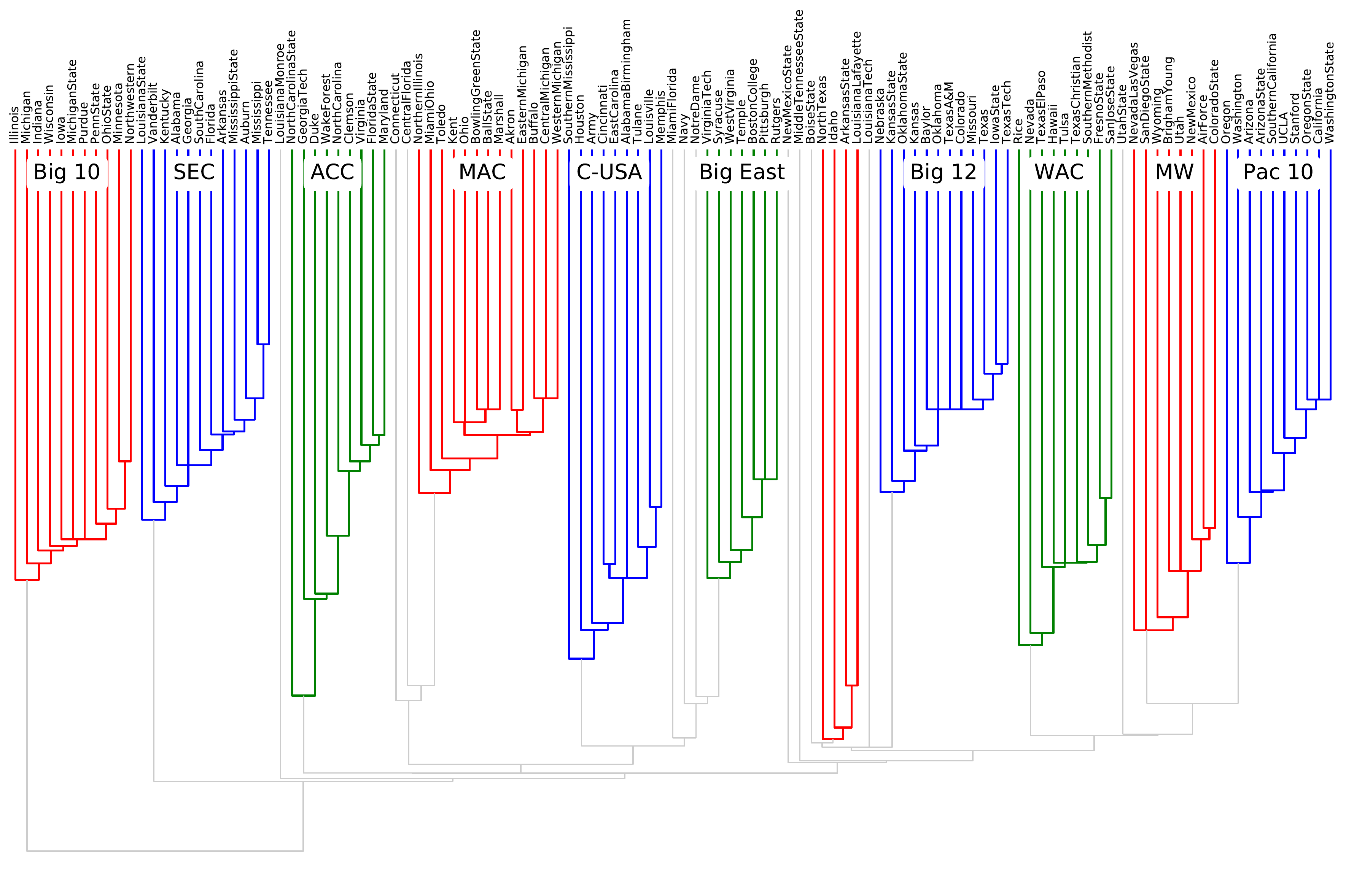}
    \caption{\label{fig:football_clustering} The result of \Cref{fig:algorithm}
        as applied to the football network. Nodes joining at higher levels of
        the tree are more similar. If all of the leaf nodes in a subtree belong
        to the same conference, every edge in the subtree is marked with the
        same color.  Different colors are used to distinguish such subtrees, but
        the particular color used is not meaningful. The conference labels in
        the figure are used to show where the majority of that conference's
        teams are in the clustering. Not marked is the Sun Belt conference, the
        majority of whose teams are placed between the Big East and Big 12, and
        the independent teams which belong to no conference in particular.}
\end{figure}

\clearpage

\subsection{Synthetic network sampled from a graphon}

In this experiment we apply \Cref{fig:algorithm} to a network sampled from the
graphon shown in \Cref{fig:tricky_graphon}. This graphon was chosen to
demonstrate a non-trivial case where a simple clustering method may yield the
incorrect result. The graphon consists of three large blocks along the diagonal
which take value 0.7. The first two of these blocks are joined by a small region
whose value is 0.5. As such, the cluster tree of this graphon is as shown in
\Cref{fig:tricky_cluster_tree}.

The adjacency matrix of a graph sampled from this graphon is shown in
\Cref{fig:tricky_adj}. The matrix in the figure has been rearranged in order to
show the cluster structure of the graph; The matrix given as input to the
smoothing algorithm is a permutation of this matrix. Smoothing was applied with
a neighborhood size parameter of $C = 0.1$. The result is shown in
\Cref{fig:tricky_p_hat}.

\begin{figure}
    \centering
    \begin{subfigure}[t]{.45\textwidth}
    \centering
    \includegraphics[width=.8\textwidth]{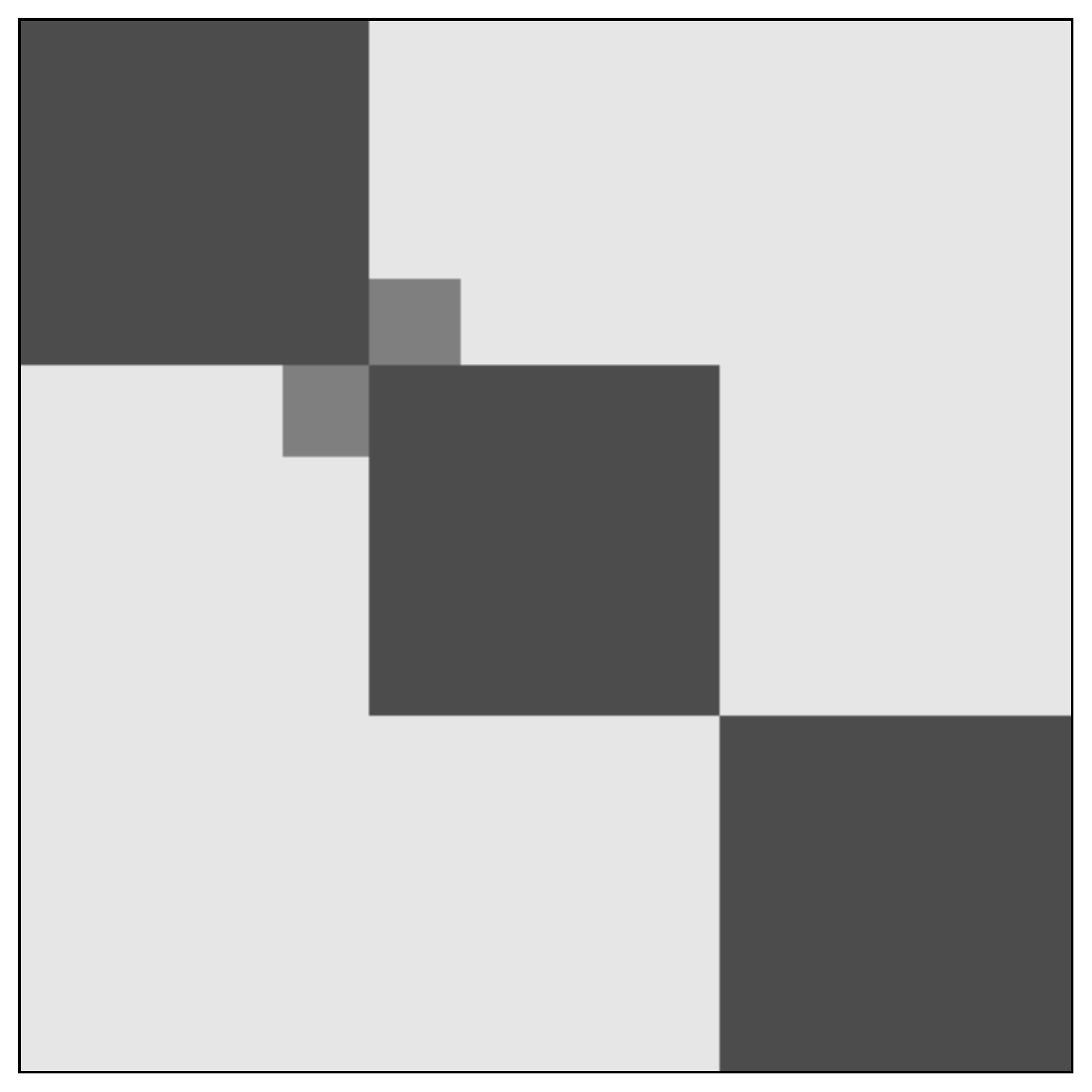}
    \caption{\label{fig:tricky_graphon} The graphon used in the synthetic
        experiment. The graphon takes on three values: The darkest region has a
        height of 0.7; the small, medium-dark blocks are of height 0.5; the
        remaining light area has value 0.1.  }
    \end{subfigure}%
    \hfill
    \begin{subfigure}[t]{.45\textwidth}
    \centering
    \includegraphics[width=.8\textwidth]{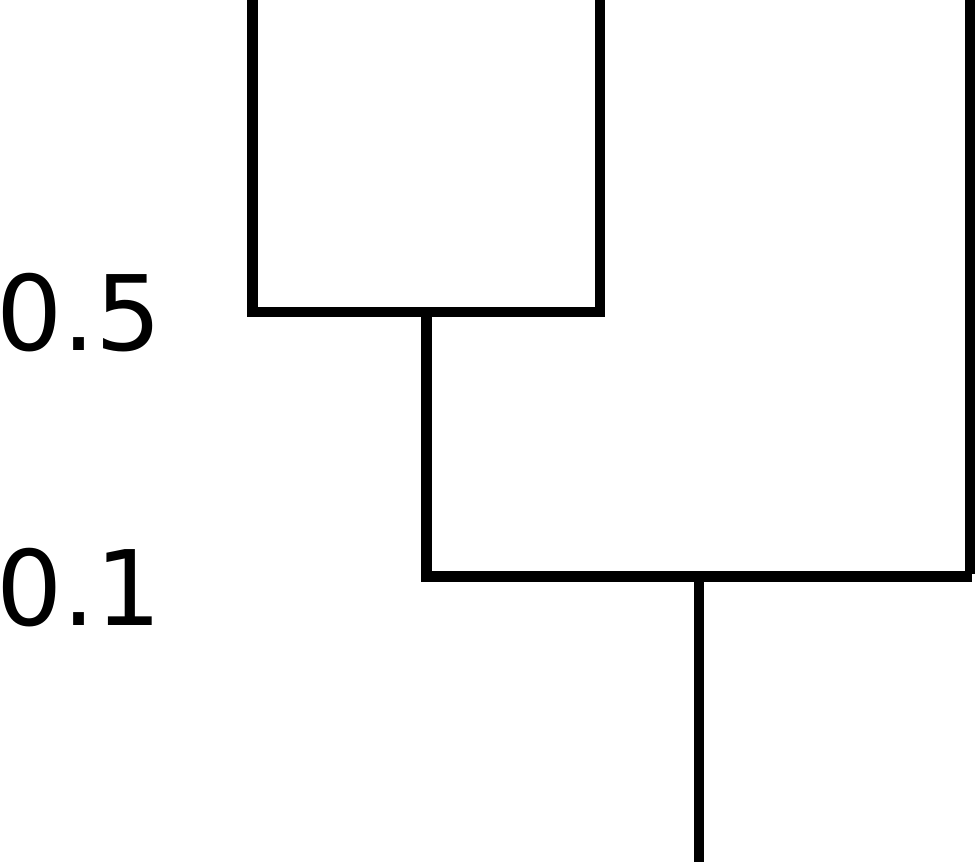}
    \caption{\label{fig:tricky_cluster_tree} The cluster tree of this graphon.
        The two leftmost blocks join at a height of 0.5. These join with the
        remaining block at 0.1.}
    \end{subfigure}\\[1em]
    \begin{subfigure}[t]{.45\textwidth}
    \centering
    \includegraphics[width=.8\textwidth]{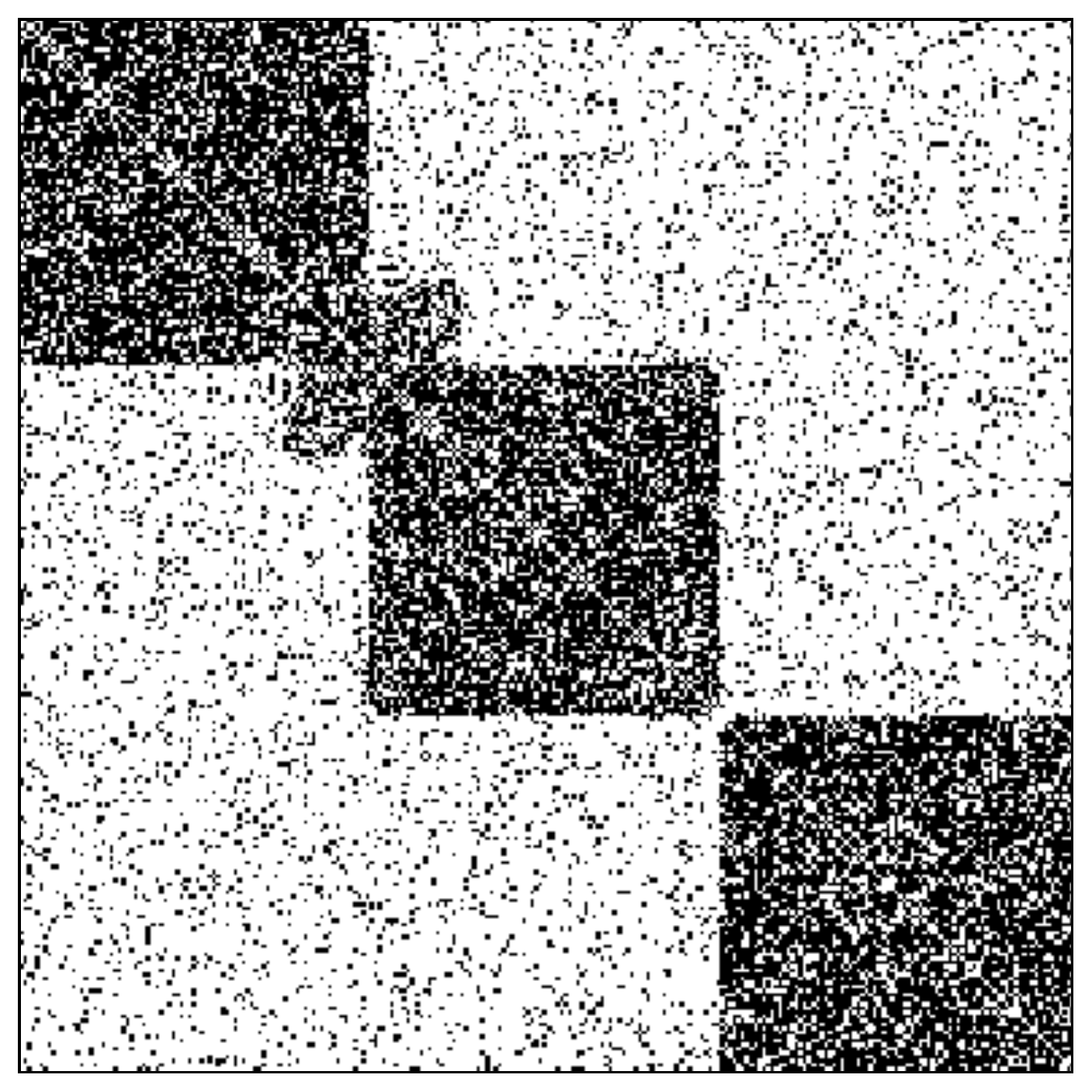}
    \caption{\label{fig:tricky_adj} An adjacency matrix sampled from the
        graphon, rearranged for the presentation (the algorithm receives a
    random permutation of this matrix.}
    \end{subfigure}%
    \hfill
    \begin{subfigure}[t]{.45\textwidth}
    \centering
    \includegraphics[width=.8\textwidth]{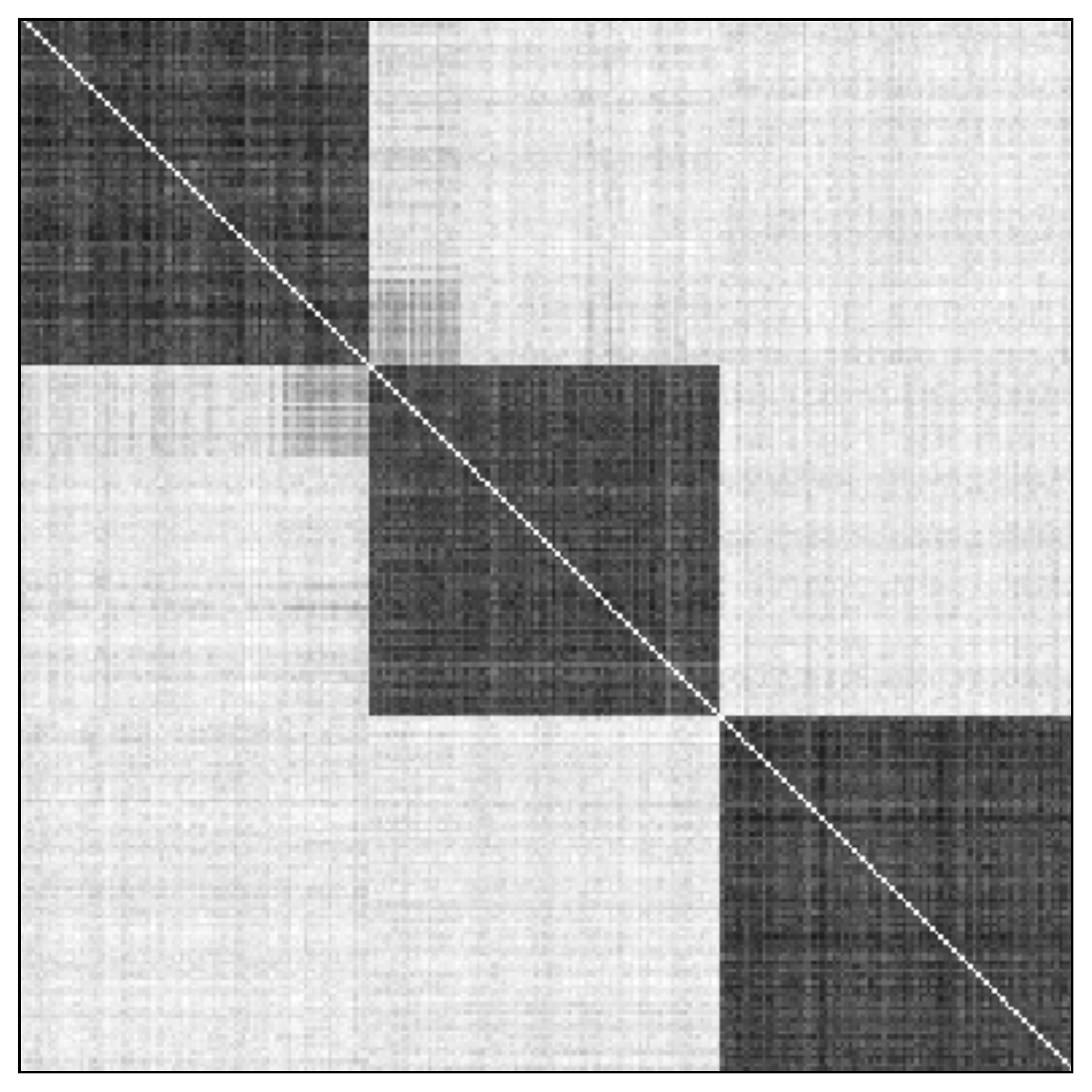}
    \caption{\label{fig:tricky_p_hat} The smoothed estimate of edge
        probabilities computed from the adjacency matrix at left.}
    \end{subfigure}
    \caption{\label{fig:tricky}}
\end{figure}

In the cluster step, single linkage is applied to the smoothed estimate of edge
probabilites. The resulting dendrogram is shown in \Cref{fig:tricky_clustering}.
Three major clusters are evident in the tree, two of which are joined at a
noticeably higher level. As we would expect from a consistent clustering method,
the dendrogram resembles the ground-truth cluster tree shown in
\Cref{fig:tricky_cluster_tree}.

On the other hand, one simple approach to network clustering fails. In this
approach, we use the pairwise distance between columns of the adjacency matrix
as input to single-linkage clustering. That is, for every $i, j \in \{1, \ldots,
n\}$, we use the matrix $D$ whose $i,j$ entry is $\|A_i - A_j\|$, where $A_i$
and $A_j$ are the $i$th and $j$th columns of $A$, respectively, and $\|\cdot\|$
is a suitable norm -- here, we use the 2-norm. Such a simple approach can often
work in practice; for example, this method works well on the football network in
the previous section. However, as the results shown in
\Cref{fig:tricky_distance_clustering} demonstrate, it does not work as well for
recovering the graphon cluster tree. Though the method appears to recover three
clusters, it does not join two of them at a significantly higher level.
Therefore the resulting tree does not resemble the graphon cluster tree. In
fact, is easily seen that this method is not consistent in the sense described
earlier.

\begin{figure}
    \begin{subfigure}[t]{\textwidth}
    \includegraphics[width=\textwidth]{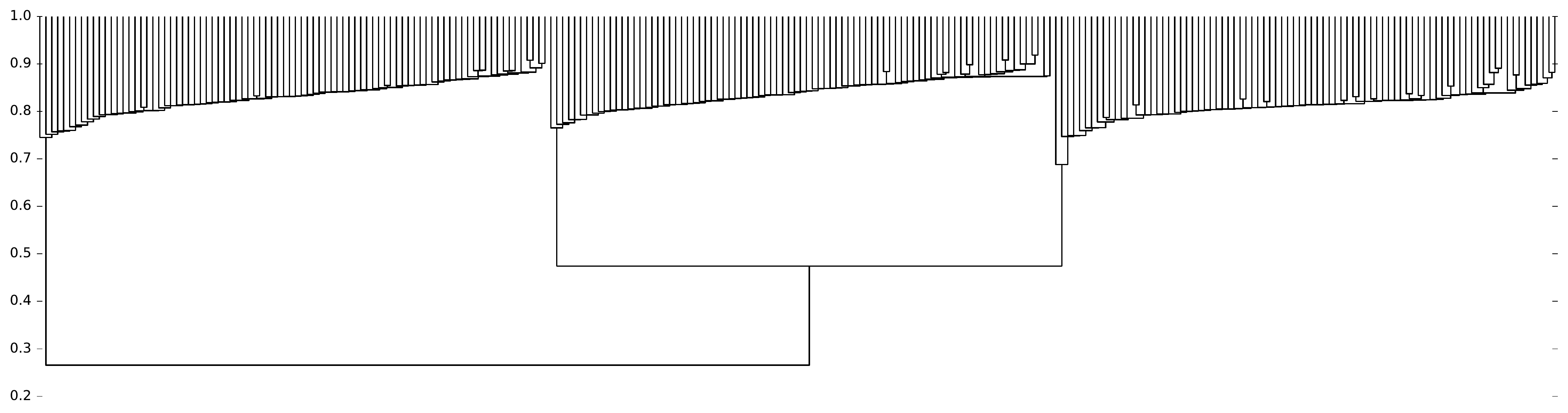}
    \caption{\label{fig:tricky_clustering} The result of applying
    \Cref{fig:algorithm} to the synthetic network generated from a graphon.}
\end{subfigure}\\[1em]
    \begin{subfigure}[t]{\textwidth}
    \includegraphics[width=\textwidth]{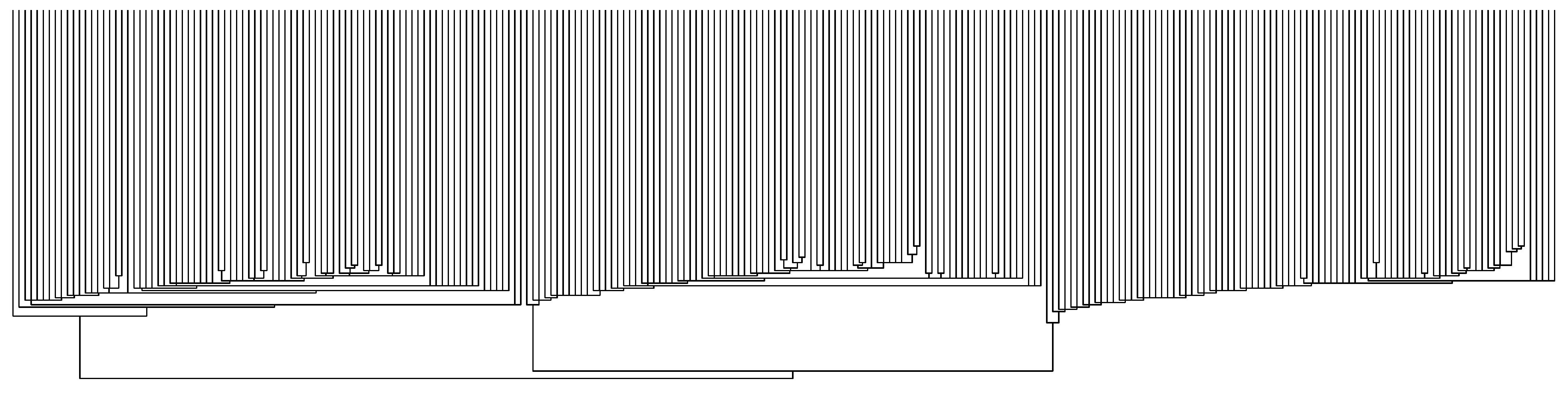}
    \caption{\label{fig:tricky_distance_clustering} The result of a simple,
        inconsistent clustering algorithm which applies single-linkage to the
    pairwise distances between the columns of the adjacency matrix.}
    \end{subfigure}
    \caption{\label{fig:tricky_clusterings}}
\end{figure}

\end{document}